%% file: neurips_2026.tex
\documentclass{article}

 \usepackage[nonatbib,preprint]{neurips_2026}

\usepackage[utf8]{inputenc} 
\usepackage[T1]{fontenc}    
\usepackage{hyperref}       
\usepackage{url}            
\usepackage{booktabs}       
\usepackage{amsfonts}       
\usepackage{nicefrac}       
\usepackage{microtype}      
\usepackage{xcolor}         

\usepackage{thm-restate}
\usepackage{amsmath, amssymb, amsthm}
\usepackage{hyperref}
\usepackage{cleveref}
\usepackage{tikz, tikz-cd}%
\usepackage{quiver}
\usepackage{enumitem}
\usepackage{adjustbox}
\usepackage{cancel}
\usepackage[
style = alphabetic,
maxnames=2, maxbibnames=4,uniquelist=false,natbib=true,
]{biblatex} 

\input{math}

\DeclareMathOperator{\E}{\mathbb{E}}
\DeclareMathOperator{\risk}{\mathcal{R}}
\DeclareMathOperator{\Var}{Var}
\newcommand{\inner}[2]{\langle #1, #2 \rangle}
\newcommand{\err}{\mathrm{err}}

\newtheorem{lemma}{Lemma}

\newtheorem{assumption}{Assumption}
\newtheorem{remark}{Remark}

\usepackage{mathtools}

\usepackage[dvipsnames]{xcolor}
\usepackage{wrapfig}

\title{Beyond Linear and Overcomplete Regimes:\\ A Mean-Field Analysis of Bottleneck Autoencoders
}

\author{%
Santanu Das\thanks{Equal Contributions}\\
STCS department\\
Tata Institute Of Fundamental Research\\ 
Mumbai, India-400005\\
dassantanu315@gmail.com
\And
    Ramyak Bilas$^*$\\
    Department of Mathematics\\
    Indiana University Bloomington\\
    Indiana, USA-47405 \\
    \texttt{rbilas@iu.edu} \\
    \And
    Pascal Esser$^*$\\
     Department of Mathematics\\
     Ludwig-Maximilians-Universit\"at M\"unchen\\
     Germany,  M\"unchen-80799\\
     \texttt{esser@math.lmu.de}
    \And
    Satyaki Mukherjee$^*$\\
    Department of Mathematics\\
    National University of Singapore\\
    Singapore, Singapore-119077 \\
    \texttt{satyaki@nus.edu.sg} \\
}

\begin{document}

\maketitle

\begin{abstract}
Autoencoders (AEs) learn low-dimensional representations by mapping data into a latent space while minimizing reconstruction error. Despite their empirical success, theoretical understanding remains limited and largely restricted to linear models or settings without a bottleneck.
In this work, we study nonlinear AEs with a fixed finite-dimensional bottleneck in the mean-field (MF) regime. We derive explicit MF learning dynamics for both encoder and decoder, providing a tractable characterization of training in the nonlinear setting. We show that, over finite time horizons, the empirical risk of finite-width networks trained with stochastic gradient descent closely tracks the MF risk trajectory with high probability. At optimality, we further establish that the finite-width risk converges to the MF optimum, demonstrating that finite networks are sufficiently expressive to approximate the infinite-width solution.

\end{abstract}

\input{sections/new_introduction}

\input{sections/main_results}
\input{sections/setup}
\input{sections/Risk}

\input{sections/discussion}

{
\printbibliography[] 
}

\clearpage

\appendix

\input{sections/appendix}
\input{sections/further_questions}

\end{document}

%% file: math.tex
\newcommand{\norm}[1]{\left\lVert#1\right\rVert}

\newcommand{\bbR}{\mathbb{R}}

\newif\ifmynotes
\mynotesfalse


%% file: sections/new_introduction.tex
\section{Introduction}

Representation learning is grounded in the central paradigm that the essential structure of high-dimensional data can be captured in a latent, lower-dimensional space. Classical approaches realize this idea through linear transformations, including principal component analysis (PCA), independent component analysis, factor analysis, projection pursuit, and non-negative matrix factorization \cite{pml1Book, bookMVstat,Friedman1974APP}. While these methods remain widely used in practice, a common limitation is that they are inherently restricted to learning linear representations of unlabeled data.
To overcome this limitation, a broad class of nonlinear methods has been developed. These include multidimensional scaling and its variants such as t-SNE \cite{bookMVstat,Tenenbaum2000AGG}, kernel-based methods \cite{Schoelkopf}, and tensor factorization techniques \cite{doi:10.1137/07070111X}. More recently, neural network–based approaches have emerged as a dominant paradigm for representation learning, including self-organizing maps, restricted Boltzmann machines (RBMs), and autoencoders (AEs) \cite{Kramer1991AIChE_AE,pml1Book}. These models enable the learning of highly expressive nonlinear representations and have been successfully applied across a wide range of domains \citep{math11081777,Berahmand2024AutoencodersAT}.

And while several properties of such models have been studied empirically \cite{Disentanglement2020, s23042362, leeb2023exploringlatentspaceautoencoders,Fournier_2019}, this increased expressivity, however, comes at the cost of reduced theoretical understanding \cite{esser2025theoreticalfoundationsrepresentationlearning}. This gap is especially critical given the increasing deployment in system-critical applications such as healthcare \cite{MAO2025112536, killian2020empiricalstudyrepresentationlearning,9308129}. In particular, for nonlinear representation models, one fundamental question remains largely open:
\begin{center}
    \emph{What is the structure of the representation learned by the model?}
\end{center}
Answering this question is central to provide formal guarantees and explainability \cite{Zhang2018VisualIF, XU2025102721} for downstream tasks such as clustering. 

\paragraph{Theoretical analysis of AEs.} In this work, we take a step toward closing this gap by providing a theoretical analysis of \emph{nonlinear bottleneck autoencoders} (BAEs), where the latent dimension $(b)$   is much smaller then the data dimension $(d)$ and both encoder and decoder are neural networks. This setting captures the essence of representation learning: compression through a nonlinear mapping. However this setting remains poorly understood due to its analytical complexity. 
To make progress, most prior theoretical analyses typically rely on one of the following simplifying assumptions that either removes a key aspect of representation learning or simplify the network architecture..

\emph{i) Linear autoencoders.} A common simplification is to assume linear encoder and decoder mappings. In this setting, the learned representation can be directly linked to PCA \citep{BALDI1989NN,Kunin20219pmlr_AE_losss_landscape,Bao2020Neurips_AE_PC,Pretorius2018LearningDO,esser23a}, and a detailed theoretical understanding has been developed, including results on learning dynamics \cite{golikov2026exact}, implicit regularization \cite{Identity}, and generalization \cite{ham2025impactbottlenecklayersskip}. However, the linearity assumption severely restricts the class of representations that can be learned, limiting these results to linear latent structures. As illustrated in Figure~\ref{fig:intro}, linear autoencoders fail to capture nonlinear structure in the data, highlighting the need for a theory beyond the linear regime.

\emph{ii) Autoencoders without a bottleneck.} Another line of work studies nonlinear autoencoders without a bottleneck constraint, i.e., in the regime $b \gg d$. While this preserves the nonlinear nature of the model, it removes the compression that is central to representation learning. In this setting, the latent representation can degenerate to a trivial embedding of the input into a higher-dimensional space (see Figure~\ref{fig:intro}). Consequently, theoretical analyses in this regime focus primarily on properties of the overall input-output mapping, including learning dynamics in the neural tangent kernel (NTK) \cite{AENTK} and mean-field (MF) limit \cite{nguyen2021analysis}, feature learning \cite{han2025on,mendes2026solvablehighdimensionalmodelnonlinear}, and generalization \cite{mendes2026solvablehighdimensionalmodelnonlinear,epstein2019generalizationboundsunsupervisedsemisupervised}, rather than on the structure of the learned representation itself.

Taken together, therefore most existing theory captures either \emph{nonlinearity without compression} or \emph{compression without nonlinearity}, but not both simultaneously. Working towards the non-linear bottleneck setting, a third simplified setting for theoretical analysis has been established.

\emph{iii) Shallow networks with a bottleneck and nonlinearity.}
This setting is characterized by a single linear encoder and decoder (which are often weight-tied) with a nonlinearity at the bottleneck\footnote{Formally, this nonlinear shallow AE is usually written as $f(x) = U^\top\sigma(Wx)$, where $W,U\in\bbR^{b\times d}$ are the linear \emph{encoder} and \emph{decoder}, and $\sigma(\cdot)$ is the point-wise nonlinearity.}.
For shallow networks, under additional assumptions on the data and in the high-dimensional limit, \cite{refinetti2022dynamicsrepresentationlearningshallow} derives the dynamics of representation learning. Under additional weight-tying assumptions, \cite{cui2023highdimensional} provides a characterization of the risk. In the regime where the input dimension scales linearly with the representation size, \cite{Aleksandr23} characterizes the minimizer of the population risk. \cite{Kgler2024CompressionOS} focuses on the compression of sparse data. In the denoising setting, and under Gaussian mixture assumptions on the data, \cite{NEURIPS2025_082d3d79} derives a tight asymptotic characterization of the generated distribution. \cite{Bardone2026ATO} proves a simplicity bias and establishes nearly sharp thresholds for the number of samples required by a simple denoiser.

While these results represent important steps toward a theoretical understanding of nonlinear autoencoders, the considered architecture differs significantly from those used in practice, which typically employ more expressive encoder and decoder functions.

For more general encoder and decoder functions, existing theoretical work on nonlinear autoencoders is largely confined to a manifold-learning perspective \cite{lee2023geometricperspectiveautoencoders}, which focuses on expressivity but does not explicitly account for the optimization dynamics that determine the learned representation.

Understanding nonlinear autoencoders with a bottleneck and sufficiently expressive encoders and decoders therefore remains a central open problem.

\begin{figure}[t]
    \centering
    \includegraphics[width=\linewidth]{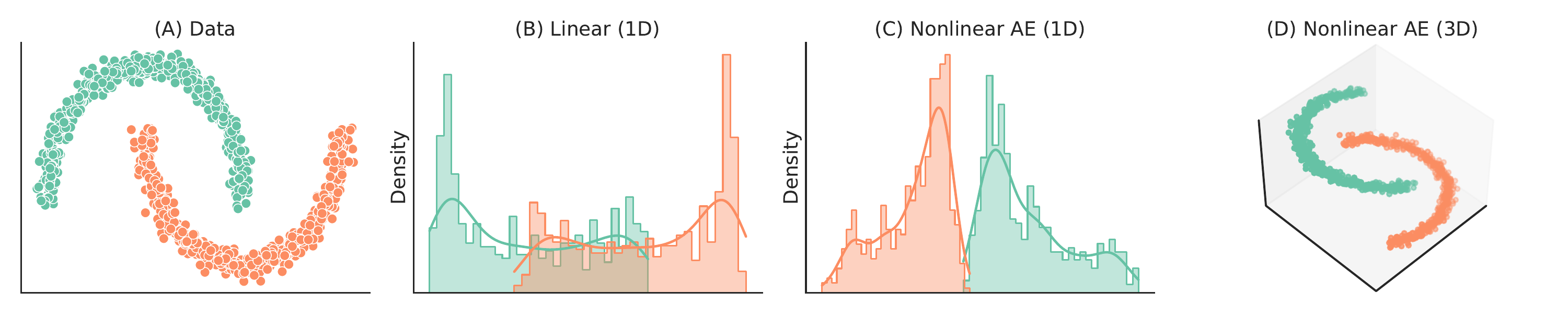}
    \caption{Why both nonlinearity and low-dimensional latent spaces are crucial for representation learning.
\textbf{(A)} Non-linearly separable data $(d=2)$.
\textbf{(B)} Embedding obtained by a linear AE with $b=1$.
\textbf{(C)} Embedding obtained by a nonlinear AE with $b=1$.
\textbf{(D)} Latent representation in a nonlinear AE without a bottleneck $(b=3)$.}
    \label{fig:intro}
\end{figure}

\paragraph{Our contribution.}
We provide a theoretical framework for analyzing nonlinear bottleneck autoencoders using mean-field (MF) methods. Mean-field analysis has emerged as a powerful tool for studying neural networks \cite{Sompolinsky1988ChaosIR,Chizat2018OnTG,Sirignano2018MeanFA,Mei2019MeanfieldTO,Goldt2019DynamicsOS,Mei2018AMF}, enabling the characterization of nonlinear models trained with gradient-based methods. Closest to our work is \cite{nguyen2021analysis}, which derives MF dynamics for shallow autoencoders; however, their analysis does not consider the bottleneck regime and therefore does not address the representation learning aspect of the problem.

Focusing on shallow autoencoders with a one-dimensional bottleneck $(b=1)$, we make the following contributions:
(i) we derive the mean-field limit of nonlinear bottleneck autoencoders and characterize their learning dynamics through a system of coupled PDEs;
(ii) we establish that the risk of finite-width networks closely tracks the MF risk during training and at optimality.

\textbf{Technical novelty.}
Our work builds on the mean-field program for neural-network optimization \citep{Chizat2018OnTG,Sirignano2018MeanFA,Mei2019MeanfieldTO}, but the bottleneck autoencoder is not a standard two-layer mean-field model. In the setting of \citet{Mei2019MeanfieldTO}, the predictor is an average over one neuron population, so the risk is a quadratic functional of a single probability measure, with one external potential and one pairwise interaction kernel. Here the bottleneck creates a genuinely compositional two-population object: an encoder measure $\rho_W$ first forms the latent variable $\alpha(x;\rho_W)=\int \varphi(x;w)\,d\rho_W(w)$, and a decoder measure $\rho_U$ then reconstructs through $\hat x(x;\rho_U,\rho_W)=\int \psi(\alpha(x;\rho_W);u)\,d\rho_U(u).$

This composition is the central obstruction. For fixed $\rho_W$, the decoder remains additive in $\rho_U$ and the risk has the familiar quadratic structure. By contrast, $\rho_W$ enters before the nonlinear decoder activation, through $\alpha(x;\rho_W)$. Thus $\risk(\rho_U,\rho_W)$ is not a quadratic functional in the tuple $(\rho_U,\rho_W)$; it is a coupled encoder--decoder functional in which one population changes the input seen by the other.

This asymmetry changes both the variational calculus and the finite-width approximation. The first variation in $\rho_U$ has the usual  form as the two-layer neural network, but the first variation in $\rho_W$ must differentiate through the bottleneck and contains the averaged decoder Jacobian $\int \partial_\alpha \psi(\alpha(x;\rho_W);u)\,d\rho_U(u)$, the mean-field analogue of backpropagation through the latent layer. At finite width, perturbing an encoder particle changes the bottleneck seen by all decoder particles, while perturbing the decoder changes the encoder vector field. Therefore the proof cannot be reduced to two independent one-population propagation-of-chaos arguments; it requires joint stability estimates for the coupled vector fields (\Cref{lem:vector_lipschitz}) and a coupled Gronwall argument (\Cref{lem:trajectory}) which was not explored previously to the best of our knowledge.

The second key ingredient is the localized concentration in \Cref{lem:trajectory}. Controlling the empirical measures in global weak metrics such as bounded Lipschitz or Wasserstein distance would introduce unfavorable high-dimensional concentration rates. Instead, we use the tailored parametric distance $d_{\mathcal{F}}(\mu,\nu)=\sup_{f\in\mathcal{F}}\|\int_\Omega f(\omega)\,d\mu(\omega)-\int_\Omega f(\omega)\,d\nu(\omega)\|_2$, where $\mathcal{F}$ contains only the test functions that actually appear in the encoder and decoder vector fields. The concentration step (\Cref{lem:empirical_concentration}) therefore controls precisely the quantities the dynamics evaluate, rather than all bounded Lipschitz test functions on the particle space.

Finally, the optimality and variational results inherit the same compositional difficulty. In \Cref{thm:closeness}, encoder sampling perturbs the latent variable before it is passed through the nonlinear decoder, so the proof must separate encoder bottleneck fluctuations from decoder averaging fluctuations. In \Cref{thm:optimal_characterization}, perturbations of $\rho_W$ must again be propagated through the decoder nonlinearity by a Taylor expansion, whereas perturbations of $\rho_U$ remain additive once the bottleneck is fixed. These are the points at which the bottleneck architecture requires new arguments beyond the standard additive mean-field theory.

\begin{remark}[Generalization to Multi-Dimensional Bottlenecks]\label{remark: multi dim}
    The presented analysis is for clarity written for a one-dimensional bottleneck ($b=1$) to avoid vector notation. The same proof strategy should extend to any fixed bottleneck dimension $b>1$. We outline the main  steps in Appendix~\ref{sec:multidim_bottleneck}.
\end{remark}

The remainder of the paper is organized as follows. In Section~\ref{sec: main results}, we provide an overview of the main results. The formal mean-field formulation and theoretical guarantees are presented in Sections~\ref{sec: def bottleneck AE MF}.

%% file: sections/main_results.tex
\section{Main Results}\label{sec: main results}

In this paper we introduce a MF formulation for BAE in section~\ref{sec: def bottleneck AE MF}.
In particular starting from the finite-width risk $\mathcal{R}_N\big(\theta_U,\theta_W\big)$ for the BAE as defined above the MF-risk is obtained by taking the limit $N \to \infty$, in which the empirical distributions of the parameters converge to probability measures $\rho_U$ and $\rho_W$. In this regime, the random finite network output converges to a deterministic limit that depends only on these measures, and the risk functional correspondingly converges to
$
\mathcal{R}(\rho_U, \rho_W),
$
which describes the reconstruction error of the infinite-width autoencoder in terms of the parameter distributions rather than individual neurons.
For the case of $b=1$ and under mild regularity assumptions on the activation functions, learning rate, and boundedness of the data distribution we establish a rigorous link between finite-width BAE and their MF limit through the following.

\underline{Result 1a: BAE learning dynamics.} Going beyond the commonly studied overcomplete (without bottleneck) setting, we derive the learning dynamics in the MF limit explicitly as given by ~\eqref{eq:pde_u} and \eqref{eq:pde_w} respectively. This expression builds the foundation for our risk analysis.

\underline{Result 1b: Finite-Risk tracks MF-Risk during training.} 
Let $(\theta_{U,t}, \theta_{W,t})$ denote the parameters of a finite-width autoencoder trained by SGD at step $t$ with $N$ hidden units, and let $(\rho_{U,t}, \rho_{W,t})$ denote the solution to the coupled mean-field PDEs for the encoder and decoder. 
For any finite time horizon $T > 0$, the empirical risk of the finite network tracks the MF risk up to an approximation error $\varepsilon$:
\begin{align}
\sup_{t \in [0, T] } 
\left| \mathcal{R}_N(\theta_{U,t}, \theta_{W,t}) - \mathcal{R}(\rho_{U,t}, \rho_{W,t}) \right|
\le O\left( e^{CT^2}/\sqrt{N} \right),
\end{align}
with high probability. This result is precisely stated in \Cref{thm:main_dynamics} and establishes that, during training, a finite-width BAE evolves approximately according to the MF gradient flow, and the finite-network risk closely follows the MF risk trajectory.

\underline{Result 2: Finite-Risk vs. MF-Risk at optimum.} 
Extending the previous result to analyze the behavior at the optimum,  we observe
\begin{align}
\left|\inf_{\theta_U,\theta_W} \mathcal{R}_N(\theta_U,\theta_W) - \inf_{\rho_U, \rho_W} \mathcal{R}(\rho_U, \rho_W)\right|\le O\left(1/N\right),
\end{align}
which formalizes that finite networks are sufficiently expressive to approximate the MF optimum, with an error that vanishes as $N \to \infty$. The exact expression is stated in Theorem~\ref{thm:closeness}. 

In combination Result~1a provides us with a tool to analyze the learning dynamics of non-linear BAE, while Result~1b~\&~2 guarantee that the risk under those learning dynamics sufficiently well models the finite width setting.

%% file: sections/setup.tex
\section{Bottleneck AE in the  Mean Field Limit}\label{sec: def bottleneck AE MF}

Let us start by formalizing the basic idea of AEs, where one aims to map a datapoint $x\in\bbR^d$ into a lower dimensional latent space $\alpha\in\bbR^b$ and then reconstruct it,
\begin{align}\label{eq: AE mapping}
    x\in\bbR^d
    \quad\overset{\phi(x; \theta_W)}{\longrightarrow}\quad
    \alpha\in\bbR^b
    \quad\overset{\psi(\alpha; \theta_U)}{\longrightarrow}\quad
   \hat x\in\bbR^d,\qquad\text{s.t. }\ b\ll d.
\end{align}
Here $\phi(x; \theta_W)$ is the \emph{encoder} function parameterized by $\theta_W$ and $\psi(\alpha; \theta_U)$ the \emph{decoder} function parameterized by $\theta_U$. 
We further illustrate this in Appendix~\ref{sec: app architecture}.
Both, $\phi(\cdot)$ and $\psi(\cdot)$ are jointly optimized by minimizing the expected mean squared error, $\risk_N(\theta_U, \theta_W)$.
Now moving from the finite width to the mean field expression,  we first collapse the layer weights into ``particles''. For the remainder of the theoretical analysis we assume that $b=1$. Therefore let an \emph{encoder particle} be $\theta_W = (w_1, w_2) \in \mathbb{R}^{d} \times \mathbb{R}$ whose corresponding activation is 
$
\phi(x; \theta_W) = w_2 \sigma(\langle w_1, x \rangle)
$
and a \emph{decoder particle} be $\theta_U = (u_1, u_2) \in \mathbb{R}^{d} \times \mathbb{R}$ whose corresponding activation is
$
\psi(\alpha; \theta_U) = u_1 \sigma(u_2 \alpha).
$
For a network of width $N$, the layer weights or rather the particles are drawn from discrete empirical probability measures: $\hat{\rho}_W^{(N)} = \frac{1}{N} \sum_{i=1}^N \delta_{\theta_W^i}, $ and $\hat{\rho}_U^{(N)} = \frac{1}{N} \sum_{j=1}^N \delta_{\theta_U^j}$, where $\delta$ is the Dirac delta measure.
Using these measures, the intermediate bottleneck $\alpha^{(N)}(x)$ and the final output $\hat{x}^{(N)}$ can be written purely as integrals over the empirical measures: $\alpha^{(N)}(x) =
\int \phi(x; \theta_W)\, d\hat{\rho}_W^{(N)}(\theta_W), $ and $ \hat{x}^{(N)}=
\int \psi\big(\alpha^{(N)}(x); \theta_U\big)\, d\hat{\rho}_U^{(N)}(\theta_U)$. 

If we change our point of view from parameters $(\theta_U,\theta_W)$ to the discrete empirical probability measure $\hat{\rho}^{(N)}_U$ and $\hat{\rho}^{(N)}_W$, we can rewrite the risk in terms of measure as $
\mathcal{R}_N\big(\hat{\rho}_U^{(N)}, \hat{\rho}_W^{(N)}\big)$. By taking the limit $N\rightarrow\infty$ the empirical measures converge to limiting probability measures such that we can write the mean-field risk functional as
$\mathcal{R}\big({\rho}_U, {\rho}_W\big)
=
\frac{1}{2}\,
\mathbb{E}_{x\sim P} \left[ \|x - \hat{x}\|^2 \right].$
Therefore in the following we aim to characterize the relation between
$
\risk_N(\theta_U, \theta_W) 
\ \text{and}\ 
\mathcal{R}\big({\rho}_U, {\rho}_W\big)
$
during training and at the optimum.

%% file: sections/Risk.tex
\section{Mean Field Risk}\label{sec: mean field risk}

With the mean field risk properly defined we can now formulate the two main characterizations outlined in the introduction. For technical reasons we will need the following assumptions.

\begin{assumption}\label{assump:learning_rate}
    The learning rate schedule $t \mapsto \xi(t)$ is absolutely bounded by a constant $M_\xi > 0$ and it is $L_\xi$-Lipschitz continuous. We also assume $\int_0^\infty \xi(t) dt = \infty$ so training progresses.
\end{assumption}
\begin{assumption} The parameter spaces for the encoder, $\Omega_W$, and decoder, $\Omega_U$ lie in a compact and convex set. \label{it:bounded_enocder_decoder-domain}
\end{assumption}
\begin{assumption}  \label{it:bounded_data-domain}The data distribution has bounded support.
\end{assumption}
\begin{assumption}   \label{it:bounded_activation}The activation function $\sigma$ is  twice continuously differentiable.
\end{assumption}

\begin{remark}[Uniform boundedness]\label{rem:bounded_context}
By Assumption~\ref{it:bounded_enocder_decoder-domain},\ref{it:bounded_data-domain}, \ref{it:bounded_activation},
the encoder and decoder parameter spaces and the data support are compact and convex, and
the activation is twice continuously differentiable. Hence, by the Extreme
Value Theorem and the Mean Value Theorem, all network quantities used below are
uniformly bounded and Lipschitz on their relevant domains.
The compact-parameter assumption can be viewed as the hard-constrained analogue
of standard $L_2$ regularization. Indeed, for finite-width parameters $\theta_U,\theta_W$, the constraint
$\|\theta_U\|_2^2+\|\theta_W\|_2^2 \le R^2$
is the constrained counterpart of the penalized objective
$\mathcal R_N(\theta_U,\theta_W)
    +\lambda\bigl(\|\theta_U\|_2^2+\|\theta_W\|_2^2\bigr),$
which is commonly used in autoencoders
\citep{vincent2010stacked,sedhain2015autorec,kunin2019loss,
ghosh2020from,kumar2020regularized}.
Consequently, throughout the paper we use a single constant $M<\infty$,
depending only on the compact domains and the activation, such that, for all
admissible $x,w,u,\alpha$,
\[
    \|x\|_2 \le M, \qquad
    |\phi(x;w)| \le M, \qquad
    \|\psi(\alpha;u)\|_2 \le M.
\]
Furthermore, the first- and second-order derivatives of
the encoder and decoder used in the analysis are also bounded by $M$, and the
corresponding maps are $M$-Lipschitz with respect to their parameters. A proof
is given in \Cref{lem:uniform_bounds} in \Cref{subsec:technical_lemma}.
\end{remark}

The first theorem derives the coupled Vlasov--McKean PDEs for the mean-field autoencoder dynamics and proves that finite-width SGD tracks these dynamics, both at the level of empirical measures and risk, over finite time horizons.

\begin{restatable}{theorem}{meanfieldlimit}[Mean-Field Dynamics and Propagation of Chaos for Autoencoders]
\label{thm:main_dynamics}
Under Assumptions \ref{assump:learning_rate}, \ref{it:bounded_enocder_decoder-domain}, \ref{it:bounded_data-domain}, and \ref{it:bounded_activation}, let $\rho_{U,0}$ and $\rho_{W,0}$ be initial probability measures and $\xi$ the learning rate. Consider the discrete stochastic gradient descent (SGD) updates initialized with iid samples from $\rho_{U,0}$ and $\rho_{W,0}$ 
and with step size $s_k = \epsilon \xi(k\epsilon)$.
For $t \ge 0$, let the measures $(\rho_{U,t}, \rho_{W,t})$ be the unique solutions to the coupled distributional Vlasov-McKean PDEs,
\begin{align}
    \partial_t \rho_{U,t}(u) &= \xi(t) \nabla_u \cdot \Big(\rho_{U,t}(u) \nabla_u \Psi_U(u; \rho_{U,t}, \rho_{W,t})\Big), \label{eq:pde_u} \\
    \partial_t \rho_{W,t}(w) &= \xi(t) \nabla_w \cdot \Big(\rho_{W,t}(w) \nabla_w \Psi_W(w; \rho_{U,t}, \rho_{W,t})\Big). \label{eq:pde_w}
\end{align}
Then, for any fixed $t \ge 0$, the empirical measures converge almost surely,
\[
\hat{\rho}^{(N)}_{U,\lfloor t/\epsilon \rfloor} \Rightarrow \rho_{U,t},
\qquad
\hat{\rho}^{(N)}_{W,\lfloor t/\epsilon \rfloor} \Rightarrow \rho_{W,t},
\]
along any sequence $(N, \epsilon = \epsilon_N)$ such that
$
N \to \infty, 
 \epsilon_N \to 0, 
 \nicefrac{N}{\log(N/\epsilon_N)} \to \infty, 
 \epsilon_N \log(\nicefrac{N}{\epsilon_N}) \to 0.
$

Furthermore let $D=d+1$ denote the dimension of the encoder and the decoder particles; $(\theta_{U,k},\theta_{W,k})$ denote the decoder and encoder weights for all neurons at $k$-th time step; and $u_j^k$ and $w_j^k$ denote the weights of $j$'th decoder and encoder neuron at the $k$'th time step updated by SGD.
Then there exists a constant $C= 10^3M^{10}\max\{ L_{\xi}M_{\xi}^2, M_{\xi}d^{1/4}\} > 0$, where $M_\xi$, $L_\xi$ (\Cref{assump:learning_rate}) and $M$ (\Cref{rem:bounded_context}) are constants, such that for any bounded Lipschitz test function $f$ with 
$\|f\|_\infty, \|f\|_{\mathrm{Lip}} \le 1$, and for $\epsilon < 1$, the following holds with probability at least $1-7e^{-\delta^2}$,
\begin{align*}
    &\max_{k \in [0, T/\epsilon] \cap \mathbb{N}}  \left| \frac{1}{N}\sum_{j=1}^N f(u_j^k) - \int f(u)\, \rho_{U, k\epsilon}(du) \right| \le  2C (\sqrt{T} \lor T) \exp\big(C T^2\big) \err_{N,D}(\delta),\\
    & \max_{k \in [0, T/\epsilon] \cap \mathbb{N}}  \left| \frac{1}{N}\sum_{j=1}^N f(w_j^k) - \int f(w)\, \rho_{W, k\epsilon}(dw) \right| \le  2C (\sqrt{T} \lor T) \exp\big(C T^2\big) \err_{N,D}(\delta),\\
    & \max_{k \in [0, T/\epsilon] \cap \mathbb{N}} \left| \mathcal{R}_N(\theta_{U,k}, \theta_{W,k}) - \mathcal{R}(\rho_{U,k\epsilon}, \rho_{W,k\epsilon}) \right|  \le 10M^3C (\sqrt{T} \lor T \lor 1) \exp\big(C T^2\big) \err_{N,D}(\delta),
\end{align*}
where the error term is given by $\err_{N,D}(\delta) \equiv \sqrt{N^{-1} \lor \epsilon} \left( \sqrt{D + \log\left(N \left(\frac{T}{\epsilon} \lor 1\right)\right)} + \delta \right)$. 
\end{restatable}

\begin{proof}[\textbf{Proof sketch}]
The proof couples the discrete SGD iterates of the finite-width network to the continuous characteristic flow of the limiting Vlasov--McKean PDE. We track the joint maximal deviation $\Delta(t):=\sup_{s\le t}\{\max_j \|u^{j}_{\lfloor s/\epsilon\rfloor}-\bar u^j_s\|^2_2+\max_i \|w^{i}_{\lfloor s/\epsilon\rfloor}-\bar w^i_s\|^2_2\}$, where $(\bar u^j_s,\bar w^i_s)$ are ideal particles driven by the infinite-width mean-field vector fields. Rewriting SGD as an integral equation decomposes the error into four terms: time discretization ($E_{\operatorname{time}}$), trajectory tracking ($E_{\operatorname{track}}$), measure concentration ($E_{\operatorname{meas}}$), and martingale noise ($E_{\operatorname{mart}}$).

The time-discretization term $E_{\operatorname{time}}$ is bounded deterministically using the boundedness and Lipschitz regularity of the vector fields and learning-rate schedule, giving an $\mathcal{O}(\epsilon^2)$ contribution at the squared-trajectory level. The tracking term $E_{\operatorname{track}}$ is controlled by Lipschitz dependence of the vector fields on particle positions and is bounded by the time integral of $\Delta(s)$. The measure term $E_{\operatorname{meas}}$ is isolated via ``bridge measures'' formed from the ideal particles, separating spatial tracking from statistical sampling; the latter is then controlled by the parametric measure gaps adapted to the network vector fields. Finally, the martingale term $E_{\operatorname{mart}}$ from the stochastic gradients is controlled uniformly over particles and time by Azuma--Hoeffding-type maximal inequalities and a union bound, contributing an $\mathcal{O}(T\epsilon)$-type term at the squared-trajectory level.
Combining these estimates gives a coupled integral inequality for $\Delta(t)$, and Gronwall's lemma yields the non-asymptotic trajectory bound; the empirical-measure and risk estimates follow by transfer to bounded Lipschitz test functions and to the risk functional.

The technical point in the treatment of $E_{\operatorname{meas}}$ is that global weak metrics such as bounded Lipschitz or Wasserstein distances would introduce unfavorable high-dimensional concentration rates. Instead, we prove Lipschitz stability of $G_U$ and $G_W$ with respect to tailored parametric distances $d_{\mathcal{F}_U}$ and $d_{\mathcal{F}_W}$, whose test classes contain only the functions appearing in the encoder and decoder vector fields. Thus $E_{\operatorname{meas}}$ in \Cref{lem:trajectory} depends on the function classes evaluated by the dynamics, rather than on a global metric over the full probability space. Together with uniform concentration for these classes and a union bound over the discretized time interval, yields the error bound $\err_{N,D}(\delta)$ for the empirical measures and the finite-width risk trajectory $|R_N(\theta_{U,k},\theta_{W,k})-R(\rho_{U,k\epsilon},\rho_{W,k\epsilon})|$. Full details are provided in Appendix~\ref{app:mean_field_dynamics_chaos}.
\end{proof}


The following theorem estimates the difference between the optimal mean field risk and the optimal finite auto-encoder risk.

\begin{restatable}{theorem}{closeness}[\textbf{Closeness between Empirical and continuous Mean-field risk}]
\label{thm:closeness}
Let the continuous mean-field risk be $\mathcal{R}(\rho_U, \rho_W)$ and the empirical finite-$N$ risk be $\mathcal{R}_N(\theta_U, \theta_W)$. Under Assumptions \ref{it:bounded_enocder_decoder-domain}, \ref{it:bounded_data-domain}, and \ref{it:bounded_activation},
suppose $(\rho_{U^*}, \rho_{W^*})$ be an $\varepsilon_0$-approximate, with $\varepsilon_0 > 0$, and the optimal measure satisfying:
\begin{equation}
    \mathcal{R}(\rho_{U^*}, \rho_{W^*}) \le \inf_{\rho_U, \rho_W} \mathcal{R}(\rho_U, \rho_W) + \varepsilon_0.
\end{equation}
Then the following holds
\[
\left| \inf_{\theta_U, \theta_W} \mathcal{R}_N(\theta_U, \theta_W) - \inf_{\rho_U, \rho_W} \mathcal{R}(\rho_U, \rho_W) \right| \le \frac{3M^4+M^2}{2N} + \frac{M^4}{N^2} + \varepsilon_0,
\]
where $M$ is a constant as defined in \Cref{rem:bounded_context}. \end{restatable}

\begin{remark}
     Since the continuous mean-field risk $\mathcal{R}(\rho_U, \rho_W)$ is non-negative, it is bounded from below and its infimum always exists. Therefore, even if an exact optimal measure achieving exactly $\inf_{\rho_U, \rho_W} \mathcal{R}(\rho_U, \rho_W)$ does not exist in the space of probability measures, we can always find a minimizing sequence of measures that approaches the infimum arbitrarily closely. This justifies the existence of the $\varepsilon_0$-approximate optimal measure $(\rho_{U^*}, \rho_{W^*})$ used in Theorem \ref{thm:closeness} for any arbitrarily small $\varepsilon_0 > 0$. If an exact minimizer does exist, the theorem holds exactly with $\varepsilon_0 = 0$.
\end{remark}

\begin{proof}[\textbf{Proof sketch}]
We give the main argument in two parts and defer the details to Appendix~\ref{sec:Closeness_Between_Empirical_and_Continuous Mean-Field_Risk}.

\textbf{Lower bound.}
Any finite network $(\theta_U,\theta_W)$ induces empirical measures
$(\hat \rho_U,\hat \rho_W)$. By definition of the mean-field risk evaluated at
empirical measures,
$
    \mathcal R_N(\theta_U,\theta_W)
    =
    \mathcal R(\hat \rho_U,\hat \rho_W).
$
Since empirical measures are admissible probability measures, taking the
infimum over finite parameters gives
$
    \inf_{\theta_U,\theta_W}\mathcal R_N(\theta_U,\theta_W)
    \ge
    \inf_{\rho_U,\rho_W}\mathcal R(\rho_U,\rho_W).
$

\textbf{Upper bound.}
Let $(\rho_{U^*},\rho_{W^*})$ be an $\varepsilon_0$-optimal mean-field pair.
Sample $u_j\sim \rho_{U^*}$ and $w_i\sim \rho_{W^*}$ independently, and define
$
    \alpha^*(x)=\int \phi(x;w)\,d\rho_{W^*}(w),$
    and
$    \hat \alpha_N(x)=\frac1N\sum_{i=1}^N \phi(x;w_i).
$
The \emph{main difficulty} is that the finite network is a Monte
Carlo approximation but does not have the linear properties of an additive model. Conditional on the bottleneck, the decoder average is additive in the sampled $u_j$'s, but the bottleneck
$\hat \alpha_N$ itself is random and is passed through the nonlinear decoder
map $\psi$. Thus encoder creates a nonlinear bias in the reconstruction
map, in addition to the usual decoder variance.

To separate these two effects, we first condition on the sampled encoder
particles. The expected decoder error decomposes into the risk of the continuous
decoder evaluated at the random bottleneck $\hat \alpha_N$, plus a conditional
decoder-variance term of order $1/N$. The first term is compared to
$\mathcal R(\rho_{U^*},\rho_{W^*})$ by applying a second-order Taylor expansion,
in the scalar bottleneck variable, around $\alpha^*(x)$. Since
$\hat \alpha_N(x)$ is unbiased for $\alpha^*(x)$, the first-order term vanishes,
and the second-order term is controlled by the uniform Hessian bound and $\E\left[(\hat \alpha_N(x)-\alpha^*(x))^2\right]\le \nicefrac{M^2}{N}$.
The conditional decoder-variance term is controlled similarly, using the
boundedness of $\psi$ and a second-order expansion of
$\int \|\psi(\hat\alpha_N(x);u)\|_2^2\,d\rho_{U^*}(u)$. This gives $\E[\mathcal R_N]
    \le
    \mathcal R(\rho_{U^*},\rho_{W^*})
    +
    \nicefrac{3M^4+M^2}{2N}
    +
    \nicefrac{M^4}{N^2}$.
Using $\varepsilon_0$-optimality of $(\rho_{U^*},\rho_{W^*})$ and an averaging
argument, there exists a realization of the sampled particles satisfying the
same bound. Taking the infimum over finite parameters and combining with the
lower bound yields: $|
\inf_{\theta_U,\theta_W}\mathcal R_N(\theta_U,\theta_W)
-
\inf_{\rho_U,\rho_W}\mathcal R(\rho_U,\rho_W)
|
\le
\nicefrac{3M^4+M^2}{2N}
+
\nicefrac{M^4}{N^2}
+
\varepsilon_0 .$

\end{proof}

Finally we state a variational optimality condition for the mean-field risk, in the spirit of Proposition~1 of \cite{Mei2018AMF}. Because the bottleneck couples the encoder and decoder measures through a nonlinear composition, we obtain a weaker, one-sided result: every global minimizer must be supported on minimizers of its first-variation potentials.

\begin{restatable}{theorem}{variationalcharacterization}[\textbf{Variational Characterization of Local Minima}]\label{thm:optimal_characterization}
Under Assumptions \ref{it:bounded_enocder_decoder-domain}, \ref{it:bounded_data-domain}, and \ref{it:bounded_activation},
suppose the pair of probability measures $(\rho_U^*, \rho_W^*)$ is a global minimizer of the risk functional $\mathcal{R}(\rho_U, \rho_W)$. Then
\begin{enumerate}
    \item $\psi_U^* := \inf_{u} \Psi_U(u; \rho_U^*, \rho_W^*) > -\infty \quad \text{and} \quad \mathrm{supp}(\rho_U^*) \subseteq \arg\min_{u} \Psi_U(u; \rho_U^*, \rho_W^*)$
    \item $\psi_W^* := \inf_{w} \Psi_W(w; \rho_U^*, \rho_W^*) > -\infty \quad \text{and} \quad \mathrm{supp}(\rho_W^*) \subseteq \arg\min_{w} \Psi_W(w; \rho_U^*, \rho_W^*)$
\end{enumerate}
\end{restatable}

\begin{proof}[\textbf{Proof sketch}]
We prove the claim by contradiction through a local mass-transport perturbation.
The potentials $\Psi_U$ and $\Psi_W$ are the first variations of the risk, so
they describe the marginal cost of placing infinitesimal mass at a given
decoder or encoder particle. Suppose first that $\rho_U^*$ assigns positive
mass to a region on which $\Psi_U$ is strictly larger than
$\inf_u \Psi_U(u;\rho_U^*,\rho_W^*)$. Moving a small amount of mass $t$ from
this region to a minimizer of $\Psi_U$ changes the risk by
$
    -t\Delta + O(t^2),
$
for some $\Delta>0$, since the risk is quadratic in $\rho_U$ when $\rho_W$ is
fixed. For sufficiently small $t$, this gives a strict decrease, contradicting
optimality.

The encoder argument is analogous but tricky because the perturbation must be propagated
through the bottleneck. Moving mass in $\rho_W$ changes
$\alpha(x;\rho_W)$ linearly to first order, and a second-order Taylor expansion
of the decoder controls the nonlinear remainder. The resulting risk change is
again of the form $-t\Delta+$higher order terms with positive coefficients, yielding the same contradiction. Hence
both $\rho_U^*$ and $\rho_W^*$ must be supported on the global minimizers of
their respective first-variation potentials. Full details are given in
Appendix~\ref{sec:Variational_Characterization_of_Local_Minima}.
\end{proof}

%% file: sections/discussion.tex
\section{Discussion and Future Work}

In this work, we developed a theoretical framework for analyzing nonlinear bottleneck autoencoders through the lens of mean-field theory. In contrast to prior work, which typically studies either linear models or nonlinear models without a bottleneck, our framework captures the core challenge of representation learning: compression through nonlinear mappings. This enables a direct analysis of the learned representation itself, rather than only the input-output mapping. The proposed formulation opens several directions for future research.

A natural extension of the present work is to broaden the class of models under consideration. In particular, to extend the analysis of Section~\ref{sec: def bottleneck AE MF} to (i)  multi-dimensional bottlenecks $(b>1,~b<d)$ as outlined in Remark~\ref{remark: multi dim} and Appendix~\ref{sec:multidim_bottleneck} and (ii) deeper encoder and decoder architectures. Both extensions are expected to introduce significantly richer geometric structures in the learned representations, while also posing new analytical challenges.

\section*{Acknowledgments and Disclosure of Funding}
This research was supported in part by the International Centre for Theoretical Sciences (ICTS) for the Data Science: Probabilistic and Optimization Methods II (code: ICTS/DSPOM2025/08).
We would like to thank Sebastian Goldt and Frederieke Richert for pointing us to additional prior work and feedback on an earlier draft.

%% file: sections/appendix.tex
\section{Notation}

For $n \in \mathbb{N}$, we write $[n] := \{1,\dots,n\}$. Vectors are denoted by lowercase letters (e.g., $x, u, w$), and probability measures by $\rho_U, \rho_W, \mu, \nu$, etc. The Euclidean norm is denoted by $\|\cdot\|_2$, and $\langle \cdot, \cdot \rangle$ denotes the standard inner product.
We consider data $x \sim P$, where $P$ is a distribution supported on a bounded subset of $\mathbb{R}^d$. Time is denoted by $t \ge 0$, with learning rate schedule $\xi(t)$ and discrete step size $\epsilon > 0$.
Expectation with respect to $P$ is denoted by $\mathbb{E}_{x \sim P}[\cdot]$, and the support of a measure $\rho$ is written as $\mathrm{supp}(\rho)$ or $\mathcal{X}$. We use standard asymptotic notation $O(\cdot)$, with universal constants absorbed into $M > 0$, which uniformly bounds all relevant quantities under the standing assumptions. 

\section{Illustration of the AE Architecture}\label{sec: app architecture}

Let us recall the general idea of the AE mapping: As such one aims to map a datapoint $x\in\bbR^d$ into a latent space $\alpha\in\bbR$ and then reconstruct it
\begin{align*}
    x\in\bbR^d
    \quad\overset{\phi(x; w)}{\longrightarrow}\quad
    \alpha\in\bbR
    \quad\overset{\psi(\alpha; u)}{\longrightarrow}\quad
   \hat x\in\bbR^d
\end{align*}
where $\phi(x; w)$ is the \emph{encoder} function parameterized by $w$ and $\psi(\alpha; u)$ the \emph{decoder} function parameterized by $u$. The encoder and decoder particle lie in $\mathbb{R}^D$ where $D=d+1$. Our bottleneck $\alpha$ lies in $\mathbb{R}^b$ and here we take $b=1$. This structure is  illustrated in the following diagram:

\[
\begin{adjustbox}{width=\textwidth}
\begin{tikzcd}
	&& {\sigma(\langle x,w_1\rangle) \in \mathbb{R}} &&&& {\sigma(u_2\alpha) \in \mathbb{R}} && \begin{array}{c} \text{contribution}\\ \text{to each output} \end{array} \\
	&& \bullet &&&& \bullet && \begin{array}{c} \psi(\alpha;u_1,u_2)\\=u_1\sigma(u_2\alpha) \end{array} \\
	\bullet && \bullet && \begin{array}{c} \phi(x;w_1,w_2)\\=w_2\sigma(\langle x, w_1\rangle) \\ \text{per particle} \end{array} && \bullet && \bullet \\
	\vdots && \vdots && \bullet && \vdots && \vdots \\
	\bullet && \bullet && {\alpha \in \mathbb{R}} && \bullet && \bullet \\
	{x\in \mathbb{R}^d} && \bullet &&&& \bullet && {\in \mathbb{R}^d} \\
	&& {\in \mathbb{R}^N} &&&& {\in \mathbb{R}^N}
	\arrow["{=}", curve={height=-12pt}, squiggly, from=1-3, to=2-3]
	\arrow["{=}", shift left, curve={height=6pt}, squiggly, from=1-7, to=2-7]
	\arrow["{\color{blue}{w_2} \in \mathbb{R}}", color={rgb,255:red,92;green,92;blue,214}, from=2-3, to=4-5]
	\arrow["{u_1 \in \mathbb{R}^d}", color={rgb,255:red,214;green,92;blue,92}, from=2-7, to=3-9]
	\arrow[draw={rgb,255:red,214;green,92;blue,92}, from=2-7, to=4-9]
	\arrow[draw={rgb,255:red,214;green,92;blue,92}, from=2-7, to=5-9]
	\arrow[shift right=5, curve={height=-18pt}, squiggly, from=2-9, to=3-9]
	\arrow[shift right, curve={height=-18pt}, squiggly, from=2-9, to=4-9]
	\arrow[shift left, curve={height=-24pt}, squiggly, from=2-9, to=5-9]
	\arrow["{w_1 \in \mathbb{R}^d}", color={rgb,255:red,214;green,92;blue,92}, from=3-1, to=2-3]
	\arrow[from=3-1, to=3-3]
	\arrow[from=3-1, to=5-3]
	\arrow[from=3-1, to=6-3]
	\arrow[from=3-3, to=4-5]
	\arrow[curve={height=-6pt}, squiggly, from=3-5, to=4-5]
	\arrow[from=3-7, to=3-9]
	\arrow[from=3-7, to=4-9]
	\arrow[from=3-7, to=5-9]
	\arrow[draw={rgb,255:red,214;green,92;blue,92}, from=4-1, to=2-3]
	\arrow[from=4-1, to=3-3]
	\arrow[from=4-1, to=5-3]
	\arrow[from=4-1, to=6-3]
	\arrow[from=4-3, to=4-5]
	\arrow["{u_2 \in \mathbb{R}}", color={rgb,255:red,92;green,214;blue,92}, from=4-5, to=2-7]
	\arrow[from=4-5, to=3-7]
	\arrow[from=4-5, to=4-7]
	\arrow[from=4-5, to=5-7]
	\arrow[from=4-5, to=6-7]
	\arrow[from=4-7, to=3-9]
	\arrow[from=4-7, to=4-9]
	\arrow[from=4-7, to=5-9]
	\arrow[draw={rgb,255:red,214;green,92;blue,92}, from=5-1, to=2-3]
	\arrow[from=5-1, to=3-3]
	\arrow[from=5-1, to=5-3]
	\arrow[from=5-1, to=6-3]
	\arrow[from=5-3, to=4-5]
	\arrow["{=}"', curve={height=-6pt}, squiggly, from=5-5, to=4-5]
	\arrow[from=5-7, to=3-9]
	\arrow[from=5-7, to=4-9]
	\arrow[from=5-7, to=5-9]
	\arrow[from=6-3, to=4-5]
	\arrow[from=6-7, to=3-9]
	\arrow[from=6-7, to=4-9]
	\arrow[from=6-7, to=5-9]
\end{tikzcd}
\end{adjustbox}
\]

\section{The Mean-Field Risk Functional}\label{sec:mean-field_risk}

Let the data distribution be $P$, then the population risk for the Autoencoder can be written as
\[
\risk_N(\theta_U, \theta_W) = \frac{1}{2} \mathbb{E}_{x\sim P} \left[ \|x - \hat{x}^{(N)}\|^2 \right]. 
\]
If we change our point of view from parameters $(\theta_U,\theta_W)$ to the discrete empirical probability measure $\hat{\rho}^{(N)}_U$ and $\hat{\rho}^{(N)}_W$, we can rewrite the risk in terms of measure as
\[
\mathcal{R}_N\big(\hat{\rho}_U^{(N)}, \hat{\rho}_W^{(N)}\big)
=
\frac{1}{2}\,
\mathbb{E}_{x\sim P} \left[ \|x - \hat{x}^{(N)}\|^2 \right].
\]
We can rewrite the risk by expanding the quadratic expression within the expectation as
\[
\mathcal{R}_N\big(\hat{\rho}_U^{(N)}, \hat{\rho}_W^{(N)}\big)
=
\underbrace{\frac{1}{2}\mathbb{E}_x\big[\|x\|^2\big]}_{\text{Term 1}}
-
\underbrace{\mathbb{E}_x\big[\langle x, \hat{x}^{(N)} \rangle\big]}_{\text{Term 2}}
+
\underbrace{\frac{1}{2}\mathbb{E}_x\big[\|\hat{x}^{(N)}\|^2\big]}_{\text{Term 3}}.
\]
Following \cite{Mei2018AMF} we would like to get hold of a risk functional $\risk(\rho_U, \rho_W)$ for any probability measure. We, thus analyze heuristically of what the function could be in the limit as one lets $N\to \infty$. 

The Term 1 is just the baseline risk and is independent of $N$ and the weights. We let 
\[
\text{Term 1}=\mathcal{R}_\#:=\frac{1}{2}\mathbb{E}_x\big[\|x\|^2\big]
.
\]
If we expand the term 2 by substituting the integral form of $\hat{x}^{(N)}$ we get,
\[
T_2^{(N)}
=
\mathbb{E}_x
\left[
\left\langle
x,
\int \psi\big(\alpha^{(N)}(x); u\big)\,
d\hat{\rho}_U^{(N)}(u)
\right\rangle
\right].
\]

If Fubini's theorem were applicable, one would hope to swap the expectation and integral and under limit if the empirical measure $\hat{\rho}_U^{(N)} \to \rho_U$, and the argument $\alpha^{(N)}(x) \to \alpha(x; \rho_W)$, we can hope that 
\[
\lim_{N \to \infty} T_2^{(N)}
\approx
\int
\underbrace{
\mathbb{E}_x
\Big[
\langle x, \psi(\alpha(x; \rho_W); u) \rangle
\Big]
}_{\equiv V(u; \alpha)}
\, d\rho_U(u).
\]
We formally define the external potential $V(u; \alpha)$ and the bottleneck potential $\alpha(x)$ (we shall suppress $\rho_W$ dependence) as
\[
V(u; \alpha):=
\frac{1}{2}\mathbb{E}_x
\Big[
\langle x, \psi(\alpha(x); u) \rangle
\Big]
\]
\[\alpha(x)=\int\phi(x;w)\,d\rho_W(w)\]
which makes the second term
\[
\text{Term 2} = 2\int V(u; \alpha)\, d\rho_U(u).
\]
If we substituting the integral form of $\hat{x}^{(N)}$ into the third term then, the squared norm of an integral can be rewritten as a double integral over a product measure as follows (assuming the integrals and expectation can be switched).
\[
T_3^{(N)}
=
\frac{1}{2}
\mathbb{E}_x
\left[
\left\langle
\int \psi\big(\alpha^{(N)}(x); u\big)\,
d\hat{\rho}_U^{(N)}(u),
\,
\int \psi\big(\alpha^{(N)}(x); u'\big)\,
d\hat{\rho}_U^{(N)}(u')
\right\rangle
\right],
\]

\[
T_3^{(N)}
=
\frac{1}{2}
\iint
\mathbb{E}_x
\Big[
\langle
\psi\big(\alpha^{(N)}(x); u\big),
\psi\big(\alpha^{(N)}(x); u'\big)
\rangle
\Big]
\, d\hat{\rho}_U^{(N)}(u)\,
d\hat{\rho}_U^{(N)}(u').
\]
If we let $\rho_U \otimes \rho_U$ denote the limiting product measure of $\hat{\rho}_U^{(N)} \otimes \hat{\rho}_U^{(N)}$, and follow as similar heuristic as in the analysis of the second term, our prospective third term should look like as below
\begin{align*}
   \text{Term 3} \approx\lim_{N \to \infty} T_3^{(N)}
&\approx
\frac{1}{2}
\iint
\mathbb{E}_x
\Big[
\langle
\psi(\alpha(x); u),
\psi(\alpha(x); u')
\rangle
\Big]\, d\rho_U(u)\, d\rho_U(u')
\\
&=
\iint
\underbrace{\frac{1}{2}
\mathbb{E}_x
\Big[
\langle
\psi(\alpha(x); u),
\psi(\alpha(x); u')
\rangle
\Big]
}_{\equiv U(u, u'; \alpha)}
\, d\rho_U(u)\, d\rho_U(u')
\end{align*}

This lets us formally define the pairwise interaction kernel $U(u, u'; \alpha)$,

\[
U(u, u'; \alpha) :=  \frac{1}{2}
\mathbb{E}_x
\Big[
\langle
\psi(\alpha(x); u),
\psi(\alpha(x); u')
\rangle
\Big].
\]

Combining the three terms, we arrive at the continuous risk functional on the space of probability measures
\[
\mathcal{R}(\rho_U, \rho_W)
=
\mathcal{R}_\#
-
2\int V(u; \alpha)\, d\rho_U(u)
+
\iint
U(u, u'; \alpha)\,
d\rho_U(u)\, d\rho_U(u').
\]

where
\[
\alpha(x)
=
\int \phi(x; w)\, d\rho_W(w).
\]

We also note that the limiting mean-field output $\hat{x}$ is given by
\[\hat{x} = \int \psi(\alpha(x); u) \, d\rho_U(u).\]

\subsection{Derivation of the Mean-Field Functional}\label{sec:derivation_mean_field}

We give a brief description of the first variations of $\risk$ with respect to the encoder and decoder measure. This will be used to obtain the Wasserstein gradient flow formulation and play a key role in determining necessary and sufficient conditions on minima of the coupled PDE (\Cref{thm:main_dynamics},\Cref{thm:optimal_characterization}).

\subsection*{Derivative with respect to the Decoder Measure ($\rho_U$)}

The bottleneck potential $\alpha(x)$ has all the $\rho_W$ dependence and is independent of the decoder measure $\rho_U$; hence we can treat it as a constant. The functional derivative $\Psi_U(u;\rho_U,\rho_W)$ of the risk $\risk$ wrt the decoder measure $\rho_U$ yields
\[
\Psi_U(u; \rho_U, \rho_W)
:=
\frac{\delta \mathcal{R}(u)}{\delta \rho_U}
=
-2V(u; \alpha)
+
2\int
U(u, u'; \alpha)\,
\rho_U(du').
\]

If we substitute the expressions for $V$ and $U$ in the expression for $\Psi_U$, we get
\begin{align*}
    \Psi_U(u; \rho_U, \rho_W) &= \frac{\delta \mathcal{R}}{\delta \rho_U}(u) = -2V(u; \alpha) + 2\int U(u, u'; \alpha) \rho_U(du') \\
    &= -\mathbb{E}_x \Big[ \langle x, \psi(\alpha(x); u)\rangle \Big] + \int \mathbb{E}_x
    \Big[
    \langle
    \psi(\alpha(x); u),
    \psi(\alpha(x); u')
    \rangle
    \Big] \,\rho_U(du')\\
    & = \mathbb{E}_x \Big[ -\langle \psi(\alpha(x); u), x\rangle + \langle\psi(\alpha(x); u), \int \psi(\alpha(x); u')\, \rho_U(du') \rangle \Big]\\
    & = \mathbb{E}_x \Big[ \langle \psi(\alpha(x); u), \int \psi(\alpha(x); u')\, \rho_U(du') -x\rangle \Big]\\
    & = \mathbb{E}_x [\langle \psi(\alpha(x);u), \hat{x}-x\rangle], 
    \end{align*}
where $\hat{x}$ is the limiting output. One way to interpret this is to say that the gradient of the risk with respect to a decoder particle $u$ is the expected inner product of the difference of the output and the input of the AE (`` the residue") with the activation of that specific particle $\psi_u$.

\subsection*{Derivative with respect to the Encoder Measure ($\rho_W$)}

The encoder affects the risk only through $\alpha$ which is a function of input $x$ and the encoder measure $\rho_W$. The variation of $\alpha$ with respect to the encoder measure is
\[
\frac{\delta \alpha}{\delta \rho_W}
=
\phi(x;w).
\]

We apply the functional chain rule through the intermediate activation $\alpha$ to calculate the functional derivative $\Psi_W(w;\rho_U,\rho_W)$ of the risk $\risk$ wrt to the encoder measure $\rho_W$ as shown below

\begin{equation*}
    \Psi_W(w;\rho_U,\rho_W):=\frac{\delta \mathcal{R}(w)}{\delta \rho_W} = \frac{\delta \mathcal{R}(w)}{\delta \alpha} \frac{\delta \alpha}{\delta \rho_W}.
\end{equation*}

We compute the variation of the risk with respect to $\alpha$. 

\begin{align*}
\delta \mathcal{R}(\alpha)
=& -2 \int \delta V(u;\alpha)\, d\rho_U(u)
+ \int\int \delta U(u,u';\alpha) \, d\rho_U(u)\, d\rho_U(u')\\
=& - \int \mathbb{E}_x \left[ 
\left\langle x, \frac{\partial \psi(\alpha; u)}{\partial \alpha} \right\rangle 
\, \delta\alpha 
\right] \, d\rho_U(u) \\
&\quad + \int \int \frac{1}{2}\mathbb{E}_x \left[ 
\left\langle \frac{\partial \psi(\alpha;u)}{\partial \alpha}, \psi(\alpha;u') \right\rangle 
+ \left\langle \psi(\alpha;u), \frac{\partial \psi(\alpha;u')}{\partial \alpha} \right\rangle 
\right] \delta\alpha \, d\rho_U(u)\, d\rho_U(u')\\
=& - \int \mathbb{E}_x \left[ 
\left\langle x, \frac{\partial \psi(\alpha; u)}{\partial \alpha} \right\rangle 
\, \delta\alpha 
\right] \, d\rho_U(u)  \\& + \int \int \mathbb{E}_x \left[  
\left\langle \psi(\alpha;u), \frac{\partial \psi(\alpha;u')}{\partial \alpha} \right\rangle  
\, \delta\alpha 
\right] \, d\rho_U(u)\, d\rho_U(u')\\
=& \mathbb{E}_x \left[ 
\left\langle -x , \int \frac{\partial \psi(\alpha; u)}{\partial \alpha} \, d\rho_U(u) \right\rangle 
\, \delta\alpha 
\right] \\&+ \mathbb{E}_x \left[ 
\left\langle \int \psi(\alpha;u) \, d\rho_U(u) , 
\int \frac{\partial \psi(\alpha;u)}{\partial \alpha} \, d\rho_U(u) \right\rangle 
\, \delta\alpha 
\right] \\
=& \mathbb{E}_x \left[ 
\left\langle -x + \int \psi(\alpha;u)\, d \rho_U(u), 
\int \frac{\partial \psi(\alpha; u)}{\partial \alpha} \, d\rho_U(u) \right\rangle 
\, \delta \alpha  
\right]\\
=& \mathbb{E}_x \left[ 
\left\langle \hat{x}(x) - x , 
\int \frac{\partial \psi(\alpha; u)}{\partial \alpha} \, d\rho_U(u) \right\rangle 
\, \delta \alpha 
\right].
\end{align*}

Putting both the part of the chain rule together gives us
\[\Psi_W(w; \rho_U, \rho_W) = \mathbb{E}_x \left[ \left \langle \hat{x} - x , \int \frac{\partial \psi(\alpha; u)}{\partial \alpha} \, d\rho_U(u) \right \rangle \phi(x;w) \right].\]

This is the continuous mean field analogue of backpropagation.

\subsection{Coupled Vlasov-McKean Equations}

The limiting dynamics of $\rho_{U,t}$ and $\rho_{W,t}$ are governed by the coupled continuity equations as given below

\[
\partial_t \rho_{U,t}(u)
=
\xi(t)
\nabla_u \cdot
\Big(
\rho_{U,t}
\nabla_u
\Psi_U(u; \rho_{U,t}, \rho_{W,t})
\Big),
\]

\[
\partial_t \rho_{W,t}(w)
=
\xi(t)
\nabla_w \cdot
\Big(
\rho_{W,t}
\nabla_w
\Psi_W(w; \rho_{U,t}, \rho_{W,t})
\Big),
\]
where the first variations of the risk $\risk$, $\Psi_U$ and $\Psi_W$ are given by 
\[
\Psi_U(u; \rho_U, \rho_W) = \mathbb{E}_x [\langle \hat{x}-x, \psi(\alpha(x);u)\rangle]
\]

\[
\Psi_W(w; \rho_U, \rho_W) = \mathbb{E}_x \left[ \left \langle \hat{x} - x , \int \frac{\partial \psi(\alpha; u)}{\partial \alpha} \, d\rho_U(u) \right \rangle \phi(x;w) \right].
\]

\section{SDG Updates}

Let us describe the forward pass and optimization steps for the discrete version of neural network. For a given input data sample $x_k$ drawn from the data distribution $P$, the intermediate bottleneck representation $\alpha_k$ is computed as the empirical average over $N$ encoder particles. This is expressed as
\[
\alpha_k = \frac{1}{N}\sum_{i=1}^N \phi(x_k; w_i^k).
\]
Using this bottleneck feature, the network produces its final output prediction $\hat{x}_k$ by averaging over $N$ decoder particles, which is given by
\[
\hat{x}_k = \frac{1}{N}\sum_{j=1}^N \psi(\alpha_k; u_j^k).
\]
The reconstruction error for this single sample is measured using the squared Euclidean distance, yielding the loss function
\[
\mathcal{L}_k = \frac{1}{2} ||\hat{x}_k - x_k||^2.
\]

To minimize this loss, we use  stochastic gradient descent. At each step, we apply a time-dependent step size $s_k = \epsilon \xi(k\epsilon)$. By computing the gradients of the loss function with respect to the network parameters, we can define the update rules. The parameters of the decoder are updated by moving against the gradient, resulting in the equation
\[
u_{j}^{ k+1} = u_j^k - s_k \big(\nabla_u \psi(\alpha_k; u_j^k)\big)^\top (\hat{x}_k - x_k).
\]
Similarly, the update for the encoder parameters is derived by applying the chain rule through the bottleneck and the decoder layer. The resulting update step for the encoder particles is
\[
w_i^{k+1} = w_i^k - s_k \big(\nabla_w \phi(x_k; w_i^{k})\big)^\top \left[ \frac{1}{N} \sum_{j=1}^N \nabla_\alpha \psi(\alpha_k; u_j^{k}) \right]^\top (\hat{x}_k - x_k).
\]

%% file: sections/further_questions.tex
\section{Preliminaries}
In this section, we recall and prove some of the results that we shall use in proving the main theorems. We first talk about the bounded Lipschitz metric between two probability measure. The \textbf{bounded Lipschitz metric} between two probability measures $\mu$ and $\nu$ is defined as
\[
d_{BL}(\mu, \nu) := \sup_{\substack{\norm{h}_\infty \le 1 \\ \norm{h}_{\operatorname{Lip}} \le 1}} \left| \int h(z) \mu(dz) - \int h(z) \nu(dz) \right|.
\]

\begin{lemma}[Folklore] \label{lem:BL_scaling}
If a specific test function $g$ is bounded by a constant $A$ (so $\norm{g}_\infty \le A$) and is $L$-Lipschitz (so $\norm{g}_{\operatorname{Lip}} \le L$), then we can say that
\[
\left| \int g(z) d(\mu - \nu)(z) \right| \le \max(A, L) \cdot d_{BL}(\mu, \nu).
\]
\end{lemma}

\begin{proof}
Let $K = \max(A, L)$. We define a new, rescaled function $\tilde{g}(x) := \frac{g(x)}{K}$. 
Now, the maximum value of $\tilde{g}$ is $\frac{A}{\max(A, L)} \le 1$, so $\norm{\tilde{g}}_\infty \le 1$.
Similarly, the Lipschitz constant of $\tilde{g}$ is $\frac{L}{\max(A, L)} \le 1$, so $\norm{\tilde{g}}_{\operatorname{Lip}} \le 1$.
Because $\tilde{g}$ satisfies both conditions, it is a valid candidate for the supremum in the definition of $d_{BL}$. Therefore, we can conclude that
\[
\left| \int \tilde{g}(z) d(\mu - \nu)(z) \right| \le d_{BL}(\mu, \nu)
\]
Substituting $\tilde{g}(x) = g(x)/K$ back in and multiplying both sides by $K$ yields the final result.
\end{proof}

We next define the \textbf{Wasserstein metric} between two probability measures as 
\[W_p(\mu, \nu) := \inf{\substack{\gamma\in \mathcal{C}(\mu, \nu )}} \left( \int \norm{x-y}^p_2 \, \gamma(dx,dy) \right)^{1/p}\]
where $\mathcal{C}(\mu,\nu)$ is the set of coupling of $\mu$ and $\nu$. By Kantorovich-Rubinstein one has the following for the case $p=1$.
\[W_1(\mu, \nu) := \sup_{\substack{ \norm{h}_{\operatorname{Lip}} \le 1}} \left| \int h(z) \mu(dz) - \int h(z) \nu(dz) \right|\]
As the set of all functions considered in the expression for $d_{BL}$ is subset of that of $W_1$, we conclude that
\[d_{BL}(\mu,\nu)\le W_1(\mu,\nu).\]
 Now, we state a result about $W_1$ norm of two measures when one measure is the push-forward of the other measure.
\begin{lemma}[Wasserstein Bound of a Push-forward Measure] \label{lem:pushforward_bound}
Let $\mu$ be a probability measure on $\mathbb{R}^d$, and let $T: \mathbb{R}^d \to \mathbb{R}^d$ be a measurable map. Let $\nu = T_\#\mu$ denote the push-forward measure of $\mu$ under $T$. Then, the 1-Wasserstein distance between $\mu$ and $\nu$ are bounded by 
\[
d_{BL}(\mu, \nu) \le W_1(\mu, \nu) \le \sup_{x \in \mathrm{supp}(\mu)} \norm{x - T(x)}_2.
\]
\end{lemma}

\begin{proof}
First, we prove the bound on the 1-Wasserstein distance. By definition,
\[
W_1(\mu, \nu) = \inf_{\pi \in \Pi(\mu, \nu)} \int_{\mathbb{R}^d \times \mathbb{R}^d} \norm{x - y}_2 \, d\pi(x, y),
\]
where $\Pi(\mu, \nu)$ is the set of all joint probability measures $\pi(x,y)$ with marginals $\mu$ and $\nu$. 

Because $\nu = T_\#\mu$ is a deterministic push-forward, we can construct a specific deterministic coupling $\hat{\pi} \coloneqq (\operatorname{id}, T)_\# \mu$. For any measurable function $f(x, y)$, integrating with respect to $\hat{\pi}$ is equivalent to evaluating $f(x, T(x))$ over the base measure $\mu$. So, we can say that
\[
\int_{\mathbb{R}^d \times \mathbb{R}^d} f(x, y) \, d\hat{\pi}(x, y) = \int_{\mathbb{R}^d} f(x, T(x)) \, d\mu(x).
\]
By setting $f(x,y) = f(x)$, it is clear the first marginal is $\mu$. By setting $f(x,y) = g(y)$, the second marginal is $\int g(T(x)) d\mu(x) = \int g(y) d\nu(y) = \nu$. Thus, $\hat{\pi} \in \Pi(\mu, \nu)$ is a valid coupling. 

Now, $W_1$ is defined as an infimum, it must be less than or equal to the cost of our specific coupling $\hat{\pi}$. So, we can conclude that
\begin{align*}
W_1(\mu, \nu) &\le \int_{\mathbb{R}^d \times \mathbb{R}^d} \norm{x - y}_2 \, d\hat{\pi}(x, y) \\
&= \int_{\mathbb{R}^d} \norm{x - T(x)}_2 \, d\mu(x).
\end{align*}
Now we know that a function is bounded by its  supremum, so pulling the supremum outside the integral gives us
\begin{align*}
\int_{\mathbb{R}^d} \norm{x - T(x)}_2 \, d\mu(x) &\le \int_{\mathbb{R}^d} \left( \sup_{z \in \mathrm{supp}(\mu)} \norm{z - T(z)}_2 \right) d\mu(x) \\
&= \sup_{z \in \mathrm{supp}(\mu)} \norm{z - T(z)}_2 \int_{\mathbb{R}^d} 1 \, d\mu(x) \\
&= \sup_{z \in \mathrm{supp}(\mu)} \norm{z - T(z)}_2.
\end{align*}
This proves the right-hand inequality $W_1(\mu, \nu) \le \sup_{x} \norm{x - T(x)}_2$.

As we know that
\[
d_{BL}(\mu, \nu) \le W_1(\mu, \nu).
\]
Combining both results completes the proof.
\end{proof}
We now state a version of Taylor's theorem for decoders that is used in the proof of the main theorems.

\begin{lemma}[Taylor's Theorem for decoder part]
Let $\psi(\alpha;u)$ be twice continuously differentiable with respect to $\alpha \in \mathbb{R}$. For any base point $\alpha^*$, step $\Delta \alpha$, and scalar $\lambda \in [0,1]$, the exact second-order expansion is
\[
\psi(\alpha^* + \lambda\Delta\alpha; u)
= \psi(\alpha^*; u)
+ \lambda\,\nabla_\alpha \psi(\alpha^*; u)\, \Delta\alpha
+ \lambda^2 \int_0^1 (1-s)\,\nabla_\alpha^2 \psi\big(\alpha^* + s \lambda \Delta\alpha; u\big)\, (\Delta\alpha)^2 \, ds.
\]
\end{lemma}

\begin{proof}
 Let $v = \lambda \Delta \alpha$. Now, we define the function $g: [0, 1] \to \mathbb{R}^d$ as
\[
g(s) = \psi(\alpha^* + s v. u).
\]
Now by the chain rule, we can say that the derivatives of $g(s)$ with respect to $s$ are $g'(s) = \nabla_\alpha \psi(\alpha^* + s v; u) v$ and $g''(s) = \nabla_\alpha^2 \psi(\alpha^* + s v; u) v^2$. 

So, by the \textit{fundamental theorem of calculus} we can conclude that
\[
g(1) - g(0) = \int_0^1 g'(s) \, ds.
\]
We apply integration by parts to the right-hand side. After applying integration by parts to the right-hand side we get
\begin{align*}
\int_0^1 g'(s) \, ds 
&= \Big[ -g'(s)(1 - s) \Big]_0^1 - \int_0^1 \big(-(1 - s)\big) g''(s) \, ds \\
&= g'(0) + \int_0^1 (1 - s) g''(s) \, ds.
\end{align*}
So,
\[g(1) = g(0) + g'(0) + \int_0^1 (1 - s) g''(s) \, ds.\]
Now, by substituting the definitions of $g(0), g(1), g'(0)$ and $g''(s)$ back into this equation and factoring out $\lambda$ and $\lambda^2$, we exactly yield the stated lemma.
\end{proof}

We reproduce the Azuma-H\"{o}effding inequality (see Lemma A.1 in \cite{Mei2019MeanfieldTO}).

\begin{lemma}[Azuma-H\"{o}effding bound] \label{lem:azuma-hoeffding}
Let $(X_k)_{k \ge 0}$ be a martingale taking values in $\mathbb{R}^d$ with respect to the filtration $(\mathcal{F}_k)_{k \ge 0}$, with $X_0 = 0$. Assume that the following holds almost surely for all $k \ge 1$
\begin{equation}
\mathbb{E}\!\left[ e^{\langle \lambda, X_k - X_{k-1} \rangle} \,\middle|\, \mathcal{F}_{k-1} \right]
\le e^{L^2 \|\lambda\|^2 / 2}.
\end{equation}
Then we have
\begin{equation}
\mathbb{P}\!\left( \max_{k \le n} \|X_k\|_2 \ge 2L\sqrt{n}(\sqrt{d} + t) \right)
\le e^{-t^2}.
\end{equation}
\end{lemma}

Although $d_{BL}$ and $W_p$ metric are key part of the analysis, to rigorously measure the distance between the empirical and population measures , we would like to evaluate the measures strictly through the lens of the network's predictive functions. We formalize the general integral probability metric induced by such a parametric family.

Let $\mathcal{F} = \{ f_\theta \}_{\theta \in \Theta}$ be a family of test functions mapping a domain $\Omega$ to $\mathbb{R}^p$, indexed by a parameter set $\Theta$. For any two probability measures $\mu$ and $\nu$ defined on $\Omega$, the \textbf{parametric distance} $d_{\mathcal{F}}$ between $\mu$ and $\nu$ is defined as the supremum of the Euclidean distance between their expectations over all functions in the class $\mathcal{F}$
$$
d_{\mathcal{F}}(\mu, \nu) = \sup_{f \in \mathcal{F}} \norm{ \int_\Omega f(\omega)\, d\mu(\omega) - \int_\Omega f(\omega)\, d\nu(\omega) }_2
$$
For scalar-valued function classes ($p=1$), the Euclidean norm simplifies to the standard absolute value. We now prove a result from empirical process theory that bounds the statistical concentration of the specific parametric functions evaluated by our network. 

\begin{lemma}[Parametric Concentration of Empirical Measures] \label{lem:empirical_concentration}
Let $\mu$ be a probability measure supported on a compact parameter space $\Omega \subset \mathbb{R}^D$. Let $X_1, \dots, X_N \sim \mu$ be i.i.d., and define the empirical measure $\mu_N = \frac{1}{N} \sum_{i=1}^N \delta_{X_i}$. 
Let $\mathcal{F} = \{ f_\theta : \Omega \to \mathbb{R}^p \}_{\theta \in \Theta}$ be a family of test functions indexed by a compact set $\Theta \subset \mathbb{R}^d$ (e.g., the data domain) with diameter $R$. 
Assume that for all $\theta \in \Theta$, $\sup_{w \in \Omega} \norm{f_\theta(w)}_2 \le M$, and the map $\theta \mapsto f_\theta(w)$ is $L$-Lipschitz for all $w \in \Omega$ (i.e., $\norm{f_\theta(w)-f_{\theta'}(w)}_2 \le L\norm{\theta-\theta'}_2$).

Then, for any $\delta \in (0,1)$, with probability at least $1 - \delta$,
$$
d_{\mathcal{F}}(\mu, \mu_N) \le \frac{1}{\sqrt{N}} \left( C + M \sqrt{d \log N} + M\sqrt{2 \log(1/\delta)} \right),
$$
where $C := 2L + M\sqrt{2d \log(3R)} + M\sqrt{2 \log 2}$.
\end{lemma}

\begin{proof}
Let us first note that the parametric distance $d_{\mathcal{F}}$ for $\mu$ and $\mu_N$ reduces to
$$
d_{\mathcal{F}}(\mu, \mu_N) = \sup_{\theta \in \Theta} \norm{ \mathbb{E}_\mu[f_\theta] - \frac{1}{N}\sum_{i=1}^N f_\theta(X_i) }_2.
$$
We aim to bound the uniform deviation over the parametric class $\mathcal{F}$ for which we need to construct a cover of the finite-dimensional index set $\Theta \subset \mathbb{R}^d$. Let $\Theta_\epsilon$ be an $\epsilon$-cover of $\Theta$ in the Euclidean norm. Standard volumetric arguments guarantee that the covering number is bounded by $|\Theta_\epsilon| \le (3R/\epsilon)^d$. 
For any $\theta \in \Theta$, there exists a $\theta' \in \Theta_\epsilon$ such that $\norm{\theta - \theta'}_2 \le \epsilon$. By the $L$-Lipschitz property of the test functions with respect to the index $\theta$, we can decompose the error using the triangle inequality
\begin{align*}
    \norm{ \mathbb{E}_\mu[f_\theta] - \frac{1}{N}\sum_{i=1}^N f_\theta(X_i) }_2 
    &\le \norm{ \mathbb{E}_\mu[f_\theta] - \mathbb{E}_\mu[f_{\theta'}] }_2 + \norm{ \mathbb{E}_\mu[f_{\theta'}] - \frac{1}{N}\sum_{i=1}^N f_{\theta'}(X_i) }_2\\
    &\quad + \norm{ \frac{1}{N}\sum_{i=1}^N f_{\theta'}(X_i) - \frac{1}{N}\sum_{i=1}^N f_\theta(X_i) }_2 \\
    &\le 2L\epsilon + \norm{ \mathbb{E}_\mu[f_{\theta'}] - \frac{1}{N}\sum_{i=1}^N f_{\theta'}(X_i) }_2.
\end{align*}
Taking the supremum over all $\theta \in \Theta$, the worst-case continuous error is bounded by the worst-case error on the finite cover
$$
\sup_{\theta \in \Theta} \norm{ \mathbb{E}_\mu[f_\theta] - \frac{1}{N}\sum_{i=1}^N f_\theta(X_i) }_2 \le 2L\epsilon + \max_{\theta' \in \Theta_\epsilon} \norm{ \mathbb{E}_\mu[f_{\theta'}] - \frac{1}{N}\sum_{i=1}^N f_{\theta'}(X_i) }_2.
$$

For a fixed $\theta' \in \Theta_\epsilon$, the random vectors $f_{\theta'}(X_i) \in \mathbb{R}^p$ are independent, identically distributed, and bounded in Euclidean norm by $M$. Applying the vector Hoeffding's inequality for this specific $\theta'$
$$
\mathbb{P}\left( \norm{ \mathbb{E}_\mu[f_{\theta'}] - \frac{1}{N}\sum_{i=1}^N f_{\theta'}(X_i) }_2 \ge \tau \right) \le 2 \exp\left( - \frac{N\tau^2}{2M^2} \right).
$$

We require this to hold simultaneously for all elements in the cover. Applying a union bound over $\Theta_\epsilon$, the probability that the maximum fluctuation exceeds $\tau$ is
$$
\mathbb{P}\left( \max_{\theta' \in \Theta_\epsilon} \norm{ \mathbb{E}_\mu[f_{\theta'}] - \frac{1}{N}\sum_{i=1}^N f_{\theta'}(X_i) }_2 \ge \tau \right) \le 2 \left( \frac{3R}{\epsilon} \right)^d \exp\left( - \frac{N\tau^2}{2M^2} \right).
$$
Setting the right-hand side equal to $\delta$ and solving for $\tau$ gives us
$$
\tau = M \sqrt{\frac{2}{N}} \sqrt{ d \log\left(\frac{3R}{\epsilon}\right) + \log\left(\frac{2}{\delta}\right) }.
$$
The total uniform error is bounded by $2L\epsilon + \tau$. To optimally balance the deterministic approximation error and the stochastic error, we choose the covering radius $\epsilon = 1/\sqrt{N}$. 
Substituting this $\epsilon$, we observe that $\log(3R/\epsilon) = \log(3R) + \frac{1}{2}\log N$. 

Using the algebraic subadditivity property $\sqrt{a+b} \le \sqrt{a} + \sqrt{b}$ repeatedly, we can expand $\tau$
\begin{align*}
\tau &\le M \sqrt{\frac{2}{N}} \left( \sqrt{\frac{d}{2} \log N} + \sqrt{d \log(3R)} + \sqrt{\log 2} + \sqrt{\log(1/\delta)} \right) \\
&= \frac{1}{\sqrt{N}} \left( M \sqrt{d \log N} + M\sqrt{2d \log(3R)} + M\sqrt{2 \log 2} + M\sqrt{2 \log(1/\delta)} \right).
\end{align*}
Adding the $2L\epsilon = \frac{2L}{\sqrt{N}}$ term, we consolidate the static components to explicitly define the absolute constant $C$
$$
C := 2L + M\sqrt{2d \log(3R)} + M\sqrt{2 \log 2}.
$$
Thus, the total bound isolates exactly as
$$
d_{\mathcal{F}}(\mu, \mu_N) \le \frac{1}{\sqrt{N}} \left( C + M \sqrt{d \log N} + M\sqrt{2 \log(1/\delta)} \right).
$$
\end{proof}

\begin{lemma}[Relationship between Parametric and Bounded Lipschitz Distances] \label{lem:parametric_vs_bl}
Let $\mathcal{F}$ be a family of test functions defined on a domain $\Omega$. Assume that there exists a constant $M > 0$ such that every function $f \in \mathcal{F}$ is uniformly bounded by $M$ ($\sup_{x \in \Omega} \norm{f(x)}_2 \le M$) and is $M$-Lipschitz continuous ($\norm{f(x) - f(y)}_2 \le M \norm{x - y}_2$). 

Then, for any two probability measures $\mu$ and $\nu$ on $\Omega$, the parametric distance is bounded by the scaled Bounded Lipschitz distance.
\[
d_{\mathcal{F}}(\mu, \nu) \le M d_{\mathrm{BL}}(\mu, \nu)
\]
\end{lemma}

\begin{proof}
We begin by evaluating the integral difference for a single, fixed function $f \in \mathcal{F}$. Because $f$ maps to $\mathbb{R}^p$, its integral difference is a vector. The Euclidean norm of any vector can be expressed as the supremum of its inner product with all unit vectors.
\begin{align*}
    \norm{ \int f d\mu - \int f d\nu }_2 & = \sup_{\norm{v}_2 \le 1} \left\langle v, \int f d\mu - \int f d\nu \right\rangle \\& = \sup_{\norm{v}_2 \le 1} \left( \int \langle v, f(x) \rangle d\mu(x) - \int \langle v, f(x) \rangle d\nu(x) \right)
\end{align*}
Fix an arbitrary unit vector $v \in \mathbb{R}^p$ such that $\norm{v}_2 \le 1$, and define the scalar projection function $h(x) = \langle v, f(x) \rangle$. We evaluate the boundedness and Lipschitz properties of $h$. By the Cauchy-Schwarz inequality and our assumptions on $\mathcal{F}$, the function $h$ is bounded.
\[
|h(x)| \le \norm{v}_2 \norm{f(x)}_2 \le 1 \cdot M = M
\]
Similarly, $h$ inherits the Lipschitz continuity of $f$.
\[
|h(x) - h(y)| = |\langle v, f(x) - f(y) \rangle| \le \norm{v}_2 \norm{f(x) - f(y)}_2 \le M \norm{x - y}_2
\]
We now define the normalized scalar function $g(x) = h(x) / M$. Because $h$ is $M$-bounded and $M$-Lipschitz, $g$ is bounded by $1$ ($\|g\|_\infty \le 1$) and is $1$-Lipschitz ($\|g\|_{\mathrm{Lip}} \le 1$). This means $g$ perfectly satisfies the admissibility conditions for the definition of the Bounded Lipschitz metric. Therefore, the integral difference over $g$ is bounded by $d_{\mathrm{BL}}$.
\[
\int g(x) d\mu(x) - \int g(x) d\nu(x) \le d_{\mathrm{BL}}(\mu, \nu)
\]
Substituting $g(x) = \langle v, f(x) \rangle / M$ back into the inequality and multiplying both sides by $M$ yields the bound for our projection.
\[
\int \langle v, f(x) \rangle d\mu(x) - \int \langle v, f(x) \rangle d\nu(x) \le M d_{\mathrm{BL}}(\mu, \nu)
\]
Because this inequality holds for every unit vector $v$, it must hold for the supremum over all unit vectors, which implies that $\norm{\int f d\mu - \int f d\nu}_2 \le M d_{\mathrm{BL}}(\mu, \nu)$. Finally, taking the supremum over all functions $f \in \mathcal{F}$ yields the target result.
\[
\sup_{f \in \mathcal{F}} \norm{ \int f d\mu - \int f d\nu }_2 \le M d_{\mathrm{BL}}(\mu, \nu)
\]
\end{proof}

\section{Proofs of the main statements} 
\subsection{Technical Lemma}\label{subsec:technical_lemma}
We state and prove a few of the technical lemmas that are specific to our setting of bottleneck AE which are later used in the proof of the main theorems.
\begin{lemma}[Uniform Boundedness and Lipschitz Continuity] \label{lem:uniform_bounds}
Under Assumptions \ref{it:bounded_enocder_decoder-domain}, \ref{it:bounded_data-domain}, and \ref{it:bounded_activation}, there exists a single absolute constant $M > 1$ such that:
\begin{enumerate}[label=(\roman*)]
    \item The data is bounded i.e. $\norm{x}_2 \le M$ almost surely for $x \sim P$.
    \item The mean-field bottleneck feature is bounded i.e. $|\alpha|\leq M$ almost surely for $x \sim P$
    \item The network activations are uniformly bounded: $|\phi(x;w)| \le M$ and $\|\psi(\alpha;u)\|_2 \le M$.
    \item The spatial derivatives of the decoder up to the second order are bounded: $\|\nabla_\alpha\psi(\alpha;u)\|_{\mathrm{op}} \le M$ and $\|\nabla_\alpha^2 \psi(\alpha;u)\|_{\mathrm{op}} \le M$.
    \item The encoder and decoder particles are $M$-Lipschitz continuous with respect to their parameters:
    \[
    |\phi(x;w_1) - \phi(x';w_2)| \le M \left(\norm{w_1 - w_2}_2+\norm{x-x'}_{2}\right)\]   and  
    \[\norm{\psi(\alpha;u_1) - \psi(\alpha';u_2)}_2 \le M \left(\norm{u_1 - u_2}_2+\norm{\alpha-\alpha'}_{2}\right).
    \]
    \item The derivatives $\nabla_{u}\psi(\alpha;u)$ and $\nabla_w\phi(x;w)$ are $M$-Lipschitz continuous with respect to their parameters.
     
\end{enumerate}
\end{lemma}

\begin{proof}
The proof relies on the Extreme Value Theorem, which states that any continuous function on a compact (closed and bounded) set achieves a finite maximum.

To give bounds on the data and parameter (i), we first note that by Assumption \ref{it:bounded_data-domain}, the support of the data distribution $P$ is bounded. Without loss of generality, we can take the support of $P$ to be compact (e.g., by taking its closure). Thus, there exists a constant $M_x$ such that $\|x\|_2 \le M_x$. 
Similarly, by Assumption \ref{it:bounded_enocder_decoder-domain}, $\Omega_W$ and $\Omega_U$ are compact. Hence, the parameter vectors $w = (w_1, w_2)$ and $u = (u_1, u_2)$ are bounded. There exist constants $M_{w1}, M_{w2}, M_{u1}, M_{u2}$ bounding the norms of these respective components.

To prove (ii) and (iii), recall that the encoder activation is $\phi(x; w) = w_2 \sigma(\langle w_1, x \rangle)$. Since $x \in \operatorname{supp}(P)$ and $w_1 \in \Omega_W$ are compact sets, their inner product $\langle w_1, x \rangle$ maps to a bounded closed interval in $\mathbb{R}$. Because $\sigma$ is continuous (Assumption \ref{it:bounded_activation}), $\sigma(\langle w_1, x \rangle)$ attains a finite maximum. Multiplied by the bounded scalar $w_2$, it follows that there exists $M_\phi > 0$ such that $|\phi(x; w)| \le M_\phi$.

Because the encoder output is bounded, the mean-field bottleneck feature $\alpha = \int \phi(x; w) \rho_W(dw)$ is also bounded, $|\alpha| \le \int M_\phi \rho_W(dw) = M_\phi$. Thus, the effective latent space $\mathcal{H} := [-M_\phi, M_\phi]$ is a compact set.

The decoder activation is $\psi(\alpha; u) = u_1 \sigma(u_2 \alpha)$. Since $\alpha \in \mathcal{H}$ and $u_2$ are bounded, the argument $u_2 \alpha$ is bounded. Again, by the continuity of $\sigma$, $\sigma(u_2\alpha)$ is bounded, and thus there exists $M_\psi > 0$ such that $\|\psi(\alpha; u)\|_2 \le M_\psi$.

Now we give the bound of the spatial derivatives (iv). By Assumption \ref{it:bounded_activation}, $\sigma$ is $C^2$, meaning $\sigma'$ and $\sigma''$ are continuous. Differentiating the decoder activation with respect to the bottleneck $\alpha$ yields,
$$\nabla_\alpha \psi(\alpha; u) = u_1 u_2 \sigma'(u_2 \alpha), \qquad \nabla_\alpha^2 \psi(\alpha; u) = u_1 u_2^2 \sigma''(u_2 \alpha).$$
Since the arguments $u_2 \alpha$ belong to a compact set, the continuous functions $\sigma'$ and $\sigma''$ achieve finite maximums on this domain. Combined with the bounds on $u_1$ and $u_2$, there exist finite constants $M_{\psi'}$ and $M_{\psi''}$ bounding the operator norms of these derivatives.

To show Lipschitz continuity (v), it suffices to show that the gradients with respect to the variables are uniformly bounded. 
For the encoder, the gradient with respect to $w$ consists of $\nabla_{w_1} \phi = w_2 x \sigma'(\langle w_1, x \rangle)$ and $\nabla_{w_2} \phi = \sigma(\langle w_1, x \rangle)$. Because $x, w_1, w_2$ are bounded and $\sigma, \sigma'$ are continuous on compact domains, these parameter gradients are bounded by some constant $L_W$. By the Mean Value Theorem, $\phi$ is $L_W$-Lipschitz with respect to $w$.
By identical logic, the gradients of $\psi(\alpha;u)$ with respect to $u_1$ and $u_2$ consist of $\sigma(u_2 \alpha)$ and $u_1 \alpha \sigma'(u_2 \alpha)$. These are bounded over the compact domain $\mathcal{H} \times \Omega_U$, yielding a Lipschitz constant $L_U$.

To complete the proof of (v), we must also establish Lipschitz continuity with respect to the spatial variables $x$ and $\alpha$. The spatial gradient of the encoder is $\nabla_x \phi(x;w) = w_2 \sigma'(\langle w_1, x \rangle) w_1$. Since $w_1, w_2$ are bounded and $\sigma'$ is continuous on a compact domain, this gradient is bounded by a constant $L_x$. For the decoder, we already established in (iii) that the operator norm of $\nabla_\alpha \psi(\alpha;u)$ is bounded by $M_{\psi'}$. By the Mean Value Theorem and the triangle inequality, the joint variations are bounded by the sum of the partial variations,
\begin{align*}
    |\phi(x;w_1) - \phi(x';w_2)| &\le L_W \|w_1 - w_2\|_2 + L_x \|x - x'\|_2 \\&\le \max\{L_W, L_x\} \left(\|w_1 - w_2\|_2 + \|x - x'\|_2\right).
\end{align*}
An identical argument holds for $\psi$ using $L_U$ and $M_{\psi'}$, fully satisfying condition (iv).

To prove (vi), we must show that the parameter gradients $\nabla_w \phi(x;w)$ and $\nabla_u \psi(\alpha;u)$ are themselves Lipschitz continuous with respect to their parameters. This is equivalent to showing that the Hessians $\nabla_w^2 \phi(x;w)$ and $\nabla_u^2 \psi(\alpha;u)$ are bounded. 
For the encoder, the second-order partial derivatives with respect to $w$ are,
$$\nabla_{w_1}^2 \phi = w_2 \sigma''(\langle w_1, x \rangle) x x^\top, \qquad \nabla_{w_1 w_2}^2 \phi = \sigma'(\langle w_1, x \rangle) x, \qquad \nabla_{w_2}^2 \phi = 0.$$
For the decoder, the second-order partial derivatives with respect to $u$ consist of,
$$\nabla_{u_1}^2 \psi = 0, \qquad \nabla_{u_1 u_2}^2 \psi = \alpha \sigma'(u_2 \alpha) I, \qquad \nabla_{u_2}^2 \psi = u_1 \alpha^2 \sigma''(u_2 \alpha).$$
By Assumption \ref{it:bounded_activation}, $\sigma$ is twice continuously differentiable ($C^2$), so both $\sigma'$ and $\sigma''$ are continuous. Because all inputs $x, w, \alpha,$ and $u$ belong to compact sets, every term in these Hessian matrices is a continuous function evaluated on a compact domain. Consequently, by the Extreme Value Theorem, the operator norms of these Hessians are uniformly bounded by finite constants $L_{\nabla W}$ and $L_{\nabla U}$. Thus, the parameter gradients are Lipschitz continuous.

We have established finite bounds $M_x, M_\phi, M_\psi, M_{\psi'}, M_{\psi''}, L_W, L_x, L_U, L_{\nabla W},$ and $L_{\nabla U}$. By simply defining $M$ as the maximum of all these constants,
$$M := \max \{ M_x, M_\phi, M_\psi, M_{\psi'}, M_{\psi''}, L_W, L_x, L_U, L_{\nabla W}, L_{\nabla U} \},$$
we guarantee that all conditions (i) through (vi) hold simultaneously under the single uniform constant $M$.
\end{proof}

\begin{lemma}[\textbf{Properties of the Mean-Field Risk}]\label{lem:bounded_properties}
Assume that Assumption \ref{it:bounded_enocder_decoder-domain}, \ref{it:bounded_data-domain}, and \ref{it:bounded_activation} hold.  Let $(\rho_U, \rho_W)$ be any probability measures over $\Omega_U \times \Omega_W$. $\alpha(x) = \int \phi(x; w) \rho_W(dw)$ is the feature representation. Then, there exist absolute constants $L_f, L_g, K_U, K_W > 0$ such that
\begin{enumerate}[label = (\roman*)]
    \item \textbf{Bounded Hessians:} The functions $f(\alpha) = \frac{1}{2}\norm{x - \int \psi(\alpha; u) \rho_U(du)}^2$ and $g(\alpha) = \int \norm{\psi(\alpha; u)}^2 \rho_U(du)$ have double derivatives (Hessians) uniformly bounded by $3M^2$ and $4M^2$, respectively.
    \item \textbf{Decoder bounded variance:} Let $K_U(x) := \int \norm{\psi(\alpha(x); u)}^2 \rho_U(du)$. Then $\E_x \left[ K_U(x) \right] \le M^2 $.
    \item \textbf{Encoder bounded variance:} Let $K_W(x) := \int \phi(x; w)^2 \rho_W(dw)$. Then $\E_x \left[ K_W(x) \right] \le M^2$.
\end{enumerate}
\end{lemma}
\begin{proof}

 By using \Cref{lem:uniform_bounds}  we can say that there exists a constant $M < \infty$ such that $|\phi(x; w)| \le M$ for all $x \in \operatorname{supp}(P), w \in \Omega_W$. So this implies that the feature representation $\alpha(x) = \int \phi(x; w) \rho_W(dw)$ is bounded by $M$ because
\[
|\alpha(x)| \le \int |\phi(x; w)| \rho_W(dw) \le \int M \rho_W(dw) = M.
\]
Note that $\alpha(x)$ is a continuous function, which implies that $|\alpha(x)|$ is a continuous function. So we can conclude that, the set $\mathcal{A} := \{ \alpha : |\alpha| \le M \}$, is also a compact set because it is the continuous image of  a compact set. 

Now, as we assume that activation function $\sigma$ is $C^2$ and decoder space $\Omega_U$ is compact, which implies $\psi(\alpha; u)$ is $C^2$ on the compact product space $\mathcal{A} \times \Omega_U$. By using \Cref{lem:uniform_bounds}, we can say that the function and its partial derivatives with respect to $\alpha$ are uniformly bounded by $M$. Particularly, we can say that
\[
\norm{\psi(\alpha; u)} \le M, \quad \norm{\nabla_\alpha \psi(\alpha; u)} \le M, \quad \norm{\nabla_\alpha^2 \psi(\alpha; u)} \le M .
\]

We first give a proof for part (iii) (Encoder bounded variance). Using the bound on $\phi$, we can say that
\[
K_W(x) = \int \phi(x; w)^2 \rho_W(dw) \le \int M^2 \rho_W(dw) = M^2.
\]
Now, by taking the expectation over the data distribution $x \sim P$, we can conclude that
\[
\E_x [K_W(x)] \le M^2.
\]

Now, we give a proof for part (ii) (Decoder bounded variance).
Similarly, for the encoder, using the uniform bound on $\psi$, we can say that
\[
K_U(x) = \int \norm{\psi(\alpha(x); u)}^2 \rho_U(du) \le \int M^2 \rho_U(du) = M^2.
\]
Now, by taking the expectation over the data distribution $x \sim P$, we can conclude that
\[
\E_x [K_U(x)] \le M^2.
\]

Now, we consider part (i) Bounded Hessians.
We  first consider
\[
g(\alpha) = \int \|\psi(\alpha; u)\|^2 \,\rho_U(du)
= \int \sum_{k=1}^d \psi_k(\alpha; u)^2 \,\rho_U(du),
\]
$\psi_k(\alpha; u)$ is the k-th coordinate of the function $\psi(\alpha; u)$. Now, by differentiating under the integral sign (justified by  Leibniz's rule as derivatives are uniformly bounded on compact sets), we obtain
\[
g'(\alpha)
= 2 \int \sum_{k=1}^d \psi_k(\alpha; u)\, \psi_k'(\alpha; u)\,\rho_U(du),
\]
and
\begin{align*}
 g''(\alpha)
&= 2 \int \sum_{k=1}^d \left[
(\psi_k'(\alpha; u))^2
+ \psi_k(\alpha; u)\, \psi_k''(\alpha; u)
\right] \rho_U(du)\\
&=2 \int \norm{\psi'(\alpha; u)}^2\rho_U(du)+2 \int \langle \psi(\alpha; u), \psi_k''(\alpha; u)\rangle \rho_U(du)\\
&\overset{(a)}{\leq} 2 \int M^2\rho_U(du)+2 \int \norm{\psi(\alpha; u) }\norm{\psi''(\alpha; u)}\\
&\leq 2M^2+ 2\int M^2 \rho_U(du)= 4 M^2 < \infty,
\end{align*}
where $(a)$ follows from $\norm{\psi(\alpha; u)} \le M$ and Cauchy Schwarz inequality, $(b)$ follows from $\norm{\nabla_\alpha \psi(\alpha; u)} \le M, \norm{\nabla_\alpha^2 \psi(\alpha; u)} \le M$.

Now, we consider $f(\alpha) = \frac{1}{2}\norm{x - \bar{\psi}(\alpha)}^2$, where we define $\bar{\psi}(\alpha) := \int \psi(\alpha; u) \rho_U(du)$. 
Note that, by Jensen's inequality we can say that
\[
\|\bar{\psi}(\alpha)\|
= \left\| \int \psi(\alpha;u)\,\rho_U(du) \right\|
\le \int \|\psi(\alpha;u)\|\,\rho_U(du)
\le M.
\]
Now, by differentiating under the integral sign (justified by  Leibniz's rule as derivatives are uniformly bounded on compact sets) we can say that
\[
\nabla_\alpha \bar{\psi}(\alpha)
= \int \nabla_\alpha \psi(\alpha;u)\,\rho_U(du),
\quad
\nabla_\alpha^2 \bar{\psi}_k(\alpha)
= \int \nabla_\alpha^2 \psi_k(\alpha;u)\,\rho_U(du),
\]
which further imply that
\[
\|\nabla_\alpha \bar{\psi}(\alpha)\| \le M,
\quad
\|\nabla_\alpha^2 \bar{\psi}_k(\alpha)\| \le M 
\]
because
\[\|\nabla_\alpha \bar{\psi}(\alpha)\|
= \left\| \int \nabla_\alpha \psi(\alpha)\,\rho_U(du) \right\|
\le \int \|\nabla_\alpha \psi(\alpha)\|\,\rho_U(du)
\le M\] and
\[\|\nabla_\alpha^2 \bar{\psi}_k(\alpha)\|
= \left\| \int \nabla^2_\alpha \psi(\alpha)\,\rho_U(du) \right\|
\le \int \|\nabla^2_\alpha \psi(\alpha)\|\,\rho_U(du)
\le M.\]
Now, \[\nabla_\alpha f(\alpha) = - (\nabla_\alpha \bar{\psi}(\alpha))^\top (x - \bar{\psi}(\alpha)),\] 
\[
\nabla_\alpha^2 f(\alpha) = (\nabla_\alpha \bar{\psi}(\alpha))^\top (\nabla_\alpha \bar{\psi}(\alpha)) - \sum_k \left( x_k - \bar{\psi}_k(\alpha) \right) \nabla_\alpha^2 \bar{\psi}_k(\alpha).
\]
Now we put an upper bound on $\nabla_\alpha^2 f(\alpha)$.  
\begin{align*}
\nabla_\alpha^2 f(\alpha) &= (\nabla_\alpha \bar{\psi}(\alpha))^\top (\nabla_\alpha \bar{\psi}(\alpha)) - \sum_k \left( x_k - \bar{\psi}_k(\alpha) \right) \nabla_\alpha^2 \bar{\psi}_k(\alpha)\\
&=\norm{\nabla_\alpha \bar{\psi}(\alpha)}^2-\langle x-\bar{\psi}(\alpha), \nabla_\alpha^2 \bar{\psi}_k(\alpha)\rangle\\
&\leq M^2+\norm{x-\bar{\psi}(\alpha)}\norm{\nabla_\alpha^2 \bar{\psi}_k(\alpha)}\\
&\leq M^2+(\norm{x}+\norm{\bar{\psi}(\alpha)})\leq M^2+2M^2=3M^2< \infty,
\end{align*}
where first inequality follows form Cauchy Schwarz inequality and $\|\nabla_\alpha \bar{\psi}(\alpha)\| \le M,$ and the last inequality follows from $\| \bar{\psi}(\alpha)\| \le M,
\|\nabla_\alpha^2 \bar{\psi}_k(\alpha)\| \le M$.

\end{proof}

\subsection{Proof of Theorem~\ref{thm:closeness}}
\label{sec:Closeness_Between_Empirical_and_Continuous Mean-Field_Risk}
For the conveinience of the reader let us restate Theorem~\ref{thm:closeness} which is the result regarding the closeness between empirical and continuous Mean-Field Risk.
\closeness* 
\begin{proof}
For any specific finite parameters $(\theta_U, \theta_W)$, we can trivially construct discrete empirical probability measures $\hat{\rho}_U, \hat{\rho}_W$. Because discrete empirical measures are perfectly valid probability measures within the space over which the continuous risk is defined, we can conclude that $\mathcal{R}_N(\theta_U, \theta_W) = \mathcal{R}(\hat{\rho}_U, \hat{\rho}_W)$. Since the infimum of the continuous risk is taken over all probability measures (including discrete ones), taking the infimum over all finite parameters trivially yields $\inf_{\theta_U, \theta_W} \mathcal{R}_N(\theta_U, \theta_W) \geq  \inf_{\rho_U, \rho_W} \mathcal{R}(\rho_U, \rho_W)$.

Let's define $K^*_W(x):=\E_{w \sim \rho_W^*}[\phi(x; w)^2]$, $K^*_U(x) := \int \norm{\psi(\alpha(x); u)}^2 \rho_U(du)$, $f^*(\alpha) := \frac{1}{2}\norm{x - \int \psi(\alpha; u) \rho_U(du)}^2$ and $g^*(\alpha) := \int \norm{\psi(\alpha; u)}^2 \rho_U(du)$. Then, by using \cref{lem:bounded_properties} we can say that $\E_x \left[ K^*_U(x) \right] \le M^2 $, $\E_x \left[ K^*_W(x) \right] \le M^2$, $f^*(\alpha)$ and $g^*(\alpha)$ are both $C^2$ functions have double derivatives  uniformly bounded by $3M^2$ and $4M^2$, respectively. We are going use these results throughout the proof.

Now sample $N$ decoder particles $u_j \sim \rho_U^*$ and $N$ encoder particles $w_i \sim \rho_W^*$. The empirical bottleneck $\hat{\alpha}_N(x) = \frac{1}{N} \sum_{i=1}^N \phi(x; w_i)$ is an unbiased estimator of the true bottleneck $\alpha^*(x)$. Now as because the samples are independent, so the variance of their sum is the sum of their variances
\[
\Var_W(\hat{\alpha}_N(x)) = \frac{1}{N^2} \sum_{i=1}^N \Var_{w_i}(\phi(x; w_i)) \le \frac{\E_{w \sim \rho_W^*}[\phi(x; w)^2]}{N} \le \frac{K^*_W(x)}{N}.
\]

We evaluate the expected finite-$N$ risk, conditioning on a fixed draw of encoder weights $W$, and taking the expectation purely over the random decoder particles $u_j \sim \rho_U^*$. Let's recall $\risk_N(U, W) = \frac{1}{2} \mathbb{E}_{x\sim P} \left[ \|x - \hat{x}^{(N)}\|^2 \right]$. Now, by expanding the squared $L_2$ norm we get
\[
\E_U [\mathcal{R}_{N \mid W}] = \underbrace{ \frac{1}{2} \E_x [\norm{x}^2] }_{R_\#} - \underbrace{\E_x \left[ \inner{x}{ \E_U [\hat{x}_N] } \right]}_{1(a)} + \frac{1}{2} \E_x \left[ \E_U [\norm{\hat{x}_N}^2] \right]
\]
Now, $\E_U[\hat{x}_N] = \int \psi(\hat{\alpha}_N(x); u) \rho_U^*(du)$. For the squared norm, we expand the product of sums. Since $u_k$ and $u_j$ are independent for $k \neq j$,
\begin{align*}
    \E_U [\norm{\hat{x}_N}^2] &= \frac{1}{N^2} \E_U \left[ \sum_{j=1}^N \norm{\psi(\hat{\alpha}_N; u_j)}^2 + \sum_{k \neq j} \inner{\psi(\hat{\alpha}_N; u_k)}{\psi(\hat{\alpha}_N; u_j)} \right] \\
    &= \frac{1}{N} \int \norm{\psi(\hat{\alpha}_N; u)}^2 \rho_U^*(du) + \left(1 - \frac{1}{N}\right) \norm{\int \psi(\hat{\alpha}_N; u)\rho_U^*(du)}^2\\
    &=\frac{1}{N}\left(\int \norm{\psi(\hat{\alpha}_N; u)}^2 \rho_U^*(du)-\norm{\int \psi(\hat{\alpha}_N; u)\rho_U^*(du)}^2\right)\\
    &\qquad+\underbrace{\norm{\int \psi(\hat{\alpha}_N; u)\rho_U^*(du)}^2}_{1(b)}
\end{align*}
We substitute these expectations back into the risk equation. By grouping the $R_{\#}$, $1(a)$ and $1(b)$ terms together, we get the risk of the ideal continuous decoder operating on the empirical bottleneck (Term 1). The residual $\frac{1}{N}$ term forms the kind of variance term of the decoder (Term 2). So. we can write
\begin{align*}
    \E_U [\mathcal{R}_{N \mid W}] &= \underbrace{\frac{1}{2} \E_x \left[ \norm{x - \int \psi(\hat{\alpha}_N(x); u) \rho_U^*(du)}^2 \right]}_{\text{Term 1: Continuous Decoder on Empirical Bottleneck}} \\
    &\quad + \underbrace{\frac{1}{2N} \E_x \left[ \int \norm{\psi(\hat{\alpha}_N; u)}^2 \rho_U^*(du) - \norm{\int \psi(\hat{\alpha}_N; u)\rho_U^*(du)}^2 \right]}_{\text{Term 2:  Variance}}
\end{align*}
As, norm of sum vector is always non-negative, we upper-bound Term 2 by dropping it. So,
$\text{Term 2} \le \frac{1}{2N} \E_x \left[ \int \norm{\psi(\hat{\alpha}_N; u)}^2 \rho_U^*(du) \right]$.

We evaluate the expected values of Term 1 and Term 2 over the random encoder weights $W$. For Term 2, let $g^*(\alpha) = \int \norm{\psi(\alpha; u)}^2 \rho_U^*(du)$. Now, by using \cref{lem:bounded_properties},we can say that $g^*(\alpha)$ is a $C^2$ function with second derivative bounded by $4M^2$. We perform a second-order Taylor expansion of $g^*(\hat{\alpha}_N)$ around the optimal continuous bottleneck $\alpha^*$. So, by using Taylor expansion, we can say that
\[
g(\hat{\alpha}_N) = g(\alpha^*) + \nabla g(\alpha^*)^\top (\hat{\alpha}_N - \alpha^*) + \frac{1}{2} (\hat{\alpha}_N - \alpha^*)^\top \nabla^2 g(\tilde{\alpha}) (\hat{\alpha}_N - \alpha^*)
\]
Now, by taking the expectation over $W$, the linear term perfectly vanishes because the empirical bottleneck is unbiased ($\E_W[\hat{\alpha}_N - \alpha^*] = 0$). Now, by using the uniform double derivative bound $4M^2$ to the quadratic term, we can conclude that
\[
\E_W [g(\hat{\alpha}_N(x))] \le g(\alpha^*(x)) + \frac{4M^2}{2} \E_W \left[ (\hat{\alpha}_N - \alpha^*)^2 \right] \le K^*_U(x) + \frac{4 M^2 K^*_W(x)}{2N}.
\]
Now, by taking expectation over $x\sim P$, Term 2 is bounded by $\frac{1}{2N}(M^2 + \frac{4M^4 }{2N})$.

We deploy the exact same second-order Taylor expansion for function $f^*(\alpha) = \frac{1}{2} \norm{x - \int \psi(\alpha; u) \rho_U^*(du)}^2$. We expand $f(\hat{\alpha}_N)$ around the optimal continuous bottleneck $\alpha^*$. After expanding in Taylor series, we get
\[
f(\hat{\alpha}_N)
=
f(\alpha^*)
+
\nabla f(\alpha^*)^\top (\hat{\alpha}_N - \alpha^*)
+
\frac{1}{2}
(\hat{\alpha}_N - \alpha^*)^\top
\nabla^2 f(\tilde{\alpha})
(\hat{\alpha}_N - \alpha^*),
\]
where $\tilde{\alpha}$ lies on the line segment between $\hat{\alpha}_N$ and $\alpha^*$. Now, by taking expectation over the encoder weights $W$, the linear term vanishes due to the unbiasedness of the empirical bottleneck $(E_W[\hat{\alpha}_N - \alpha^*] = 0)$.
 And after applying the uniform double derivative bound $3M^2$ to the quadratic term, we obtain
\[
\mathbb{E}_W\big[f(\hat{\alpha}_N)\big]
\le
f(\alpha^*)
+
\frac{ 3M^2}{2}
\mathbb{E}_W\big[(\hat{\alpha}_N - \alpha^*)^2\big].
\]
Now, by taking expectation over $x\sim P$ gives us
\[
\E_x \E_W [f(\hat{\alpha}_N(x))] \le \E_x [f(\alpha^*(x))] + \frac{3M^4}{2N}.
\]
Notice that $\E_x[f(\alpha^*(x))]$ is exactly the continuous mean-field risk $\mathcal{R}(\rho_U^*, \rho_W^*)$. 

Now after summing the expectations of Term 1 and Term 2 gives us
\[
\E_U [\mathcal{R}_{N \mid W}] \le \mathcal{R}(\rho_U^*, \rho_W^*) + \frac{3M^4 + M^2}{2N} + \frac{M^4}{N^2}.
\]
Now, after taking expectation on both sides with respect to $W$, we get
\[\E_{\rho_U^*, \rho_W^*} [\mathcal{R}_{N}] \le \mathcal{R}(\rho_U^*, \rho_W^*) + \frac{3M^4 + M^2}{2N} + \frac{M^4}{N^2}.\]

 Since the expected empirical risk under $(\rho_U^*, \rho_W^*)$ satisfies
\[
\E_{\rho_U^*, \rho_W^*} [\mathcal{R}_{N}]  \le
\mathcal{R}(\rho_U^*, \rho_W^*) + \frac{3M^4 + M^2}{2N} + \frac{L_g K_W}{4N^2},
\]
there exists a finite realization whose empirical risk is at most this bound. So, we can say that 
  \[
 \mathcal{R}_N(\theta_U, \theta_W) \le \inf_{\rho_U, \rho_W} \mathcal{R}(\rho_U, \rho_W) + \varepsilon_{0} + \frac{3M^4+M^2}{2N} + \frac{M^4}{N^2}.
\]
Since one realization of the of the finite network is bounded by some fixed number, therefore, the absolute infimum over all finite networks is bounded and we can conclude that
\[
\inf_{\theta_U, \theta_W} \mathcal{R}_N(\theta_U, \theta_W) \le \inf_{\rho_U, \rho_W} \mathcal{R}(\rho_U, \rho_W) + \varepsilon_{0} + \frac{3M^4+M^2}{2N} + \frac{M^4}{N^2}.
\]
Now the above result combining with the lower bound $(\inf_{\theta_U, \theta_W} \mathcal{R}_N(\theta_U, \theta_W) \geq  \inf_{\rho_U, \rho_W} \mathcal{R}(\rho_U, \rho_W))$ completes the proof.
\end{proof}

\section{Variational Characterization of Local Minima}\label{sec:Variational_Characterization_of_Local_Minima}

\variationalcharacterization*
\begin{proof}
We fist prove this condition for the decoder measure $\rho_U^*$. Then, we go for the the encoder measure $\rho_W^*$ and it will follow exactly the same perfectly symmetric logic.

The core intuition is that the potential $\Psi_U(u)$ represents the marginal "energy" or "cost" of placing a single particle at location $u$. If the current distribution $\rho_U^*$ is truly a minimizer, there cannot exist any mass in a region where the energy is higher than the absolute minimum, because we could lower the total risk by simply picking up that mass and moving it to the minimum energy location. We formalize this by contradiction.

Under Assumptions~\ref{it:bounded_enocder_decoder-domain}, \ref{it:bounded_data-domain}, and \ref{it:bounded_activation}, we conclude that $V(u;\alpha)$ and $U(u,u';\alpha)$ are continuous functions. Furthermore, the first variation potentials $\Psi_U(u; \rho_U, \rho_W)$ and $\Psi_W(w; \rho_U, \rho_W)$ are also continuous.

Let $\psi_U^* = \inf_u \Psi_U(u; \rho_U^*, \rho_W^*)$ be the absolute lowest possible value of the decoder potential and by assumption we know that $\psi_U^* >-\infty$. Let $S_U$ be the set of all particles that achieve this infimum. So,
\[
S_U = \arg\min_u \Psi_U(u; \rho_U^*, \rho_W^*) = \{ u \in \mathbb{R}^{D_u} \mid \Psi_U(u) = \psi_U^*. \}
\]
 We know that the potential $\Psi_U$ is continuous, which further implies that the set $S_U$ is a closed set.

 Let's assume that the optimal measure $\rho_U^*$ places some amount of mass outside the optimal set $S_U$. This means there exists a point $u_0$ such that
\[
u_0 \in \mathrm{supp}(\rho_U^*) \quad \text{but} \quad u_0 \notin S_U.
\]
As $u_0$ is not in the minimum set, so this implies that its potential must be greater than the minimum. Therefore, there exists some positive margin $2\Delta > 0$ such that $\Psi_U(u_0) = \psi_U^* + 2\Delta$.

Now, the definition of lower continuity, we can conclude that there must exist a small open ball $B(u_0, r_0)$ of radius $r_0$ centered at $u_0$ wherein the potential remains large. We can conclude that for all $u \in B(u_0, r_0)$, 
\[
\Psi_U(u) \ge \psi_U^* + \Delta.
\]
Let's define $
\lambda_0 := \rho_U^*\big(B(u_0, r_0)\big) > 0.
$ We want to take a small amount of mass $\lambda \in [0, \min\{\lambda_0,1\}]$ out of this bad ball and move it to a good point $u_1 \in S_U$ and using this we want to define one perturbed probability measure which suffices for showing contradiction in our case. 

First, we define a probability measure $\nu$ on the bad ball
\[
\nu = \frac{1}{\lambda_0} \rho_U^* \Big|_{B(u_0, \varepsilon_0)}.
\]
Notice that $\nu$ is a valid probability measure (it integrates to 1) and it is entirely confined to the bad region where $\Psi_U \ge \psi_U^* + \Delta$.

Now, we construct our perturbed probability measure $\rho_U^\lambda$ as
\[
\rho_U^\lambda = \rho_U^* - \lambda \nu + \lambda \delta_{u_1} = \rho_U^* + \lambda(\delta_{u_1} - \nu).
\]

Let $\Delta \rho = (\delta_{u_1} - \nu)$. We evaluate the new risk $\mathcal{R}(\rho_U^t, \rho_W^*)$. As autoencoder risk is quadratic with respect to the output measure, then  the Taylor expansion of a quadratic functional is exact after the second derivative. So, after writing Taylor series expansion with respect to the output measure we get
\[
\mathcal{R}(\rho_U^\lambda, \rho_W^*) = \mathcal{R}(\rho_U^*, \rho_W^*) + \lambda \int \Psi_U(u) d(\Delta \rho)(u) + \frac{\lambda^2}{2} \iint U(u, u') d(\Delta \rho)(u) d(\Delta \rho)(u').
\]
We first evaluate the linear (first-order) term $\lambda \int \Psi_U d(\delta_{u_1} - \nu)$. 
\begin{align*}
    \lambda \int \Psi_U(u) d(\delta_{u_1} - \nu)(u) &= \lambda \left( \int \Psi_U(u) \delta_{u_1}(du) - \int \Psi_U(u) \nu(du) \right) \\
    &= \lambda \left( \Psi_U(u_1) - \int_{B(u_0, \varepsilon_0)} \Psi_U(u) \nu(du) \right).
\end{align*}
Because $u_1 \in S_U$, we have $\Psi_U(u_1) = \psi_U^*$. Moreover, since $\nu$ is supported entirely on the set $\{u : \Psi_U(u) \ge \psi_U^* + \Delta\}$, it follows that $\int \Psi_U(u)\,\nu(du) \ge \psi_U^* + \Delta.$
Now,  after substituting these bounds we get
\begin{align*}
    \lambda \int \Psi_U(u) d(\Delta \rho)(u) &\le \lambda \Big( \psi_U^* - (\psi_U^* + \Delta) \Big) = -\lambda\Delta.
\end{align*}

Next, we evaluate the quadratic (second-order) term $\frac{\lambda^2}{2} \iint U d(\Delta \rho) d(\Delta \rho)$. Now, as we know that the interaction kernel $U(u, u'; \alpha)=\frac{1}{2}
\mathbb{E}_x
\Big[
\langle
\psi(\alpha(x); u),
\psi(\alpha(x); u')
\rangle
\Big]$ is continuous on the compact support of the perturbation, the double integral over the bounded signed measure $\Delta \rho$ is upper bounded by some finite constant. Now, using \cref{lem:uniform_bounds} we know that $\|\psi(\alpha(x); u)\|\leq M$, which implies that 
\begin{align*}
|U(u, u'; \alpha)|&=\left|\frac{1}{2}
\mathbb{E}_x
\Big[
\langle
\psi(\alpha(x); u),
\psi(\alpha(x); u')
\rangle
\Big]\right|\le\frac{1}{2} \mathbb{E}_x
\Big[
\norm{\psi(\alpha(x); u)}
\norm{\psi(\alpha(x); u')} \Big]\le \frac{M^2}{2}. 
\end{align*}

Thus, the second-order term is bounded by $M^2\lambda^2$.

Now after combining the linear and quadratic bounds, we get 
\begin{align*}
    \mathcal{R}(\rho_U^\lambda, \rho_W^*) - \mathcal{R}(\rho_U^*, \rho_W^*) &\le -\lambda\Delta + M^2 \lambda^2 \\
    &= \lambda \left( - \Delta + M \lambda \right).
\end{align*}
We can choose a positive mass $\lambda$ such that $0 < \lambda < \min\{\frac{\Delta}{M},\lambda_{0},1\}:=\tilde{\lambda}$, then the term $\lambda \left( - \Delta + M \lambda\right)\leq 0$. Therefore, $\mathcal{R}(\rho_U^\lambda, \rho_W^*) < \mathcal{R}(\rho_U^*, \rho_W^*)$. This implies there exists a pair of probability measures $(\rho_U^\lambda,  \rho_W^*)$ that achieves a lower risk. This contradicts the initial assumption that $(\rho_U^*, \rho_W^*)$ is a global minimizer of the risk functional. So, we get a contradiction. Therefore, the support of the optimal measure must be entirely contained within the set of global minimizers of the potential, particularly
\[
\mathrm{supp}(\rho_U^*) \subseteq \arg\min_u \Psi_U(u; \rho_U^*, \rho_W^*).
\]
This completes the proof for decoder part. 

Now we consider the encoder part. Although the risk is quadratic in the decoder measure $\rho_U$, it is nonlinear in the encoder measure $\rho_W$, since $\rho_W$ appears inside the nonlinear activation $\psi$. Consequently, a quadratic expansion is not available. Instead, we employ a Taylor expansion, which requires $\psi(\alpha;u)$ to be $C^2$ in $\alpha$ with uniformly bounded first and second derivatives, as guaranteed by \Cref{lem:uniform_bounds}.

Let $\psi_W^* = \inf_w \Psi_W(w; \rho_U^*, \rho_W^*)$ and $S_W = \arg\min_w \Psi_W(w)$. Let's assume that there exists $w_0 \in \mathrm{supp}(\rho_W^*)$ but $w_0 \notin S_W$. 
By using the same continuity argument used for the decoder, we can say that there exists a margin $\Delta > 0$ and an open ball $B(w_0, r_0)$ containing mass $\lambda_0 > 0$ where $\Psi_W(w) \ge \psi_W^* + \Delta$. 

We define the localized measure $\nu = \frac{1}{\lambda_0} \rho_W^* \Big|_{B(w_0, \varepsilon_0)}$ and pick an optimal point $w_1 \in S_W$. We perturb the encoder measure for $\lambda \in [0, \min\{\lambda_0,1\}]$ as
\[
\rho_W^\lambda = \rho_W^* - \lambda\nu + \lambda\delta_{w_1} = \rho_W^* + \lambda(\delta_{w_1} - \nu).
\]
Let $\Delta \rho_W = \delta_{w_1} - \nu$. This perturbation causes the intermediate bottleneck feature to shift.  Now as the bottleneck is linear in $\rho_W$, the shift is 
\[
\alpha(x; \rho_W^\lambda) = \alpha(x; \rho_W^*) + \lambda \int \phi(x; w) d(\Delta \rho_W)(w) \equiv \alpha^*(x) + \lambda \Delta\alpha(x).
\]
By using \Cref{lem:uniform_bounds}, we can say that feature map $\phi$ is continuous and bounded by $M$, which implies the shift $\Delta\alpha(x)$ exists and is uniformly bounded by some constant $2M$.

We now evaluate the perturbed prediction $\hat{x}^\lambda$. Because $\psi(\alpha;u)$ is $C^2$ with respect to $\alpha$, we expand it around the optimal bottleneck $\alpha^*$ using Taylor expansion. Using a second-order Taylor expansion of $\psi(\alpha;u)$ around $\alpha = \alpha^*(x)$, we obtain
\[
\psi(\alpha^* + \lambda\Delta\alpha; u)
= \psi(\alpha^*; u)
+ \lambda\,\nabla_\alpha \psi(\alpha^*; u)\, \Delta\alpha(x)
+ \lambda^2 \int_0^1 (1-s)\,\nabla_\alpha^2 \psi\big(\alpha^* + s \lambda \Delta\alpha; u\big)\, (\Delta\alpha)^2 \, ds.
\]
Integrating with respect to $\rho_U^*$, we obtain
\begin{align*}
\hat{x}^\lambda(x)
&= \hat{x}^*(x)
+ \lambda \left(\int \nabla_\alpha \psi(\alpha^*; u)\, , \rho_U^*(du)\right) \Delta\alpha(x)\ \\
&\quad + \lambda^2\left( \int_0^1 (1-s)
\left[
\int \nabla_\alpha^2 \psi\big(\alpha^* + s \lambda \Delta\alpha(x); u\big)\, \rho_U^*(du)
\right]
 ds\right)  (\Delta\alpha(x))^2,
\end{align*}
where we have changed the order of integration in the $\lambda^2$ term by Fubini's theorem  because integrand is bounded and $\hat{x}^*(x)=\int \psi(\alpha^*; u) \rho_U^*(du).$ Specifically, \Cref{lem:uniform_bounds} guarantees that the integrand is  bounded, since $\|\nabla^2_{h} \psi\big(\alpha^* + s \lambda \Delta\alpha(x); u\big)\| \leq M$ 
uniformly for all parameters. Now Combining this with the boundedness of $\Delta\alpha(x)$ we can say that the double integral is absolutely convergent, allow us to use Fubini's theorem. By using \Cref{lem:uniform_bounds}, we also know that $\norm{\nabla_\alpha \psi(\alpha^*; u)}\leq M$.

Now, by using \Cref{lem:uniform_bounds} and $|\Delta\alpha(x)|\le 2M$, we can say that
\begin{align*}
\norm{\hat{x}^*(x)}&=\norm{\int \psi(\alpha^*; u) \rho_U^*(du)}\leq \int \norm{\psi(\alpha^*; u) } \rho_U^*(du)\leq \int M\rho_U^*(du) \leq M,
\end{align*}

\begin{align*}
 \norm{\left(\int \nabla_\alpha \psi(\alpha^*; u)\, \rho_U^*(du)\right)\Delta\alpha(x)}&\leq \left(\int \norm{\nabla_\alpha \psi(\alpha^*; u)}\, \rho_U^*(du)\right)|\Delta\alpha(x)|\\ 
 &\leq 2M \int \norm{\nabla_\alpha \psi(\alpha^*; u)} \rho_U^*(du)\leq \int 2M^2 \rho_U^*(du)=2M^2,  
\end{align*}

\begin{align*}
&\norm{\left( \int_0^1 (1-s)
\left[
\int \nabla_\alpha^2 \psi\big(\alpha^* + s \lambda \Delta\alpha(x); u\big)\, \rho_U^*(du)
\right]
 ds\right)  (\Delta\alpha(x))^2}\\
&\leq \left(\int_0^1 (1-s)
\norm{\left[
\int \nabla_\alpha^2 \psi\big(\alpha^* + s \lambda \Delta\alpha(x); u\big)\, \rho_U^*(du)
\right]} ds\right) |\Delta\alpha(x)|^2 \\
&\leq 4M^2 \int_0^1 (1-s)
\norm{\left[
\int \nabla_\alpha^2 \psi\big(\alpha^* + s \lambda \Delta\alpha(x); u\big)\, \rho_U^*(du)
\right]} ds\\
&\leq 4M^2 \int_0^1 (1-s)
\left[
\int \norm{\nabla_\alpha^2 \psi\big(\alpha^* + s \lambda \Delta\alpha(x); u\big)\,} \rho_U^*(du) 
\right] ds\\
&\leq 4M^2 \int_0^1 (1-s)
\left[
\int M \rho_U^*(du) 
\right] ds=4M^3 \int_0^1 (1-s)
 ds=2M^3.
\end{align*}
Now, after substituting the  expansion of $\hat{x}^t(x)$ into the squared risk functional $\mathcal{R} = \frac{1}{2}\mathbb{E}_x \norm{\hat{x}^\lambda - x}^2$, we observe that the first-order (linear) term of the risk expansion matches the first variation potential $\Psi_W$ and the remaining terms consist of higher-order contributions, with non-vanishing terms appearing up to fourth order. Let $\tau_{1}(x):=\left(\int \nabla_\alpha \psi(\alpha^*; u)\, \rho_U^*(du)\right) \Delta\alpha(x)$ and $\tau_2(x):=\left(\int_0^1 (1-s) \left[ \int \nabla_\alpha^2 \psi\big(\alpha^* + s \lambda \Delta\alpha(x); u\big)\, \rho_U^*(du) \right] ds\right) (\Delta\alpha(x))^2$. 
Under this notation, we can write the perturbed prediction as $\hat{x}^\lambda(x) = \hat{x}^*(x) + \lambda\tau_1(x) + \lambda^2\tau_2(x)$. Substituting this into the squared risk functional, we expand the norm as
\begin{align*}
 2\mathcal{R}(\rho_U^*, \rho_W^\lambda) &= \mathbb{E}_x \norm{\hat{x}^\lambda(x) - x}^2 \\
 &= \mathbb{E}_x \norm{(\hat{x}^*(x) - x) + \lambda\tau_1(x) + \lambda^2\tau_2(x)}^2 \\
 &= 2\mathcal{R}(\rho_U^*, \rho_W^*) + 2 \lambda \mathbb{E}_x \left[ (\hat{x}^*(x) - x)^\top \tau_1(x) \right] \quad \\
 &\quad+ \lambda^2 \mathbb{E}_x \left[ \norm{\tau_1(x)}^2 + 2(\hat{x}^*(x) - x)^\top \tau_2(x) \right] \\
 &\quad + 2\lambda^3 \mathbb{E}_x \left[ \tau_1(x)^\top \tau_2(x) \right] + \lambda^4 \mathbb{E}_x \left[ \norm{\tau_2(x)}^2 \right].
\end{align*}

So after dividing by 2, we can write 
\begin{align*}
    \mathcal{R}(\rho_U^*, \rho_W^\lambda) - \mathcal{R}(\rho_U^*, \rho_W^*) &= \lambda \mathbb{E}_x \left[ (\hat{x}^*(x) - x)^\top \tau_1(x) \right] \\
    &\quad + \frac{\lambda^2}{2} \mathbb{E}_x \left[ \norm{\tau_1(x)}^2 + 2(\hat{x}^*(x) - x)^\top \tau_2(x) \right] \\
    &\quad + \lambda^3 \mathbb{E}_x \left[ \tau_1(x)^\top \tau_2(x) \right] + \frac{\lambda^4}{2} \mathbb{E}_x \left[ \norm{\tau_2(x)}^2 \right].
\end{align*}

We first evaluate the coefficient of the linear term in $t$. Now, by substituting the $\tau_1(x)$ alongside $\Delta\alpha(x) = \int \phi(x; w) d(\Delta \rho_W)(w)$, we apply Fubini's theorem to swap the expectation over the data $x$ and the integral over the parameter measure $w$ to get
\begin{align*}
    &\lambda \mathbb{E}_x \left[ (\hat{x}^*(x) - x)^\top \tau_1(x) \right] \\
    &\qquad= \lambda \mathbb{E}_x \left[ (\hat{x}^*(x) - x)^\top \left( \int \nabla_\alpha \psi(\alpha^*; u)\, \rho_U^*(du) \right) \left( \int \phi(x; w) d(\Delta \rho_W)(w) \right) \right] \\
    &\qquad= \lambda \int \left( \mathbb{E}_x \left[ (\hat{x}^*(x) - x)^\top \left( \int \nabla_\alpha \psi(\alpha^*; u)\, \rho_U^*(du) \right) \phi(x; w) \right] \right) d(\Delta \rho_W)(w) \\
    &\qquad= \lambda \int \Psi_W(w; \rho_U^*, \rho_W^*) d(\Delta \rho_W)(w).
\end{align*}
This is the first variation potential of the encoder. Now, by the exact same integration logic applied earlier for the decoder proof, inserting $\Delta \rho_W = \delta_{w_1} - \nu$ evaluates this linear term to at most $-\lambda\Delta$. So, we can conclude
\begin{align*}
    \lambda \int \Psi_W(w) d(\delta_{w_1} - \nu)(w) &\le \lambda \Big( \psi_W^* - (\psi_W^* + \Delta) \Big) = -\lambda\Delta.
\end{align*}

Now we consider the higher-order terms. Since we know that $\|\hat{x}^*(x)\|\leq M$, $\|x\|\leq M$, $\|\tau_1(x)\|\leq 2M^2$, and $\|\tau_2(x)\|\leq 2M^3$ are all uniformly bounded over their respective domains, the expectations multiplying $\lambda^2, \lambda^3$ and $\lambda^4$ are bounded by some positive constants. So, we can conclude that
\begin{align*}
   & \frac{1}{2} \mathbb{E}_x \left[ \norm{\tau_1(x)}^2 + 2(\hat{x}^*(x) - x)^\top \tau_2(x) \right] \\
   &\qquad\le \frac{1}{2} \mathbb{E}_x \left[ \norm{\tau_1(x)}^2 + 2\norm{\hat{x}^*(x) - x} \norm{\tau_2(x)} \right] \\
    &\qquad\le \frac{1}{2} \left[ (2M^2)^2 + 2(M+M)(2M^3) \right] = \frac{1}{2} \left[ 4M^4 + 8M^4 \right] = 6 M^4,
\end{align*}
\begin{align*}
    \mathbb{E}_x \left[ \tau_1(x)^\top \tau_2(x) \right] &\le \mathbb{E}_x \left[ \norm{\tau_1(x)} \norm{\tau_2(x)} \right] \le (2M^2)(2M^3) 
    = 4M^5,
\end{align*}
\begin{align*}
    \frac{1}{2} \mathbb{E}_x \left[ \norm{\tau_2(x)}^2 \right] &\le \frac{1}{2} (2M^3)^2 
= \frac{1}{2} (4M^6) = 2M^6.
\end{align*}

Substituting these explicit upper bounds into our risk difference, we get that
\begin{align*}
    \mathcal{R}(\rho_U^*, \rho_W^\lambda) - \mathcal{R}(\rho_U^*, \rho_W^*) &\le -\lambda\Delta + 6M^4 \lambda^2 + 4M^5 \lambda^3 + 2M^6 \lambda^4 \\
    &= \lambda\Big( -\Delta + 6M^4  \lambda + 4M^5 \lambda^2 + 2M^6 \lambda^3 \Big).
\end{align*}

We want to make this overall change strictly negative. Notice that the expression inside the parenthesis contains $-\Delta$ added to a cubic polynomial $P(\lambda) = 6M^4  \lambda + 4M^5 \lambda^2 + 2M^6 \lambda^3$ with positive coefficients. Because $P(0) = 0$ and $P(\lambda)$ is continuous, its value starts at zero and grows continuously for $\lambda>0$. Therefore, it is possible to choose a   $0<\lambda' \le \min\{\lambda_{0},1\}$ such that $P(\lambda') < \Delta$. For such a choice of $\lambda'$, the term $\big( -\Delta + P(\lambda') \big)$ is strictly negative. Thus, $\mathcal{R}(\rho_U^*, \rho_W^\lambda) - \mathcal{R}(\rho_U^*, \rho_W^*) < 0$, which implies that we have shown a valid distribution yielding a strictly lower risk. 

This directly contradicts the initial assumption that the unperturbed pair $(\rho_U^*, \rho_W^*)$ is a global minimizer. Therefore, 
\[
\mathrm{supp}(\rho_W^*) \subseteq \arg\min_w \Psi_W(w; \rho_U^*, \rho_W^*).
\]
This completes the proof. 
\end{proof}

\section{Main Result: Mean-Field Dynamics and Propagation of Chaos} \label{app:mean_field_dynamics_chaos}
Before stating the main result, we recall the key quantities governing the mean-field dynamics (see \Cref{sec:mean-field_risk}). The Vlasov-McKean Wasserstein dynamics are driven by the spatial gradients of the first variations of the risk functional $\Psi_U = \frac{\delta \mathcal{R}}{\delta \rho_U}$ and $\Psi_W = \frac{\delta \mathcal{R}}{\delta \rho_W}$, and is given by

\[
\Psi_U(u; \rho_U, \rho_W) = \mathbb{E}_x [\langle \hat{x}-x, \psi(\alpha(x);u)\rangle],
\]

\[
\Psi_W(w; \rho_U, \rho_W) = \mathbb{E}_x \left[ \left \langle \hat{x} - x , \int \frac{\partial \psi(\alpha; u)}{\partial \alpha} \, d\rho_U(u) \right \rangle \phi(x;w) \right]
.\]

The particle dynamics are governed by the negative spatial gradients of these potentials. 
For the decoder, differentiating $\Psi_U$ with respect to $u$ yields
\begin{equation}
    G_U(u; \rho_U, \rho_W) 
    := -\nabla_u \Psi_U(u; \rho_U, \rho_W) 
    = -\mathbb{E}_x \left[ \left(\nabla_u \psi(\alpha(x); u) \right)^\top  (\hat{x} - x)\right]
    \label{eq:GU}
\end{equation}
as the only term which is dependent on the microscopic variable $u$ is $\psi(\alpha(x);u)$. For the encoder, we differentiate $\Psi_W$ with respect to $w$. Observe that $\hat{x}$ and $\alpha(x)$ depend only on the measure $\rho_W$, and hence are constant with respect to the microscopic variable $w$. Therefore, the gradient acts only on the feature map $\phi(x; w)$, giving
\begin{align}
     G_W(w; \rho_U, \rho_W) 
    &:= -\nabla_w \Psi_W(w; \rho_U, \rho_W) \nonumber\\
    &= -\mathbb{E}_x \left[ \left \langle \hat{x} - x, 
    \int \frac{\partial \psi(\alpha; u)}{\partial \alpha}\, d\rho_U(u) \right \rangle 
    \nabla_w \phi(x; w) \right]  \label{eq:GW}.
   \end{align}

These vector fields $G_U$ and $G_W$ provide the mean-field gradients driving the particle dynamics. Let's also recall the gradient driving the SGD. Let $z_{k+1} = x_{k+1}$ denote the single random data sample drawn at step $k$. We define the discrete single-sample stochastic gradients for the decoder and encoder particles as
\begin{align*}
&F_{U,j}(u_j^k, \theta_{W,k}; z_{k+1}) := - \big(\nabla_u \psi(\hat{\alpha}_N(x_{k+1}); u_j^k)\big)^\top \Big( \hat{x}_N(x_{k+1}) - x_{k+1} \Big),\\
&F_{W,i}(w_i^k, \theta_{U,k}; z_{k+1}) \\
&\quad:= - \left(\nabla_w \phi(x_{k+1}; w_i^k)\right)^\top \left[ \frac{1}{N} \sum_{m=1}^N \nabla_\alpha \psi(\hat{\alpha}_N(x_{k+1}); u_m^k) \right]^\top \left( \hat{x}_N(x_{k+1}) - x_{k+1} \right),   
\end{align*}

where $\hat{\alpha}$ is the empirical bottleneck potential.
Furthermore, we distinguish between the continuous vector fields ($G_U, G_W$) and the gradients used by SGD ($F_U, F_W$). While $G$ represents the ideal, noise-free population gradient averaged over the entire dataset, actual SGD only computes a noisy gradient using a single random data sample $z_{k+1} = x_{k+1}$ at each step $k$. Notice that the expected values of these stochastic gradients, taken over the data distribution $z \sim P$, are exactly the population continuous vector fields $G_U$ and $G_W$ evaluated over the discrete empirical measures $\hat{\rho}_U^{(N)}$ and $\hat{\rho}_W^{(N)}$.

Now, we state the main result about the propagation of chaos. 
\meanfieldlimit*

Before giving the proof of \Cref{thm:main_dynamics}, we present a result that we shall use to prove the theorem. 

We first establish the Lipschitz continuity of the mean-field vector fields. Let us consider the following parametric test function classes corresponding to the decoder and encoder as follows. For decoder, we consider two separate families 
\[\mathcal{F}_{U,1} = \{ u' \mapsto \psi(\alpha; u') \}_{\alpha \in \mathcal{A}} \text{ and } \,\, \mathcal{F}_{U,2} = \left\{ u' \mapsto \frac{\partial \psi}{\partial \alpha}(\alpha; u') \right\}_{\alpha \in \mathcal{A}}\]
 where $\mathcal{A}$ is the compact subset where $\alpha$ lies, as in \Cref{lem:uniform_bounds} and for encoder, we define the family
\[\mathcal{F}_W = \{ w' \mapsto \phi(x; w') \}_{x \in \mathcal{X}},\]
where $\mathcal{X}$ is the support of $P$. These are the families of functions that our network \emph{sees}, hence the parametric distance of these family govern all the statistical concentration. We recall that the corresponding parametric distances between measures are
\[
d_{\mathcal{F}_{U,1}}(\mu, \nu) = \sup_{f \in \mathcal{F}_{U,1}} \norm{ \int f d\mu - \int f d\nu }_2 \text{ and } \,\,  d_{\mathcal{F}_{U,2}}(\mu, \nu) = \sup_{f \in \mathcal{F}_{U,2}} \norm{ \int f d\mu - \int f d\nu }_2,
\]
for the decoder, and 
\[
d_{\mathcal{F}_W}(\mu, \nu) = \sup_{f \in \mathcal{F}_W} \left| \int f d\mu - \int f d\nu \right|.
\]
for the encoder. We combine the parametric distance of both families $\mathcal{F}_{U,1}$ and $\mathcal{F}_{U,2}$ into a single metric $d_{\mathcal{F}_U}$ by taking the maximum, 
\[
d_{\mathcal{F}_U}(\mu, \nu) = \max \left( d_{\mathcal{F}_{U,1}}(\mu, \nu), d_{\mathcal{F}_{U,2}}(\mu, \nu) \right).
\]

\begin{lemma}[Lipschitz Continuity of Mean-Field Vector Fields via Parametric Distances] \label{lem:vector_lipschitz}
Under Assumptions \ref{it:bounded_enocder_decoder-domain}, \ref{it:bounded_data-domain} and \ref{it:bounded_activation}, the continuous vector fields $G_U$ and $G_W$ are Lipschitz continuous with respect to the measures under the parametric metrics. There exists a constant $L_G > 0$ such that for any measures we have
\begin{align*}
    \norm{G_U(u; \hat{\rho}_U, \hat{\rho}_W) - G_U(u; \rho_U, \rho_W)}_2 &\le L_G \Big( d_{\mathcal{F}_U}(\hat{\rho}_U, \rho_U) + d_{\mathcal{F}_W}(\hat{\rho}_W, \rho_W) \Big), \\
    \norm{G_W(w; \hat{\rho}_U, \hat{\rho}_W) - G_W(w; \rho_U, \rho_W)}_2 &\le L_G \Big( d_{\mathcal{F}_U}(\hat{\rho}_U, \rho_U) + d_{\mathcal{F}_W}(\hat{\rho}_W, \rho_W) \Big),
\end{align*}
where $L_G = \max\{M, 3M^2, 3M^3\}$.
\end{lemma}

\begin{proof}
For both vector fields, we decouple the simultaneous variation of the two measures by applying the triangle inequality. For any vector field $G$, we can add and subtract the mixed term $G(\rho_U, \hat{\rho}_W)$ and apply the triangle inequality. So, we can write
\[
\norm{G(\hat{\rho}_U, \hat{\rho}_W) - G(\rho_U, \rho_W)}_2 \le \underbrace{\norm{G(\hat{\rho}_U, \hat{\rho}_W) - G(\rho_U, \hat{\rho}_W)}_2}_{\text{varying } \rho_U \text{ only}} + \underbrace{\norm{G(\rho_U, \hat{\rho}_W) - G(\rho_U, \rho_W)}_2}_{\text{varying } \rho_W \text{ only}}.
\]
We now analyze these decoupled variations for both $G_U$ and $G_W$. Note that by \Cref{lem:uniform_bounds}, all feature maps, predictions, and their derivatives are uniformly bounded by a constant $M$.
 
First we consider the Lipschitz continuity of the decoder field $G_U$. Recall the decoder vector field for a given data point $x$ is
\[
G_U(u; \rho_U, \rho_W) = -\mathbb{E}_x \Big[ \left(\nabla_u \psi(\alpha(x; \rho_W); u)\right)^\top (\hat{x}(x; \rho_U, \rho_W) - x)\Big].
\]

We first try to bound the term $\norm{G(\hat{\rho}_U, \hat{\rho}_W) - G(\rho_U, \hat{\rho}_W)}$. Because the encoder measure $\hat{\rho}_W$ is fixed, the bottleneck feature $\hat{\alpha} = \int \phi d\hat{\rho}_W$ is constant. The only term inside the expectation depending on the decoder measure is the prediction $\hat{x}$. The difference in the vector fields is
\begin{align*}
    &\norm{G_U(u;\hat{\rho}_U, \hat{\rho}_W) - G_U(u;\rho_U, \hat{\rho}_W)}\\
    &= \norm{ -\mathbb{E}_x \Big[ \left(\nabla_u \psi(\hat{\alpha}; u)\right)^\top \big(\hat{x}(\hat{\rho}_U) - x\big) -  \left( \nabla_u \psi(\hat{\alpha}; u) \right) ^\top \big(\hat{x}(\rho_U) - x\big)\Big]  } \\
    &\le \mathbb{E}_x \left[  \norm{\hat{x}(\hat{\rho}_U) - \hat{x}(\rho_U)}_2 \norm{\nabla_u \psi(\hat{\alpha}; u)}_{op}\right].
\end{align*}

Notice that $\hat{x}(\hat{\rho}_U) - \hat{x}(\rho_U) = \int \psi(\hat{\alpha}; u') d(\hat{\rho}_U - \rho_U)(u')$. Because the function $u' \mapsto \psi(\hat{\alpha}; u')$ belongs exactly to our parametric class $\mathcal{F}_{U,1}$, this integral difference is bounded precisely by $d_{\mathcal{F}_{U,1}}(\hat{\rho}_U, \rho_U)$, which in turn is bounded by the maximum distance $d_{\mathcal{F}_U}(\hat{\rho}_U, \rho_U)$. Since by Lemma \ref{lem:uniform_bounds} we have $\norm{\nabla_u \psi}_{op} \le M$, the variation is bounded.
\[
\norm{G_U(u; \hat{\rho}_U, \hat{\rho}_W) - G_U(u; \rho_U, \hat{\rho}_W)}_2 \le M d_{\mathcal{F}_U}(\hat{\rho}_U, \rho_U).
\]

Next, we bound the variation with respect to the encoder measure. Keeping $\rho_U$ fixed, the bottleneck shifts from $\alpha = \int \phi\, d\rho_W$ to $\hat{\alpha} = \int \phi\, d\hat{\rho}_W$. The function $w \mapsto \phi(x; w)$ belongs to $\mathcal{F}_W$, so the bottleneck shift is exactly bounded by $|\hat{\alpha} - \alpha| \le d_{\mathcal{F}_W}(\hat{\rho}_W, \rho_W)$. Adding and subtracting the mixed product term $(\nabla_u \psi(\alpha; u))^\top (\hat{x}(\hat{\alpha}) - x)$ inside the expectation gives
\begin{align*}
    &\norm{G_{U}(u;\rho_U, \hat{\rho}_W) - G_{U}(u;\rho_U, \rho_W)}_2\\
    &\leq \E_{x}\left[\norm{ \left( \nabla_u \psi(\hat{\alpha}; u)\right)^\top (\hat{x}(\hat{\alpha}) - x) - \left(\nabla_u \psi(\alpha; u)\right)^\top (\hat{x}(\alpha) - x) }_2 \right]\\
    &= \E_{x}\left[\norm{  \Big( \nabla_u \psi(\hat{\alpha}; u) - \nabla_u \psi(\alpha; u) \Big)^\top (\hat{x}(\hat{\alpha}) - x) +  (\nabla_u \psi(\alpha; u))^\top \left( \hat{x}(\hat{\alpha}) - \hat{x}(\alpha) \right) }_2\right] \\
    &\le \E_{x}\left[  \norm{\nabla_u \psi(\hat{\alpha}; u) - \nabla_u \psi(\alpha; u)}_{op}\norm{\hat{x}(\hat{\alpha}) - x}_2 +  \norm{\nabla_u \psi(\alpha; u)}_{op}\norm{\hat{x}(\hat{\alpha}) - \hat{x}(\alpha)}_2\right].
\end{align*}

By Lemma \ref{lem:uniform_bounds}, $\norm{\hat{x} - x}_2 \le 2M$, $\norm{\nabla_u \psi}_{op} \le M$, and $\nabla_u \psi$ is $M$-Lipschitz with respect to $\alpha$, and furthermore, as $\psi$ is $M$-Lipschitz, we have
\begin{align*}
    \norm{\hat{x}(\hat{\alpha}) - \hat{x}(\alpha)}_2&=\norm{\int \psi(\hat{\alpha}; u')\,  d\rho_U(u')-\int \psi(\alpha; u') \, d\rho_U(u')}_2\\
    &=\norm{\int \left(\psi(\hat{\alpha}; u')-\psi(\alpha; u')\right) \, d\rho_U(u')}_2\\
    &\leq \int \norm{\left(\psi(\hat{\alpha}; u')-\psi(\alpha; u')\right)}_2 \, d\rho_U(u')\\
    &\overset{(*)}{\leq} \int M |\hat{\alpha}-\alpha| \, d\rho_U(u') = M |\hat{\alpha}-\alpha|,
\end{align*}
Substituting these bounds yields a direct inequality.
\begin{align*}
    &\norm{G_U(u; \rho_U, \hat{\rho}_W) - G_U(u; \rho_U, \rho_W)}_2 \\& \le \mathbb{E}_x \Big[ (M |\hat{\alpha} - \alpha|) (2M) + (M) (M |\hat{\alpha} - \alpha|) \Big] \\&= 3M^2 |\hat{\alpha} - \alpha| \\&\le 3M^2 d_{\mathcal{F}_W}(\hat{\rho}_W, \rho_W).
\end{align*}

Combining these gives the total bound for the decoder vector field.
\[
\norm{G_U(\hat{\rho}_U, \hat{\rho}_W) - G_U(\rho_U, \rho_W)}_2 \le M d_{\mathcal{F}_U}(\hat{\rho}_U, \rho_U) + 3M^2 d_{\mathcal{F}_W}(\hat{\rho}_W, \rho_W).
\]

Second, we establish the Lipschitz continuity of the encoder field $G_W$. Recall the encoder vector field
\[
G_W(w; \rho_U, \rho_W) = 
-\mathbb{E}_x \left[ \left \langle \hat{x}(x; \rho_U, \rho_W) - x, \int \frac{\partial \psi(\alpha (x; \rho_W); u')}{\partial \alpha} \, d\rho_U(u')  \right \rangle \nabla_w \phi(x; w) \right].
\]
We first try to bound the term $\norm{G(\hat{\rho}_U, \hat{\rho}_W) - G(\rho_U, \hat{\rho}_W)}$. As we are keeping $\hat{\rho}_W$ fixed, the bottleneck $\hat{\alpha}$ is fixed. The terms depending on the decoder measure are the prediction $\hat{x}(\rho_U)$ and the integral $\int \left( \partial \psi(\hat{\alpha} ; u')/\partial \alpha\ \right)\, d\rho_U(u')$. The final gradient $\nabla_w \phi$ is a constant with respect to $\rho_U$. So, we can write
\begin{align*}
    &\norm{G_W(w; \hat{\rho}_U, \hat{\rho}_W) - G_W(w; \rho_U, \hat{\rho}_W)}_2\\
    = & \left\lVert \mathbb{E}_x \left[ \left \langle \hat{x}(\hat{\rho}_U) - x, \int \frac{\partial \psi}{\partial \alpha} \, d\hat{\rho}_U(u')  \right \rangle \nabla_w \phi(x; w) \right]  \right.\\
    &-\left.
    \mathbb{E}_x \left[ \left \langle \hat{x}(\rho_U) - x, \int \frac{\partial \psi}{\partial \alpha} \, d\rho_U(u')  \right \rangle \nabla_w \phi(x; w) \right]\right\rVert_2\\
    \leq &   \mathbb{E}_x \left[ \underbrace{ \left|\left \langle \hat{x}(\hat{\rho}_U) - x, \int \frac{\partial \psi}{\partial \alpha} \, d\hat{\rho}_U(u')  \right \rangle -\left \langle \hat{x}(\rho_U) - x, \int \frac{\partial \psi}{\partial \alpha} \, d\rho_U(u')  \right \rangle \right|}_{(**)} \norm{\nabla_w\phi(x; w)}_2 \right]\\
\end{align*}
We add and subtract $\left \langle \hat{x}(\hat{\rho}_U) - x, \int \frac{\partial \psi}{\partial \alpha}\, d\rho_U \right \rangle$ to $(**)$ to isolate the differences. Using triangle inequality and bilinearity of the inner product we get

\begin{align*}
    &\left| \left\langle \hat{x}(\hat{\rho}_U) - x, \int \frac{\partial \psi}{\partial \alpha}(\hat{\alpha})\, d\hat{\rho}_U \right\rangle - \left\langle \hat{x}(\rho_U) - x, \int \frac{\partial \psi}{\partial \alpha}(\hat{\alpha})\, d\rho_U \right\rangle \right| \\
    &\le \norm{\hat{x}(\hat{\rho}_U) - x}_2 \norm{ \int \frac{\partial \psi}{\partial \alpha}(\hat{\alpha}) d(\hat{\rho}_U - \rho_U) }_2 + \norm{\hat{x}(\hat{\rho}_U) - \hat{x}(\rho_U)}_2 \norm{ \int \frac{\partial \psi}{\partial \alpha}(\hat{\alpha}) d\rho_U }_2.
\end{align*}
The functions $u' \mapsto \frac{\partial \psi}{\partial \alpha}(\hat{\alpha}; u')$ and $u' \mapsto \psi(\hat{\alpha}; u')$ belong to the classes $\mathcal{F}_{U,2}$ and $\mathcal{F}_{U,1}$ respectively. Therefore, both integral differences are bounded by definition by $d_{\mathcal{F}_{U,2}}(\hat{\rho}_U, \rho_U)$ and $d_{\mathcal{F}_{U,1}}(\hat{\rho}_U, \rho_U)$, and thus each is bounded by the maximum distance $d_{\mathcal{F}_U}(\hat{\rho}_U, \rho_U)$. Applying the uniform bounds (Lemma \ref{lem:uniform_bounds}) where $\norm{\hat{x}-x}_2 \le 2M$ and $\norm{\int \frac{\partial \psi}{\partial \alpha}}_2 \le M$, the inner product difference is bounded by $3M d_{\mathcal{F}_U}(\hat{\rho}_U, \rho_U)$. Since $\norm{\nabla_w \phi}_2 \le M$,  we get
\begin{align*}
    \norm{G_W(w; \hat{\rho}_U, \hat{\rho}_W) - G_W(w; \rho_U, \hat{\rho}_W)}_2 \leq 3M^2 d_{\mathcal{F}_U}(\hat{\rho}_U, \rho_U).
\end{align*}

Next we bound the term $\norm{G_{W}(w;\rho_U, \hat{\rho}_W) - G_{W}(w;\rho_U, \rho_W)}$. The decoder measure $\rho_U$ is fixed, but the bottleneck shifts from $\alpha$ to $\hat{\alpha}$. This shift affects both $\hat{x}(\alpha)$ and the expected derivative $\int \frac{\partial \psi}{\partial \alpha}(\alpha; u') d\rho_U$. The final gradient $\nabla_w \phi$ is constant with respect to $\rho_W$. Like before, we can write
\begin{align*}
   &\norm{G_{W}(w;\rho_U, \hat{\rho}_W) - G_{W}(w;\rho_U, \rho_W)}_2\\
   =& \left\| \E_{x}\left[ \left\langle \hat{x}(\hat{\alpha}) - x, \int \frac{\partial \psi}{\partial \alpha}(\hat{\alpha})\, d\rho_U \right\rangle \nabla_w \phi(x; w) - \left\langle \hat{x}(\alpha) - x, \int \frac{\partial \psi}{\partial \alpha}(\alpha)\, d\rho_U \right\rangle \nabla_w \phi(x; w) \right] \right\|_2 \\
   \leq & \E_{x}\left[\underbrace{\left| \left\langle \hat{x}(\hat{\alpha}) - x, \int \frac{\partial \psi}{\partial \alpha}(\hat{\alpha})\, d\rho_U \right\rangle - \left\langle \hat{x}(\alpha) - x, \int \frac{\partial \psi}{\partial \alpha}(\alpha)\, d\rho_U \right\rangle \right|}_{(***)} \norm{ \nabla_w \phi(x; w) }_2\right]
\end{align*}
We add and subtract the mixed term to $(***)$ exactly as before
\begin{align*}
    &\left| \left\langle \hat{x}(\hat{\alpha}) - x, \int \frac{\partial \psi}{\partial \alpha}(\hat{\alpha})\, d\rho_U \right\rangle - \left\langle \hat{x}(\alpha) - x, \int \frac{\partial \psi}{\partial \alpha}(\alpha)\, d\rho_U \right\rangle \right| \\
    \le & \norm{\hat{x}(\hat{\alpha}) - x}_2 \norm{ \int \Big( \frac{\partial \psi}{\partial \alpha}(\hat{\alpha}) - \frac{\partial \psi}{\partial \alpha}(\alpha) \Big) d\rho_U }_2 + \norm{ \hat{x}(\hat{\alpha}) - \hat{x}(\alpha) }_2 \norm{ \int \frac{\partial \psi}{\partial \alpha}(\alpha)\, d\rho_U }_2.
\end{align*}

By Lemma \ref{lem:uniform_bounds}, $\frac{\partial \psi}{\partial \alpha}$ is $M$-Lipschitz with respect to $\alpha$ due to bounded second derivatives. Substituting the Lipschitz bounds yields $(2M)(M|\hat{\alpha} - \alpha|) + (M|\hat{\alpha} - \alpha|)(M) = 3M^2 |\hat{\alpha} - \alpha| \le 3M^2 d_{\mathcal{F}_W}(\hat{\rho}_W, \rho_W)$. As $\norm{\nabla_w \phi}_2 \le M$, we have
\begin{align*}
 \norm{G_W(w; \hat{\rho}_U, \hat{\rho}_W) - G_W(w; \rho_U, \rho_W)}_2 \le 3M^3 d_{\mathcal{F}_W}(\hat{\rho}_W, \rho_W)
\end{align*}

Combining the decoupled bounds, we conclude that both vector fields are globally Lipschitz with respect to the parametric distances.
\begin{align*}
    \norm{G_U(u; \hat{\rho}_U, \hat{\rho}_W) - G_U(u; \rho_U, \rho_W)}_2 &\le \max\{M, 3M^2\} \Big( d_{\mathcal{F}_U}(\hat{\rho}_U, \rho_U) + d_{\mathcal{F}_W}(\hat{\rho}_W, \rho_W) \Big), \\
    \norm{G_W(w; \hat{\rho}_U, \hat{\rho}_W) - G_W(w; \rho_U, \rho_W)}_2 &\le \max\{3M^2, 3M^3\} \Big( d_{\mathcal{F}_U}(\hat{\rho}_U, \rho_U) + d_{\mathcal{F}_W}(\hat{\rho}_W, \rho_W) \Big).
\end{align*}
Taking $L_G = \max\{M, 3M^2, 3M^3\}$ simultaneously satisfies both conditions and completes the proof.

\end{proof}

We now consider the coupled trajectory bounds, which we use to prove \Cref{thm:main_dynamics}.

Let $z_k = x_k$ denote the $k$-th data sample. We define two coupled sets of trajectories, all sharing the same initializations $u_j^0 \sim \rho_{U,0}$ and $w_i^0 \sim \rho_{W,0}$.
\begin{itemize}
    \item \textbf{Continuous Ideal Trajectories} ($\tilde{u}_j^t, \tilde{w}_i^t$): Continuous trajectories are governed by the continuous ideal non-linear Vlasov-McKean dynamics. These particles are driven by the true, infinite-width population measures $\rho_U^s$ and $\rho_W^s$. So, we can write
    \begin{align*}
        \tilde{u}_j^t &= u_j^0 +  \int_0^t \xi(s) G_U(\tilde{u}_j^s; \rho_{U,s}, \rho_{W,s})\, ds ,\\
        \tilde{w}_i^t &= w_i^0 +  \int_0^t \xi(s) G_W(\tilde{w}_i^s; \rho_{U,s}, \rho_{W,s})\, ds.
    \end{align*}
    
    \item \textbf{Discrete SGD Trajectories} ($u_j^t, w_i^t$): We now consider discrete dynamics, where 
$
u_j^t = u_j^{k\epsilon} = u_j^k, 
 \text{for } t \in [k\epsilon, (k+1)\epsilon),
$
driven by discrete single-sample gradients \(F_U\) and \(F_W\). The discrete updates are
    \begin{align*}
        u_j^{k+1} &= u_j^k + \epsilon\xi(k\epsilon)F_{U,j}(u_j^k, \theta_{W,k}; z_{k+1}), \\
        w_i^{k+1} &= w_i^k + \epsilon\xi(k\epsilon)F_{W,i}(w_i^k, \theta_{U,k}; z_{k+1}), 
    \end{align*}
\end{itemize}
where $\theta_{W,k}\in \mathbb{R}^{dN}$ and $\theta_{U,k}\in \mathbb{R}^{dN}$ denote the position of all encoder and decoder neurons at time $k$.
Before going further, we establish the temporal structure of the discrete trajectory. The total continuous time interval is $[0, T]$. The SGD algorithm updates the particles at discrete time steps $t_k = k\epsilon$. Therefore, the total number of discrete update steps within the interval $[0, T]$ is finite and is equal to
\[
K_T = \lfloor T/\epsilon \rfloor.
\]
Because the discrete trajectories $u_j^{[s]}$ and $w_i^{[s]}$ remain constant between updates (i.e., for $s \in [k\epsilon, (k+1)\epsilon)$), evaluating the supremum of any function of the discrete trajectory over continuous time $s \in [0, T]$ is equivalent to evaluating the maximum over the $K_T$ discrete steps. Consequently, whenever we require a probabilistic concentration bound to hold uniformly over the entire continuous time interval $[0, T]$, we only need to apply a union bound over these $K_T$ distinct states. 

To compare the discrete SGD steps to the continuous Vlasov-McKean ODEs, we must represent both processes in the exact same format. We do this by rewriting the discrete SGD sum as a continuous integral using the delayed time-index $[s] \equiv \epsilon \lfloor s / \epsilon \rfloor$. Because the network weights remain perfectly frozen during the interval between updates, integrating a constant gradient over an interval of width $\epsilon$ simply multiplies it by $\epsilon$. This allows us to write the discrete SGD update rule as  $\int_{k\epsilon}^{(k+1)\epsilon} F_{U,i}([s]) ds = \epsilon F_{U,i} (k\epsilon)$ and  $\int_{k\epsilon}^{(k+1)\epsilon} F_{W,j}([s]) ds = \epsilon F_{W,j} (k\epsilon)$ for both decoder and encoder particle.

With this continuous-time embedding established, we now bound the distance between these coupled trajectories.
\begin{lemma}[Coupled Trajectory Bound] \label{lem:trajectory}
Let $\Delta_U(t) = \max_{j \le N} \sup_{s \le t} \norm{\tilde{u}_j^s - u_j^{[s]}}^2$ and $\Delta_W(t) = \max_{i \le N} \sup_{s \le t} \norm{\tilde{w}_i^s - w_i^{[s]}}^2$. Define the joint maximal trajectory deviation as $\Delta(t) = \Delta_U(t) + \Delta_W(t)$.

With probability at least $1 - 4e^{-\delta^2}$, there exist constants $C_{\operatorname{poly}}= 10^5M^{16}\max\{ L_{\xi}^2M_{\xi}^4, M^2_{\xi}\sqrt{d}\}$ and $C_{\operatorname{exp}}= 1200 M_{\xi}^2 M^{10}$ such that for all $T \ge 0$,
\[
\Delta(T) \le  C_{\operatorname{poly}} (T \lor T^2) \left( \epsilon + \frac{1}{N} \right) \Big[ d+1 + \log \big( N(T/\epsilon \lor 1) \big) + \delta^2 \Big] \exp\Big( C_{\operatorname{exp}} T^2 \Big).\]
\end{lemma}

\begin{proof}

The exact position of the continuous decoder particle at time $t$ is given by
\[
\tilde{u}_j^t = u_j^0 +  \int_0^t \xi(s) G_U(\tilde{u}_j^s; \rho_{U,s}, \rho_{W,s}) ds.
\]
The position of the discrete SGD decoder particle evaluated at the floored time $[t]$, can be written exactly as
\[
u_j^{[t]} = u_j^0 +  \int_0^t \xi([s]) F_{U,j}(u_j^{[s]}, W^{[s]}; z_{([s]/\epsilon) + 1}) ds.
\]
Subtracting the two, the shared initial positions $u_j^0$ cancel out exactly. So, the squared error is the squared norm of the integral difference. So, we can write
\[
\norm{\tilde{u}_j^t - u_j^{[t]}}^2 = \norm{  \int_0^t \Big( \xi(s) G_U(\tilde{u}_j^s; \rho_{U,s}, \rho_{W,s}) - \xi([s]) F_{U,j}(u_j^{[s]}, W^{[s]}; z_{[s]/\epsilon + 1}) \Big) ds }^2.
\]
To break this into manageable parts, we add and subtract three terms inside the integral as follows
\begin{enumerate}
    \item $\xi([s]) G_U(\tilde{u}_j^{[s]}; \rho_{U,[s]}, \rho_{W,[s]})$ \quad (Time discretization term)
    \item $\xi([s]) G_U(u_j^{[s]}; \rho_{U,[s]}, \rho_{W,[s]})$ \quad (Spatial tracking term)
    \item $\xi([s]) G_U(u_j^{[s]}; \hat{\rho}_{U,[s]}, \hat{\rho}_{W,[s]})$ \quad (Measure gap term, evaluating field on discrete measures $\hat{\rho}$).
\end{enumerate}

By using algebraic expansion, for any four vectors $a,b,c,d$, we have $\norm{a+b+c+d}^2 \le 4(\norm{a}^2 + \norm{b}^2 + \norm{c}^2 + \norm{d}^2)$. Applying this, we split the error into four distinct integrals as follows
\[
\norm{\tilde{u}_j^t - u_j^{[t]}}^2 \le 16 \left( E_{\operatorname{time}}(t) + E_{\operatorname{track}}(t) + E_{\operatorname{meas}}(t) + E_{\operatorname{mart}}(t) \right),
\]
where 
\begin{align*}
    &E_{\operatorname{time}}(t) :=  \norm{ \int_0^t \left( \xi(s) G_U(\tilde{u}_j^s; \rho_{U,s}, \rho_{W,s}) - \xi([s]) G_U(\tilde{u}_j^{[s]}; \rho_{U,[s]}, \rho_{W,[s]}) \right) \,ds }^2 \\
    &E_{\operatorname{track}}(t) :=  \norm{ \int_0^t \left(\xi([s]) G_U(\tilde{u}_j^{[s]}; \rho_{U,[s]}, \rho_{W,[s]}) - \xi([s]) G_U(u_j^{[s]}; \rho_{U,[s]}, \rho_{W,[s]}) \right) \,ds }^2 \\
    & E_{\operatorname{meas}}(t) :=  \norm{\int_0^t \left(\xi([s]) G_U(u_j^{[s]}; \rho_{U,[s]}, \rho_{W,[s]}) - \xi([s]) G_U(u_j^{[s]}; \hat{\rho}_{U,[s]}, \hat{\rho}_{W,[s]}) \right) \,ds}^2 \\
    &E_{\operatorname{mart}(t)} := \norm{ \int_0^t \xi([s]) \Big( G_U(u_j^{[s]}; \hat{\rho}_{U,[s]}, \hat{\rho}_{W,[s]}) - F_{U,j}\big(u_j^{[s]}, W^{[s]}; z_{([s]/\epsilon) + 1}\big) \Big) \,ds }^2.
\end{align*}

To bound the time discretization error $E_{\operatorname{time}}(t)$, we must control the difference between the vector fields evaluated at the continuous time $s$ and the floored time $[s]$. By applying the Cauchy-Schwarz inequality we can conclude that
\begin{align*}
E_{\operatorname{time}}(t) \le t\int_0^t \norm{ \xi(s) G_U(\tilde{u}_j^s; \rho_{U,s}, \rho_{W,s}) - \xi([s]) G_U(\tilde{u}_j^{[s]}; \rho_{U,[s]}, \rho_{W,[s]}) }_2^2 \,ds.
\end{align*}

To simplify notation, let $G(s) = G_U(\tilde{u}_j^s; \rho_{U,s}, \rho_{W,s})$ and $G([s]) = G_U(\tilde{u}_j^{[s]}; \rho_{U,[s]}, \rho_{W,[s]})$. By adding and subtracting the intermediate term $\xi([s])G(s)$ inside the norm and applying the triangle inequality, we get that
\begin{align*}
  \norm{\xi(s) G(s) - \xi([s]) G([s])}_2 &=\norm{\xi(s) G(s) - \xi([s])G(s)+ \xi([s])G(s)-\xi([s]) G([s])}_{2} \\
  &\le|\xi(s)-\xi([s])| \norm{G(s)}_2 + |\xi([s])| \norm{G(s) - G([s])}_2.  
\end{align*} 

First, we establish the maximum uniform bound $M_G$ on the continuous vector fields. We know that
\begin{align*}
G_U(u; \rho_U, \rho_W) &= -\mathbb{E}_{x \sim P} \left[ \big(\nabla_u \psi(\alpha(x); u)\big)^\top (\hat{x} - x) \right], \\
G_W(w; \rho_U, \rho_W) &= -\mathbb{E}_{x \sim P} \left[ \big(\nabla_w \phi(x; w)\big)^\top \left( \int \nabla_\alpha \psi(\alpha(x); u') d\rho_U(u') \right)^\top (\hat{x} - x) \right],
\end{align*}
where the bottleneck is $\alpha(x) = \int \phi(x; w) d\rho_W(w)$ and the mean-field output is $\hat{x} = \int \psi(\alpha(x); u') d\rho_U(u')$. 
By \Cref{lem:uniform_bounds}, the data and activations are bounded, giving the residual bound $\norm{\hat{x} - x}_2 \le \norm{\hat{x}}_2 + \norm{x}_2 \le 2M$. Again by using \Cref{lem:uniform_bounds}, we conclude that the operator norms of the spatial derivatives are bounded by $M$. So, we can compute the maximum bounds for both vector fields. So, we can conclude that
\begin{align*}
\norm{G_U}_2 &\le \mathbb{E}_{x} \left[ \norm{\nabla_u \psi}_{\mathrm{op}} \norm{\hat{x} - x}_2 \right] \le M(2M) = 2M^2, \\
\norm{G_W}_2 &\le \mathbb{E}_{x} \left[ \norm{\nabla_w \phi}_{\mathrm{op}} \norm{\textstyle\int \nabla_\alpha \psi d\rho_U}_{\mathrm{op}} \norm{\hat{x} - x}_2 \right] \le M(M)(2M) = 2M^3.
\end{align*}
Thus, the vector fields for both particles are uniformly bounded by $M_G \coloneqq \max(2M^2, 2M^3)$. As we assume that the learning rate is Lipschitz continuous with Lipschitz constant being $L_\xi$, so this implies $|\xi(s)-\xi([s])| \le L_\xi \epsilon$ because $|s - [s]| \le \epsilon$. Therefore, the first term $|\xi(s)-\xi([s])| \norm{G(s)}_2$ is bounded by $L_\xi M_G \epsilon$.

For the second term, we bound the learning rate by its maximum magnitude $|\xi([s])| \le M_\xi$. To bound the vector field difference $\norm{G(s) - G([s])}_2$, we add and subtract $G_U(\tilde{u}_j^{[s]}; \rho_U^s, \rho_W^s)$. After doing this and applying triangle inequality we get
\begin{align} \label{eq:vf_diff}
\norm{G(s) - G([s])}_2 \le& \norm{G_U(\tilde{u}_j^s; \rho_{U,s}, \rho_{W,s}) - G_U(\tilde{u}_j^{[s]}; \rho_{U,s}, \rho_{W,s})}_2 \nonumber\\
&+ \norm{G_U(\tilde{u}_j^{[s]}; \rho_{U,s}, \rho_{W,s}) - G_U(\tilde{u}_j^{[s]}; \rho_{U,[s]}, \rho_{W,[s]})}_2.
\end{align}

To bound the first term in Equation \ref{eq:vf_diff}, we use the fact that $G_U(u; \rho_{U,s}, \rho_{W,s})$ is a Lipschitz function with respect to $u$. As the mean-field residual $(\hat{x}-x)$ depends entirely on the global measure $\rho_U$ and not on the individual particle $u$, it acts as a constant with respect to $u$. By using \Cref{lem:uniform_bounds}, we know that $\nabla_u \psi$ is $M$-Lipschitz in $u$. So, we can conclude that
\begin{align*}
\norm{G_U(u_1) - G_U(u_2)}_2 &\le \mathbb{E}_{x} \left[ \norm{\nabla_u \psi(u_1) - \nabla_u \psi(u_2)}_{\mathrm{op}} \norm{\hat{x} - x}_2 \right] \\
&\le \mathbb{E}_{x} \left[ \left(M \norm{u_1 - u_2}_2\right) (2M) \right] \\
&= 2M^2 \norm{u_1 - u_2}_2.
\end{align*}
Thus, we define the parameter Lipschitz constant explicitly as $L_u \coloneqq 2M^2$, giving the bound $L_u \norm{\tilde{u}_j^s - \tilde{u}_j^{[s]}}_2$.

To bound the second term in Equation \ref{eq:vf_diff}, we apply Lemma \ref{lem:vector_lipschitz}
\begin{align*}
\norm{G_U(\tilde{u}_j^{[s]}; \rho_{U,s}, \rho_{W,s}) - G_U(\tilde{u}_j^{[s]}; \rho_{U,[s]}, \rho_{W,[s]})}_2 \le L_G \big( d_{\mathcal{F}_U}(\rho_{U,s}, \rho_{U,[s]}) + d_{\mathcal{F}_W}(\rho_{W,s}, \rho_{W,[s]}) \big).
\end{align*}
Let's recall the families
\[\mathcal{F}_{U,1} = \{ u' \mapsto \psi(\alpha; u') \}_{\alpha \in \mathcal{A}}  \,,\,\, \mathcal{F}_{U,2} = \left\{ u' \mapsto \frac{\partial \psi}{\partial \alpha}(\alpha; u') \right\}_{\alpha \in \mathcal{A}},\]
\[\mathcal{F}_W = \{ w' \mapsto \phi(x; w') \}_{x \in \mathcal{X}}.\]
By \cref{lem:uniform_bounds}, the function classes $\mathcal{F}_{U,1}$, $\mathcal{F}_{U,2}$, and $\mathcal{F}_W$ consist of uniformly bounded Lipschitz functions. In particular, every function $f$ in these classes satisfies
\[
\|f\|_\infty \le M
\quad \text{and} \quad
\mathrm{Lip}(f) \le M.
\]
So, by using \cref{lem:parametric_vs_bl} we can say that 
\[d_{\mathcal{F}_{U,1}}(\rho_{U,s}, \rho_{U,[s]})\leq M d_{BL}((\rho_{U,s}, \rho_{U,[s]})), \quad d_{\mathcal{F}_{U,2}}(\rho_{U,s}, \rho_{U,[s]})\leq M d_{BL}((\rho_{U,s}, \rho_{U,[s]}))\] and
\[ d_{\mathcal{F}_{W}}(\rho_{W,s}, \rho_{W,[s]})\leq M d_{BL}((\rho_{W,s}, \rho_{W,[s]})).\]
So, now we can conclude that
\begin{align*}
\norm{G_U(\tilde{u}_j^{[s]}; \rho_{U,s}, \rho_{W,s}) - G_U(\tilde{u}_j^{[s]}; \rho_{U,[s]}, \rho_{W,[s]})}_2 \le M L_G \big( d_{BL}(\rho_{U,s}, \rho_{U,[s]}) + d_{BL}(\rho_{W,s}, \rho_{W,[s]}) \big).
\end{align*}

We now compute the measure displacements over the interval $[[s], s]$. Under the continuous mean-field dynamics, any decoder particle $\tilde{u}_j$ and encoder particle $\tilde{w}_i$ evolve with instantaneous velocities governed by $\xi(\tau)G_U$ and $\xi(\tau)G_W$. Because the learning rate is bounded by $M_\xi$ and both vector fields are bounded by $M_G$, the velocity of every particle (encoder and decoder) is bounded by $M_\xi M_G$. Given $s - [s] \le \epsilon$, the maximum spatial displacement for any individual parameter is as follows
\begin{align*}
\norm{\tilde{u}_j^s - \tilde{u}_j^{[s]}}_2 \le M_\xi M_G \epsilon, \quad \text{and} \quad \norm{\tilde{w}_i^s - \tilde{w}_i^{[s]}}_2 \le M_\xi M_G \epsilon.
\end{align*}
We now bound the measure displacements over the time interval $[[s], s]$. As established in our mean-field formulation, the continuous time evolution of the general probability distributions $\rho_{U,t}$ and $\rho_{W,t}$ is governed by the coupled Vlasov-McKean continuity equations. Now, the evolution of these densities corresponds exactly to the transport of mass along characteristic curves given below such that
\begin{align*}
\frac{d}{d\tau} u(\tau) &= \xi(\tau) G_U(u(\tau); \rho_{U,\tau}, \rho_{W,\tau}), \\
\frac{d}{d\tau} w(\tau) &= \xi(\tau) G_W(w(\tau); \rho_{U,\tau}, \rho_{W,\tau}),
\end{align*}
with initial conditions $u(0)=u([s])$ and $w(0)=w([s])$. 

Let
\begin{align*}
&T_U=u([s])+\int_{[s]}^{s}\xi(\tau) G_U(u(\tau); \rho_{U,\tau}, \rho_{W,\tau}) d(\tau) \qquad \text{and} \\
& T_W=w([s])+\int_{[s]}^{s}\xi(\tau) G_W(w(\tau); \rho_{U,\tau}, \rho_{W,\tau}) d(\tau)
\end{align*}
 denote the deterministic flow maps, such that $T_U(u) = u(s)$ and $T_W(w) = w(s)$. A fundamental property of continuity equations is that the density at time $s$ is exactly the push-forward of the initial density under this flow map. Therefore, we can say that the measures at the continuous time $s$ as $\rho_U^s = (T_U)_\# \rho_U^{[s]}$ and $\rho_W^s = (T_W)_\# \rho_W^{[s]}$.

From above, we know that $\norm{\xi(\tau) G_U(u(\tau); \rho_{U,\tau}, \rho_{W,\tau})}\leq M_\xi M_G$ and $\norm{\xi(\tau) G_W(w(\tau); \rho_{U,\tau}, \rho_{W,\tau})}\leq M_\xi M_G$. Integrating this bounded velocity over the time increment $s - [s] \le \epsilon$, the maximum spatial displacement of the flow map for any point in the support is bounded by
\begin{align*}
\sup_{u \in \mathrm{supp}(\rho_U^{[s]})} \norm{T_U(u) - u}_2 &\le M_\xi M_G \epsilon, \\
\sup_{w \in \mathrm{supp}(\rho_W^{[s]})} \norm{T_W(w) - w}_2 &\le M_\xi M_G \epsilon.
\end{align*}

Because the distributions are pushforwards under these bounded continuous flow maps, we can directly apply  \Cref{lem:pushforward_bound}.
So, by using using \Cref{lem:pushforward_bound}
 we can conclude that
\begin{align*}
d_{BL}(\rho_{U,s}, \rho_{U,[s]}) &\le W_1(\rho_{U,s}, \rho_{U,[s]}) \le \sup_{u \in \mathrm{supp}(\rho_{U,[s]})} \norm{T_U(u) - u}_2 \le M_\xi M_G \epsilon, \\
d_{BL}(\rho_{W,s}, \rho_{W,[s]}) &\le W_1(\rho_{W,s}, \rho_{W,[s]}) \le \sup_{w \in \mathrm{supp}(\rho_{W,[s]})} \norm{T_W(w) - w}_2 \le M_\xi M_G \epsilon.
\end{align*}

Substituting these back into Equation \ref{eq:vf_diff}, we obtain
\begin{align*}
\norm{G(s) - G([s])}_2 &\le L_u \left(M_\xi M_G \epsilon\right) + M L_G \Big( (M_\xi M_G \epsilon) + (M_\xi M_G \epsilon) \Big) \\
&= M_\xi M_G (L_u + 2 M L_G) \epsilon.
\end{align*}

Now, let's define $C_{\operatorname{time}} \coloneqq L_\xi M_G + M_\xi^2 M_G (L_u + 2 M L_G)$, where $M_G = \max(2M^2, 2M^3)$, $L_u = 2M^2$, $L_{G}=\max\{M^2,3M^3,3M^4\}$. So, now we can conclude that
\begin{align*}
\norm{\xi(s) G(s) - \xi([s]) G([s])}_2 &\le|\xi(s)-\xi([s])| \norm{G(s)}_2 + |\xi([s])| \norm{G(s) - G([s])}_{2}\\
&\leq L_{\xi}M_{G}\epsilon+ M_\xi^2 M_G (L_u + 2 M L_G) \epsilon=  C_{\operatorname{time}}\epsilon 
\end{align*}
Finally, substituting this uniform bound into our initial equation gives us
\begin{align*}
E_{\operatorname{time}}(t) &\le t \int_0^t \norm{\xi(s) G(s) - \xi([s]) G([s])}_2^2 \,ds \le t \int_0^t (C_{\operatorname{time}} \epsilon)^2 \,ds = t^2 C_{\operatorname{time}}^2 \epsilon^2.    
\end{align*}
We next bound the trajectory tracking term $E_{\operatorname{track}}(t)$. This term isolates the error caused by evaluating the continuous ideal trajectory $\tilde{u}_j$ versus the discrete simulated trajectory $u_j$, while holding the continuous time evaluation $[s]$  constant. By applying the Cauchy-Schwarz inequality to the time integral $\norm{\int_0^t f(s) ds}_2^2 \le t \int_0^t \norm{f(s)}_2^2 ds$, we get that
\begin{align*}
E_{\operatorname{track}}(t) \le t \int_0^t \norm{ \xi([s]) G_U(\tilde{u}_j^{[s]}; \rho_{U,[s]}, \rho_{W,[s]}) - \xi([s]) G_U(u_j^{[s]}; \rho_{U,[s]}, \rho_{W,[s]}) }_2^2 \,ds.
\end{align*}
We bound the learning rate by its absolute maximum $|\xi([s])| \le M_\xi$. As established earlier, the continuous vector field $G_U$ is Lipschitz continuous with respect to its individual microscopic parameter $u$ with constant $L_u=2M^2$. Applying this  bound directly  yields
\begin{align*}
E_{\operatorname{track}}(t) \le t M_\xi^2 L_u^2 \int_0^t \norm{\tilde{u}_j^{[s]} - u_j^{[s]}}_2^2 \,ds.
\end{align*}

We next bound the measure gap term $E_{meas}(t)$. This term captures the macroscopic error introduced because the discrete network evaluates its gradients over the finite empirical measures $\hat{\rho}_{U}, \hat{\rho}_{W}$, whereas the ideal flow evaluates over the continuous true population measures $\rho_{U}, \rho_{W}$.
Applying the Cauchy-Schwarz inequality to the integral, we obtain
\[
E_{\operatorname{meas}}(t) \le t \int_0^t \norm{ \xi([s]) G_U(u_j^{[s]}; \rho_{U,[s]}, \rho_{W,[s]}) - \xi([s]) G_U(u_j^{[s]}; \hat{\rho}_{U,[s]}, \hat{\rho}_{W,[s]}) }_2^2 ds.
\]
By applying \Cref{lem:vector_lipschitz}, we know that the difference between vector fields evaluated at different measures is bounded by the bounded Lipschitz distance between those respective measures. So, using \Cref{lem:vector_lipschitz} and the upper bound on  learning rate $M_\xi$, we can conclude that
\begin{align*}
E_{\operatorname{meas}}(t) \le t M_\xi^2 L_G^2 \int_0^t \Big( d_{\mathcal{F}_U}(\rho_{U,[s]}, \hat{\rho}_{U,[s]}) + d_{\mathcal{F}_W}(\rho_{W,[s]}, \hat{\rho}_{W,[s]}) \Big)^2 \,ds.
\end{align*}

To bound $d_{\mathcal{F}_U}(\rho_{U,[s]}, \hat{\rho}_{U,[s]})$, we introduce a bridge measure $\bar{\rho}_{U,[s]} \coloneqq \frac{1}{N}\sum_{m=1}^N \delta_{\tilde{u}_m^{[s]}}$. This represents the empirical measure formed by the continuous particles evaluated at time $[s]$. By the triangle inequality, we decompose the distance into a statistical concentration term and a particle shift term as follows
\begin{align*}
d_{\mathcal{F}_U}(\rho_{U,[s]}, \hat{\rho}_{U,[s]}) \le \underbrace{d_{\mathcal{F}_U}(\rho_{U,[s]}, \bar{\rho}_{U,[s]})}_{\text{Concentration } Q_{U,[s]}} + \underbrace{d_{\mathcal{F}_U}(\bar{\rho}_{U,[s]}, \hat{\rho}_{U,[s]})}_{\text{Particle Shift}}.
\end{align*}
 
Similarly, to bound $d_{\mathcal{F}_W}(\rho_{W,[s]}, \hat{\rho}_{W,[s]})$, we introduce the bridge measure $\bar{\rho}_{W,[s]} \coloneqq \frac{1}{N}\sum_{m=1}^N \delta_{\tilde{w}_m^{[s]}}$, which denotes the empirical measure of the continuous particles evaluated at time $[s]$. By the triangle inequality, we obtain
\begin{align*}
d_{\mathcal{F}_W}(\rho_{W,[s]}, \hat{\rho}_{W,[s]})
\le
\underbrace{d_{\mathcal{F}_W}(\rho_{W,[s]}, \bar{\rho}_{W,[s]})}_{\text{Concentration } Q_{W,[s]}}
+
\underbrace{d_{\mathcal{F}_W}(\bar{\rho}_{W,[s]}, \hat{\rho}_{W,[s]})}_{\text{Particle Shift}}.
\end{align*}

Now, recall the definition of the  continuous particles $\tilde{u}_m$. At time $t=0$, they are drawn independently from the initial distribution $\rho_U^0$. As time progresses, their dynamics are governed entirely by the exact population vector field $G_U(\cdot; \rho_{U,t}, \rho_{W,t})$. Because this population vector field is completely deterministic and defined by the infinite continuous measures $\rho_{U,t}$ and $\rho_{W,t}$, the trajectory of particle $\tilde{u}_1$ is completely independent of the trajectory of particle $\tilde{u}_2$.  
Because applying a deterministic map to independent random variables preserves their independence, the positions of the ideal particles at any time $[s]$ remain independent. By the definition of the push-forward measure, they are identically distributed according to the exact true measure $\rho_{U,[s]}$. Therefore, the bridge measure $\bar{\rho}_{U,[s]} = \frac{1}{N}\sum_{m=1}^N \delta_{\tilde{u}_m^{[s]}}$ is a valid empirical decoder measure of i.i.d. samples. Similarly, we can argue that $\bar{\rho}_{W,[s]} = \frac{1}{N}\sum_{m=1}^N \delta_{\tilde{w}_m^{[s]}}$ is a valid empirical encoder measure of i.i.d. samples.

Note that $d_{\mathcal{F}_U}(\rho_{U,[s]}, \bar{\rho}_{U,[s]})=\max\{d_{\mathcal{F}_{U,1}}(\rho_{U,[s]}, \bar{\rho}_{U,[s]}),d_{\mathcal{F}_{U,2}}(\rho_{U,[s]}, \bar{\rho}_{U,[s]})\}.$ We want to apply  \Cref{lem:empirical_concentration} to the families of test functions $\mathcal{F}_{U,1}$, $\mathcal{F}_{U,2}$, and $\mathcal{F}_W$, we verify that they satisfy the required boundedness and Lipschitz properties using the absolute constant $M$ from \Cref{lem:uniform_bounds}.

For the family $\mathcal{F}_{U,1} = \{ u' \mapsto \psi(\alpha; u') \}_{\alpha \in \mathcal{A}}$, the index parameter $\theta$ corresponds to $\alpha \in \mathcal{A}$. $\mathcal{A}\subset\mathbb{R}^b$ with $b=1$. By part (ii) of \Cref{lem:uniform_bounds}, the mean-field bottleneck feature is bounded by $M$, which means the compact parameter space $\mathcal{A} \subset \mathbb{R}$ has a diameter $R_\alpha \le 2M$. By part (iii) of the \Cref{lem:uniform_bounds}, the the test functions are bounded by $M$. By part (v) of the \Cref{lem:uniform_bounds}, the map $\alpha \mapsto \psi(\alpha; u')$ is $M$-Lipschitz for all $u'$.

For the family $\mathcal{F}_{U,2} = \left\{ u' \mapsto \frac{\partial \psi}{\partial \alpha}(\alpha; u') \right\}_{\alpha \in \mathcal{A}}$, the index parameter is again $\alpha \in \mathcal{A}$ with diameter $2M$. By part (iv) of  \Cref{lem:uniform_bounds}, the norm of spatial derivative $\nabla_\alpha\psi(\alpha;u)$ is bounded by $M$. Furthermore, because the norm of second derivative $\nabla_\alpha^2 \psi(\alpha;u)$ is bounded by $M$, the mean value theorem guarantes that the map $\alpha \mapsto \frac{\partial \psi}{\partial \alpha}(\alpha; u')$ is $M$-Lipschitz. 

For the family $\mathcal{F}_W = \{ w' \mapsto \phi(x; w') \}_{x \in \mathcal{X}}$, the index parameter $\theta$ corresponds to the data samples $x \in \mathcal{X} \subset \mathbb{R}^{d}$. By part (i) of the \Cref{lem:uniform_bounds}, the data is bounded by $M$, yielding a diameter of at most $2M$. By part (iii) and (v) of \Cref{lem:uniform_bounds}, the map $\phi$ is bounded by $M$ and is $M$-Lipschitz in $x$. 

Therefore, all three function families  satisfy the conditions of  \Cref{lem:empirical_concentration} with the single uniform constant $2M$. Define concentration constant for the decoder measure as $C_U := 4M + M\sqrt{2b \log(12M)} + M\sqrt{2 \log 2}=4M + M\sqrt{2 \log(12M)} + M\sqrt{2 \log 2}$ and for the encoder measure as $C_W := 4M + M\sqrt{2d \log(12M)} + M\sqrt{2 \log 2}=4M + M\sqrt{2d \log(12M)} + M\sqrt{2 \log 2}$. So, by applying \Cref{lem:empirical_concentration} to class $\mathcal{F}_{U,1}$, $\mathcal{F}_{U,2}$ and $\mathcal{F}_{W}$, we can conclude that with probability at least $1-\delta$, the following holds
\begin{align*}
& d_{\mathcal{F}_{U,1}}(\rho_{U,[s]}, \bar{\rho}_{U,[s]})\le \frac{1}{\sqrt{N}} \left( C_{U} + M \sqrt{ \log N} + M\sqrt{2 \log(3/\delta)} \right),\\
&d_{\mathcal{F}_{U,2}}(\rho_{U,[s]}, \bar{\rho}_{U,[s]})\le \frac{1}{\sqrt{N}} \left( C_{U} + M \sqrt{ \log N} + M\sqrt{2 \log(3/\delta)} \right),\\
&d_{\mathcal{F}_W}(\rho_{W,[s]}, \bar{\rho}_{W,[s]})\le \frac{1}{\sqrt{N}} \left( C_{W} + M \sqrt{d \log N} + M\sqrt{2 \log(3/\delta)} \right).
\end{align*}
As $d_{\mathcal{F}_U}$ is just the maximum of the $d_{\mathcal{F}_{U,1}}$, and $d_{\mathcal{F}_{U,2}}$, we have with probability at least $1-\delta$,
\begin{align*}
& d_{\mathcal{F}_{U}}(\rho_{U,[s]}, \bar{\rho}_{U,[s]})\le \frac{1}{\sqrt{N}} \left( C_{U} + M \sqrt{ \log N} + M\sqrt{2 \log(3/\delta)} \right),\\
&d_{\mathcal{F}_W}(\rho_{W,[s]}, \bar{\rho}_{W,[s]})\le \frac{1}{\sqrt{N}} \left( C_{W} + M \sqrt{d \log N} + M\sqrt{2 \log(3/\delta)} \right).
\end{align*}
So, we can conclude that 
for any specific discrete time step $[s]$ and for any chosen failure probability $\delta$, we square the bound from the lemma and apply the elementary inequality $(a+b+c)^2 \le 3(a^2+b^2+c^2)$ to obtain
\begin{align*}
&(Q_{U,[s]})^2 \le \frac{3}{N} \Big( C_U^2 + M^2 \log N + 2 M^2 \log(3/\delta) \Big)\\
& (Q_{W,[s]})^2 \le \frac{3}{N} \Big( C_W^2 + M^2 d\log N + 2 M^2 \log(3/\delta) \Big).
\end{align*}

Taking a union bound over all time steps and particles ensures this statistical bound holds simultaneously for all evaluation times $[s]$ up to the maximum time $T$. Because the time step is $\epsilon$, there are at most $K_T = \lfloor T/\epsilon \rfloor \le T/\epsilon$ distinct discrete time steps. To align this high-probability bound with the overall martingale bounds derived earlier, which required a union bound over the $N$ particles and proper alignment across both $U$ and $W$ measure concentrations, we want the total aggregate failure probability for the entire system to be at most $e^{-\delta_{1}^2}$. To guarantee the probability of failure across any of the $K_T$ time steps or $N$ particles is properly bounded, we allocate a much smaller failure probability to each individual event using Boole's inequality and set the individual failure probability to
\begin{equation}
\delta = \frac{e^{-\delta_{1}^2}}{4 N \cdot K_T}.
\end{equation}
To avoid division by zero if $T/\epsilon < 1$, we safely substitute $K_T = (T/\epsilon \lor 1)$. The factor of 4 accounts for the maximum over the different test function classes for both $U$ and $W$.

So,
\begin{align}
\log(3/\delta) &= \log\left( \frac{12 N (T/\epsilon \lor 1)}{e^{-\delta_{1}^2}} \right) \nonumber \\
&= \log\big(12 N (T/\epsilon \lor 1)\big) + \delta_{1}^2.
\end{align}

Substituting this back into the inequality, we conclude that with probability at least $1 - e^{-\delta_{1}^2}$, the bound holds uniformly for all discrete time steps $s \in [0, t]$ for both encoder and decoder
\begin{equation}
(Q_{U,[s]})^2 \le \frac{3}{N}\Big(C_U^2 + M^2  \log N + 2 M^2 \big(\log(12N(T/\epsilon \lor 1)) + \delta_{1}^2\big)\Big),
\end{equation}
\begin{equation}
(Q_{W,[s]})^2 \le \frac{3}{N}\Big(C_W^2 + M^2 d \log N + 2 M^2 \big(\log(12N(T/\epsilon \lor 1)) + \delta_{1}^2\big)\Big).
\end{equation}

The second term measures the particle shift. By \cref{lem:uniform_bounds}, we can say that every function in $\mathcal{F}_U$ and $\mathcal{F}_W$ is bounded by  $M$ and $M$-Lipschitz with respect to the parameters. By \cref{lem:parametric_vs_bl}, we can say that the parametric metric between two distribution $\mu$ and $\nu$ is bounded by $M d_{BL}(\mu,\nu)$. Again we know that $d_{BL}$ metric is bounded by scaled 1-Wasserstein distance, which is in turn bounded by the average pairwise distance between the coupled particles. So, from above we can conclude that
\[
d_{\mathcal{F}_U}(\bar{\rho}_{U,[s]}, \hat{\rho}_{U,[s]}) \le M \cdot d_{BL}(\bar{\rho}_{U,[s]}, \hat{\rho}_{U,[s]})\le M \cdot W_1(\bar{\rho}_{U,[s]}, \hat{\rho}_{U,[s]}) \le \frac{M}{N}\sum_{m=1}^N \norm{\tilde{u}_m^{[s]} - u_m^{[s]}}_2,
\]
\[
d_{\mathcal{F}_W}(\bar{\rho}_{W,[s]}, \hat{\rho}_{W,[s]}) \le M \cdot d_{BL}(\bar{\rho}_{W,[s]}, \hat{\rho}_{W,[s]}) \le M \cdot W_1(\bar{\rho}_{W,[s]}, \hat{\rho}_{W,[s]}) \le \frac{M}{N}\sum_{m=1}^N \norm{\tilde{w}_m^{[s]} - w_m^{[s]}}_2.
\]
Substituting these bounds for both the decoder and encoder measures back into the integral, and applying the algebraic expansion $(a+b+c+d)^2 \le 4(a^2+b^2+c^2+d^2)$ alongside Jensen's inequality $(\frac{1}{N}\sum x)^2 \le \frac{1}{N}\sum x^2$, we can conclude that  with probability at least $1-e^{-\delta_{1}^2}$ 
\begin{align*}
    &E_{\operatorname{}{meas}}(t) \\
    &\le 4t M_\xi^2 L_G^2 \int_0^t \left( (Q_U^{[s]})^2 + (Q_W^{[s]})^2 + \frac{M^2}{N}\sum_{m=1}^N \norm{\tilde{u}_m^{[s]} - u_m^{[s]}}_2^2 + \frac{M^2}{N}\sum_{m=1}^N \norm{\tilde{w}_m^{[s]} - w_m^{[s]}}_2^2 \right) ds,
\end{align*}
for all $N$ decoder particles.

We finally give a bound for martingale noise $E_{\operatorname{mart}}$. The term $E_{\operatorname{mart}}(t)$ captures the accumulated error from using the noisy, single-sample stochastic gradient $F_U$ instead of the exact population vector field $G_U$ evaluated on the empirical measures. Recall that replacing the continuous integral over $[s]$ with the discrete sum yields
\[
E_{\operatorname{mart}}(t) = \norm{ \epsilon \sum_{k=1}^{\lfloor t/\epsilon \rfloor} \xi\big((k-1)\epsilon\big) \Big( G_U(u_j^{k-1}; \hat{\rho}_{U,{k-1}}, \hat{\rho}_{W,{k-1}}) - F_{U,j}\big(u_j^{k-1}, W^{k-1}; z_k\big) \Big) }^2.
\]
To bound this using the vector Azuma-Hoeffding inequality (\Cref{lem:azuma-hoeffding}), we define the filtration $\mathcal{F}_{k-1}$ as the $\sigma$-algebra generated by the history of all data samples drawn up to step $k-1$ as $\sigma(z_1, \dots, z_{k-1})$. For a specific particle $j$, we define the discrete-time sequence $(\Delta \mathbf{X}_k^{(j)})_{k \ge 1}$ taking values in $\mathbb{R}^D$ as
\[
\Delta \mathbf{X}_k^{(j)} := \xi\big((k-1)\epsilon\big) \Big( G_U(u_j^{k-1}; \hat{\rho}_{U,{k-1}}, \hat{\rho}_{W,{k-1}}) - F_{U,j}\big(u_j^{k-1}, W^{k-1}; z_k\big) \Big).
\]
Let $\mathbf{X}_n^{(j)} = \sum_{k=1}^n \Delta \mathbf{X}_k^{(j)}$ with $\mathbf{X}_0^{(j)} = \mathbf{0}$. We verify the two required conditions for \Cref{lem:azuma-hoeffding}.
\begin{enumerate}
    \item \textbf{Martingale Property:} By the definition of SGD, the expected value of the single-sample gradient equals the population gradient over the current empirical measure. Because $z_k$ is drawn independently at step $k$, $\mathbb{E}[F_{U,j}(z_k) \mid \mathcal{F}_{k-1}] = G_U$. Therefore, $\mathbb{E}[\Delta \mathbf{X}_k^{(j)} \mid \mathcal{F}_{k-1}] = \mathbf{0}$, making this a martingale difference sequence.
    
    \item \textbf{Sub-Gaussian Increments:} Applying a union bound, We must show that for any fixed vector $\lambda \in \mathbb{R}^D$, the increment satisfies $\mathbb{E}\big[ e^{\langle \lambda, \Delta \mathbf{X}_k^{(j)} \rangle} \mid \mathcal{F}_{k-1} \big] \le e^{L^2 \|\lambda\|^2 / 2}$ for some constant $L$. Let $Z = \Delta \mathbf{X}_k^{(j)}$. First, we bound the magnitude of the stochastic gradient $F_{U,j}$. For any data sample $z_k = x_k$,
    \[
    F_{U,j}(z_k) = - \big(\nabla_u \psi(\hat{\alpha}_N(x_k); u_j^{k-1})\big)^\top \big( \hat{x}_N(x_k) - x_k \big).
    \]
    By \Cref{lem:uniform_bounds}, the spatial gradient is bounded ($\norm{\nabla_u \psi} \le M$), and the prediction error is bounded by the triangle inequality ($\norm{\hat{x}_N - x_k} \le 2M$). Thus, $\norm{F_{U,j}} \le 2M^2$. Because $G_U$ is the conditional expectation of $F_{U,j}$, Jensen's inequality implies $\norm{G_U} \le 2M^2$. By the triangle inequality, the difference is bounded: $\norm{G_U - F_{U,j}} \le 4M^2$. 
    Since the learning rate is bounded by $M_\xi$, the entire increment is bounded almost surely, $\norm{Z} \le 4 M_\xi M^2 \equiv c$.
    
    Now, define the scalar random variable $Y = \langle \lambda, Z \rangle$. Since $\mathbb{E}[Z \mid \mathcal{F}_{k-1}] = \mathbf{0}$, it follows that $\mathbb{E}[Y \mid \mathcal{F}_{k-1}] = 0$. By the Cauchy-Schwarz inequality, $|Y| \le \norm{\lambda} \norm{Z} \le c \norm{\lambda}$. Thus, $Y$ takes values almost surely in the interval $[-c\norm{\lambda}, c\norm{\lambda}]$, which has length $2c\norm{\lambda}$. Applying Hoeffding's Lemma for bounded zero-mean scalar random variables, the moment-generating function is bounded by
    \[
    \mathbb{E}\big[ e^Y \mid \mathcal{F}_{k-1} \big] \le \exp\left( \frac{(2c\norm{\lambda})^2}{8} \right) = \exp\left( \frac{c^2 \norm{\lambda}^2}{2} \right).
    \]
    This exactly satisfies the sub-Gaussian condition of lemma \ref{lem:azuma-hoeffding} with $L = c = 4 M_\xi M^2$.
\end{enumerate}

We now apply the maximal bound of lemma \ref{lem:azuma-hoeffding} to the trajectory of a single particle up to the maximum number of steps $n = K_T = \lfloor T/\epsilon \rfloor \le T/\epsilon$. For any threshold $y > 0$, we can say that
\[
\mathbb{P}\left( \max_{k \le T/\epsilon} \norm{\mathbf{X}_k^{(j)}} \ge 2L \sqrt{T/\epsilon} \big( \sqrt{d+1} + y \big) \right) \le e^{-y^2}.
\]
We require this bound to hold uniformly over all $N$ particles. Applying a union bound, the probability that any of the $N$ particles exceeds this threshold is at most $N e^{-y^2}$. To ensure the total failure probability is bounded by $e^{-\delta_{2}^2}$, we set $N e^{-y^2} = e^{-\delta_{2}^2}$, which yields $y = \sqrt{\log N + \delta_{2}^2}$. 

Substituting this $y$ back in, we find that with probability at least $1 - e^{-\delta_{2}^2}$, the following holds simultaneously for all $j \in \{1, \dots, N\}$
\[
\max_{k \le T/\epsilon} \norm{\mathbf{X}_k^{(j)}} \le 2L \sqrt{T/\epsilon} \Big( \sqrt{d+1} + \sqrt{\log N + \delta_{2}^2} \Big).
\]
Notice that $E_{\operatorname{mart}}(t) = \norm{\epsilon \mathbf{X}_{\lfloor t/\epsilon \rfloor}^{(j)}}^2$. We uniformly bound the supremum of this term over the entire continuous trajectory $t \in [0, T]$ and across all particles by multiplying the high-probability bound by $\epsilon$ and squaring it as follows
\begin{align*}
    \sup_j \sup_{t \le T} E_{\operatorname{mart}}(t) &= \max_j \max_{k \le T/\epsilon} \norm{\epsilon \mathbf{X}_k^{(j)}}^2 \\
    &\le \epsilon^2 \left[ 2L \sqrt{T/\epsilon} \Big( \sqrt{d+1} + \sqrt{\log N + \delta_{2}^2} \Big) \right]^2 \\
    &= 4 L^2 T \epsilon \Big( \sqrt{d+1} + \sqrt{\log N + \delta_{2}^2} \Big)^2\\
    &\le 8 L^2 T \epsilon \big(d+1 + \log N + \delta_{2}^2 \big).
\end{align*}
Letting $C_{\operatorname{mart}} = 8 L^2 = 128 M_\xi^2 M^4$, we write this bound as $C_{\operatorname{mart}} T \epsilon \big(d+1 + \log N + \delta_{2}^2\big)$.

We now aggregate the four error components for the decoder particles. From the previous deterministic derivations, we established the strict upper bounds for the time discretization and spatial tracking terms as
\begin{align*}
E_{\operatorname{time}}(t) &\le t^2 C_{\operatorname{time}}^2 \epsilon^2, \\
E_{\operatorname{track}}(t) &\le t M_\xi^2 L_u^2 \int_0^t \norm{\tilde{u}_j^{[s]} - u_j^{[s]}}_2^2 \,ds,
\end{align*} 
where
$C_{\operatorname{time}} \coloneqq L_\xi M_G + M_\xi^2 M_G (L_u + 2 M L_G)$, $M_G = \max(2M^2, 2M^3)$, $L_u = 2M^2$, $L_{G}=\max\{M^2,3M^3,3M^4\}$. Under the assumption $M \ge 1$, we simplify each term as follows. First,
\[
M_G = \max(2M^2, 2M^3) = 2M^3, \quad
L_u = 2M^2, \quad
L_G = \max\{M^2, 3M^3, 3M^4\} = 3M^4.
\]
Substituting these into the definition of $C_{\operatorname{time}}$, we obtain
\[
C_{\operatorname{time}}
= L_\xi M_G + M_\xi^2 M_G (L_u + 2 M L_G)
\le 2 L_\xi M^3 + 2 M_\xi^2 M^3 \big(2M^2 + 6M^5\big).
\]
Using $M \ge 1$, we further bound $2M^2 + 6M^5 \le 8M^5$, which yields
\[
C_{\operatorname{time}}
\le 2 L_\xi M^3 + 16 M_\xi^2 M^8.
\]
So, we can conclude that
\begin{align*}
E_{\operatorname{time}}(t) &\le t^2 (2 L_\xi M^3 + 16 M_\xi^2 M^8)^2 \epsilon^2 \le 256 t^2 M^{16} (L_\xi+M_\xi^2)^2 \epsilon^2, \\
E_{\operatorname{track}}(t) &\le 4 t M_\xi^2 M^4 \int_0^t \norm{\tilde{u}_j^{[s]} - u_j^{[s]}}_2^2 \,ds.
\end{align*}
For the remaining two terms, we rely on our probabilistic statements for the measure gap $E_{\operatorname{meas}}(t)$ and the martingale noise $E_{\operatorname{mart}}(t)$. We require both of these bounds to hold simultaneously for all time and particles simultaneously. By setting $\delta_1 = \delta$ for the empirical measure concentration bound and $\delta_2 = \delta$ for the martingale concentration bound, we apply a union bound over these failure events. Thus, with probability at least $1 - 2e^{-\delta^2}$, the following bounds hold concurrently for all $j \le N$ and all $t \le T$
\begin{align*}
&E_{\operatorname{meas}}(t) \\
&\le 4t M_\xi^2 L_G^2 \int_0^t \left( (Q_{U,[s]})^2 + (Q_{W,[s]})^2 + \frac{M^2}{N}\sum_{m=1}^N \norm{\tilde{u}_m^{[s]} - u_m^{[s]}}_2^2 + \frac{M^2}{N}\sum_{m=1}^N \norm{\tilde{w}_m^{[s]} - w_m^{[s]}}_2^2 \right) ds, \\
&E_{\operatorname{mart}}(t) \le C_{\operatorname{mart},U} T \epsilon \big(
d+1 + \log N + \delta^2 \big),
\end{align*}
where 
\begin{align*}
  &C_{\operatorname{mart},U} = 128 M_\xi^2 M^4,\\
 &(Q_{U,[s]})^2 \le \frac{3}{N}\Big(C_U^2 + M^2  \log N + 2 M^2 \big(\log(12NT/\epsilon) + \delta^2\big)\Big),\\
 & (Q_{W,[s]})^2 \le \frac{3}{N}\Big(C_W^2 + M^2 d \log N + 2 M^2 \big(\log(12NT/\epsilon )) + \delta^2\big)\Big).
\end{align*}
such that $C_U =4M + M\sqrt{2 \log(12M)} + M\sqrt{2 \log 2}$ and $C_W =4M + M\sqrt{2d \log(12M)} + M\sqrt{2 \log 2}$. Now, as $d\geq 1$, so we can say that $C_U\le C_W$ and $(Q_{U,[s]})^2\le (Q_{W,[s]})^2$ for all $U,W $ and $[s]$. 

To simplify the integrals, we substitute the definition of the joint suprema $\Delta_U(s) = \max_j \sup_{\tau \le s} \norm{\tilde{u}_j^\tau - u_j^{[\tau]}}_2^2$ and $\Delta_W(s) = \max_i \sup_{\tau \le s} \norm{\tilde{w}_i^\tau - w_i^{[\tau]}}_2^2$. This allows us to upper bound the empirical averages appearing in the measure gap term as $\frac{1}{N}\sum_{m=1}^N \norm{\tilde{u}_m^{[s]} - u_m^{[s]}}_2^2 \le \Delta_U(s)$ and similarly for the encoder. We also define the joint trajectory deviation $\Delta(s) = \Delta_U(s) + \Delta_W(s)$.

By substituting all four bounds back into the expansion $\norm{\tilde{u}_j^t - u_j^{[t]}}_2^2 \le 16 (E_{\operatorname{time}} + E_{\operatorname{track}} + E_{\operatorname{meas}} + E_{\operatorname{mart}})$ and taking the supremum over time and particles, we can conclude that with probability at least $1-2e^{-\delta^2}$,
\begin{align*}
\Delta_U(t) \le 16 \Bigg( &T^2 C_{\operatorname{time}}^2 \epsilon^2 + 4T M_\xi^2 M^4 \int_0^t \Delta_U(s) \, ds \\
&+ 4T M_\xi^2 L_G^2 \int_0^t \Big( (Q_{U,[s]})^2 + (Q_{W,[s]})^2 + M^2 \Delta(s) \Big) \, ds\\
&+ C_{\operatorname{mart},U} T \epsilon \big(d+1 + \log N + \delta^2 \big) \Bigg).
\end{align*}

We next perform a similar symmetric calculation to bound the trajectory tracking of the encoder particles $\Delta_W(T)$. The exact same algebraic decomposition applies, splitting the error into analogous terms $E_{\operatorname{time},W}$, $E_{\operatorname{track},W}$,$E_{\operatorname{meas},W}$, and $E_{\operatorname{mart},W}$. Because the encoder vector field $G_W$ involves an application of the chain rule through the mean-field bottleneck layer, the uniform bounds carry an extra factor of $M$. Similarly to the proof of decoder part, for encoder part we can conclude that 
\begin{align*}
E_{\operatorname{time,W}}(t) &\le t^2 C_{\operatorname{time,W}}^2 \epsilon^2, \\
E_{\operatorname{track, W}}(t) &\le 4 t M_\xi^2 M^6 \int_0^t \norm{\tilde{w}_j^{[s]} - w_j^{[s]}}_2^2 \,ds,
\end{align*}
where $C_{\operatorname{time,W}}= L_\xi M_G + M_\xi^2 M_G (2M^3 + 2 M L_G)
\le 2 L_\xi M^3 + 2 M_\xi^2 M^3 \big(2M^3 + 6M^5\big)\leq 2 L_\xi M^3 + 16 M_\xi^2 M^8\le 16M^8(L_\xi+ M_\xi^2 ).$ Similarly for encoder particle , we can conclude that with probability at least $1 - 2e^{-\delta^2}$, the following bounds hold concurrently for all $j \le N$ and all $t \le T$
\begin{align*}
E_{\operatorname{meas},W}(t) &\le 4t M_\xi^2 L_G^2 \int_0^t \left( (Q_{U,[s]})^2 + (Q_{W,[s]})^2 + \frac{M^2}{N}\sum_{m=1}^N \norm{\tilde{u}_m^{[s]} - u_m^{[s]}}_2^2 + \frac{M^2}{N}\sum_{m=1}^N \norm{\tilde{w}_m^{[s]} - w_m^{[s]}}_2^2 \right) ds, \\
E_{\operatorname{mart},W}(t) &\le C_{\operatorname{mart},W} T \epsilon \big(
d+1 + \log N + \delta^2 \big),
\end{align*}
where $C_{\operatorname{mart},W} = 128 M_\xi^2 M^6$, and the measure gap formulation remains structurally identical, bounded by the exact same measure concentration terms $(Q_{U,[s]})^2$ and $(Q_{W,[s]})^2$. So, with probability at least $1 - 2e^{-z^2}$, the fully aggregated bound for the encoder particles is
\begin{align*}
\Delta_W(t) \le 16 \Bigg( &T^2 C_{\operatorname{time},W}^2 \epsilon^2 + 4T M_\xi^2 M^6 \int_0^t \Delta_W(s) \, ds \\
&+ 4T M_\xi^2 L_G^2 \int_0^t \Big( (Q_{U,[s]})^2 + (Q_{W,[s]})^2 + M^2 \Delta(s) \Big) \, ds \\
&+ C_{\operatorname{mart},W} T \epsilon \big(d+1 + \log N + \delta^2 \big) \Bigg).
\end{align*},
Note that both $C_{\operatorname{time,W}}$ and $C_{\operatorname{time}}$ upper bounded by $16M^8(L_\xi+ M_\xi^2 )$. Applying an identical union bound strategy over the encoder parameter space assigns an additional failure probability of $2e^{-\delta^2}$. Thus, with a total aggregate probability of at least $1 - 4e^{-\delta^2}$, both the encoder and decoder bounds hold simultaneously over the entire domain. So, with probability of at least $1 - 4e^{-\delta^2}$, we can conclude that
\begin{align*}
\Delta(t) \le 16 \Bigg( &512 T^2 M^{16} (L_\xi+ M_\xi^2 )^2 \epsilon^2 + 4T M_\xi^2 M^6 \int_0^t \Delta(s) \, ds \\
&+  \frac{432 T^2 M_\xi^2 M^8 \Big(C_W^2 + M^2 d \log N + 2 M^2 \big(\log(12NT/\epsilon )) + \delta^2\big)\Big)}{N}  \\
&+72T M_\xi^2 M^{10} \int_0^t  \Delta(s)  \, ds+ 256 M_\xi^2 M^6 T \epsilon \big(d+1 + \log N + \delta^2 \big) \Bigg),
\end{align*}

such that $C_W =4M + M\sqrt{2d \log(12M)} + M\sqrt{2 \log 2}$. Now, $C_W \le 20M\sqrt{2d \log(12M)}$, and by Gronwall's lemma, we get
\begin{align*}
\Delta(t) \le & 16 \Bigg( 512 T^2 M^{16} (L_\xi+ M_\xi^2 )^2 \epsilon^2 + \frac{432 T^2 M_\xi^2 M^8 \Big(C_W^2 + M^2 d \log N + 2 M^2 \big(\log(12NT/\epsilon ) + \delta^2\big)\Big)}{N} \\
&+ 256 M_\xi^2 M^6 T \epsilon \big(d+1 + \log N + \delta^2 \big) \Bigg) + \Big( 64 T M_\xi^2 M^6 + 1152 T M_\xi^2 M^{10} \Big) \int_0^t \Delta(s) \, ds.
\end{align*}

Applying Gronwall's inequality, we deduce that for any $t \le T$,
\begin{align*}
\Delta(t) \le & 16 \Bigg( 512 T^2 M^{16} (L_\xi+ M_\xi^2 )^2 \epsilon^2 + \frac{432 T^2 M_\xi^2 M^8 \Big(C_W^2 + M^2 d \log N + 2 M^2 \big(\log(12NT/\epsilon ) + \delta^2\big)\Big)}{N} \\
&+ 256 M_\xi^2 M^6 T \epsilon \big(d+1 + \log N + \delta^2 \big) \Bigg) \times \exp\left( \int_0^t \Big( 64 T M_\xi^2 M^6 + 1152 T M_\xi^2 M^{10} \Big) \, ds \right).
\end{align*}

Evaluating at $t = T$ yields the exact final bound. So we can conclude that
\begin{align*}
\Delta(T) \le & 16 \Bigg( 512 T^2 M^{16} (L_\xi+ M_\xi^2 )^2 \epsilon^2 + \frac{432 T^2 M_\xi^2 M^8 \Big(C_W^2 + M^2 d \log N + 2 M^2 \big(\log(12NT/\epsilon ) + \delta^2\big)\Big)}{N} \\
&+ 256 M_\xi^2 M^6 T \epsilon \big(d+1 + \log N + \delta^2 \big) \Bigg) \exp\Big( 64 T^2 M_\xi^2 M^6 + 1152 T^2 M_\xi^2 M^{10} \Big).
\end{align*}
We simplify the expression by consolidating the constants by grouping the polynomial coefficients and the exponential rate. Let us define the exponential growth rate constant as
\[
 64 M_\xi^2 M^6 + 1152 M_\xi^2 M^{10}< 1200 M_{\xi}^2 M^{10} = C_{\operatorname{grow}}  
\]
and the polynomial prefix constant as
\begin{align*}
&\max \Big\{ 8192 M^{16} (L_\xi+ M_\xi^2 )^2,\ 6912 M_\xi^2 M^8 \max(C_W^2 + 2M^2 \log 12, 2M^2),\ 4096 M_\xi^2 M^6 \Big\}\\
\le &\max \Big\{ 8192 M^{16} (L_\xi+ M_\xi^2 )^2,\ 6912 M_\xi^2 M^8 \max(20M\sqrt{2d \log(12M)})^2 + 2M^2 \log 12, 2M^2),\ 4096 M_\xi^2 M^6 \Big\}\\
\le &  \max \Big\{ 10^5M^{16} L_{\xi}^2M_{\xi}^4, 10^6M^{10}M_{\xi}^2\sqrt{d}, 10^{5}M^6M_{\xi}^2\Big\}\\
\le & 10^6M^{16}\max\{ L_{\xi}^2M_{\xi}^4, M^2_{\xi}\sqrt{d}\} =C_{\operatorname{poly}} 
\end{align*}
Note that $C_{\operatorname{grow}}\le C_{\operatorname{poly}}$ 
If we substitute these constants to bound the maximal coefficients, the expression  for $\Delta(T)$ simplifies to
\begin{align*}
\Delta(T) \le & \Bigg( C_{\operatorname{poly}} T^2 \epsilon^2 + C_{\operatorname{poly}} \frac{T^2}{N} \Big( 1 + d \log N + \log(NT/\epsilon) + \delta^2 \Big) \\
&+ C_{\operatorname{poly}} T \epsilon \big(d+1 + \log N + \delta^2 \big) \Bigg) \exp\big( C_{\operatorname{grow}} T^2 \big)\\
\le & C_{\operatorname{poly}} (T \lor T^2) \Bigg[ \epsilon^2 + \frac{1}{N} \Big( \mathcal{O}(1) + d \log N + \log(4NT/\epsilon) + \delta^2 \Big) \\
&+ \epsilon \big( d+1 + \log N + \delta^2 \big) \Bigg] \exp\big( C_{\operatorname{grow}} T^2 \big).
\end{align*}
We can give a common upper bound for the terms inside the parentheses. We safely assume $\epsilon \le 1$, which implies $\epsilon^2 \le \epsilon$. Furthermore, we can bound the logarithmic terms using $T/\epsilon \lor 1$,
\begin{align*}
\log N &\le \log\big(N(T/\epsilon \lor 1)\big), \\
\log(4NT/\epsilon) &\le \log 4 + \log\big(N(T/\epsilon \lor 1)\big).
\end{align*}

By absorbing the $\mathcal{O}(1)$ and $\log 4$ constants into the $d+1$ term, all three inner components share a common upper bound factor
\[
\Big[ d+1 + \log\big(N(T/\epsilon \lor 1)\big) + \delta^2 \Big].
\]

Applying this common factor to the $\epsilon$, $\epsilon^2$, and $1/N$ terms yields
\begin{align*}
\Delta(T) \le & C_{\operatorname{poly}} (T \lor T^2) \Bigg[ \epsilon \Big(d+1 + \log\big(N(T/\epsilon \lor 1)\big) + \delta^2\Big) \\
&+ \frac{1}{N} \Big(d+1 + \log\big(N(T/\epsilon \lor 1)\big) + \delta^2\Big) \Bigg] \exp\big( C_{\operatorname{exp}} T^2 \big)\\
\le & C_{\operatorname{poly}} (T \lor T^2) \left( \epsilon + \frac{1}{N} \right) \Big[ d+1 + \log \big( N(T/\epsilon \lor 1) \big) + \delta^2 \Big] \exp\Big( C_{\operatorname{exp}} T^2 \Big).
\end{align*}
where $C_{\operatorname{poly}}= 10^6M^{16}\max\{ L_{\xi}^2M_{\xi}^4, M^2_{\xi}\sqrt{d}\}$ and $C_{\operatorname{exp}}= 1200 M_{\xi}^2 M^{10}$. This completes the proof.
\end{proof}

\subsection{Bounding the Mean-Field Risk}

We now bound mean-field risk. We now extend our above logic to the risk functional. Let $\theta_{U,k} = \{u_j^k\}_{j=1}^N$ and $\theta_{W,k} = \{w_i^k\}_{i=1}^N$ be the discrete SGD configurations at step $k$. Let $\tilde{U}^{k\epsilon} = \{\tilde{u}_j^{k\epsilon}\}_{j=1}^N$ and $\tilde{W}^{k\epsilon} = \{\tilde{w}_i^{k\epsilon}\}_{i=1}^N$ be the ideal continuous particle configurations.

\begin{lemma}[Finite-$N$ Risk Lipschitzness] \label{lem:risk_lip}
Under the assumptions of \Cref{thm:main_dynamics}, and for the constant $M$ defined in Lemma~\ref{lem:uniform_bounds} we have
\begin{align*}
&\max_{k \in [0, T/\epsilon] \cap \mathbb{N}} \left| \mathcal{R}_N(\tilde{U}^{k\epsilon}, \tilde{W}^{k\epsilon}) - \mathcal{R}_N(\theta_{U,k}, \theta_{W,k}) \right| \\
&\qquad\le 2M^3 \max_k \left( \max_j \norm{u_j^k - \tilde{u}_j^{k\epsilon}}_2 + \max_i \norm{w_i^k - \tilde{w}_i^{k\epsilon}}_2 \right).
\end{align*}
\end{lemma}

\begin{proof}
Consider two arbitrary network configurations $(\theta_U, \theta_W)$ and $(\theta_{U'}, \theta_{W'})$. We first bound the difference in their outputs for a given input $x$. Let the respective empirical bottlenecks be $\alpha_N = \frac{1}{N}\sum_{i=1}^N \phi(x; w_i)$ and $\alpha'_N = \frac{1}{N}\sum_{i=1}^N \phi(x; w'_i)$. 

By Lemma \ref{lem:uniform_bounds}, the encoder activation $\phi$ is $M$-Lipschitz with respect to the parameter $w$. Applying the triangle inequality, the maximum shift in the bottleneck feature is
\[
\norm{\alpha_N - \alpha'_N}_2 \le \frac{1}{N}\sum_{i=1}^N \norm{\phi(x; w_i) - \phi(x; w'_i)}_2 \le \frac{1}{N}\sum_{i=1}^N M \norm{w_i - w'_i}_2 \le M \max_i \norm{w_i - w'_i}_2.
\]
Next, we evaluate the difference in the final predictions $\hat{x}_N = \frac{1}{N} \sum_{j=1}^N \psi(\alpha_N; u_j)$ and $\hat{x}'_N = \frac{1}{N} \sum_{j=1}^N \psi(\alpha'_N; u'_j)$. By Lemma \ref{lem:uniform_bounds}, the decoder activation $\psi$ is $M$-Lipschitz with respect to both the spatial input $\alpha$ and the parameter $u$. Using the triangle inequality again,
\begin{equation}\label{eq:hatxN-hatx'N}
\begin{aligned}
\norm{\hat{x}_N - \hat{x}'_N}_2 &\le \frac{1}{N}\sum_{j=1}^N \norm{\psi(\alpha_N; u_j) - \psi(\alpha'_N; u'_j)}_2 \\
&\le \frac{1}{N}\sum_{j=1}^N M \Big( \norm{\alpha_N - \alpha'_N}_2 + \norm{u_j - u'_j}_2 \Big) \\
&\le M \norm{\alpha_N - \alpha'_N}_2 + M \max_j \norm{u_j - u'_j}_2 \\
&\le M^2 \max_i \norm{w_i - w'_i}_2 + M \max_j \norm{u_j - u'_j}_2 \\
&\le M^2\left(\norm{w_i - w'_i}_2 + \max_j \norm{u_j - u'_j}_2\right).
\end{aligned}
\end{equation}

Finally, we evaluate the empirical risk $\mathcal{R}_N = \frac{1}{2} \mathbb{E}_x \norm{x - \hat{x}_N}_2^2$. The difference in the risk functional is
\begin{align*}
\left| \mathcal{R}_N(\theta_U, \theta_W) - \mathcal{R}_N(\theta_{U'}, \theta_{W'}) \right| &= \frac{1}{2} \left| \mathbb{E}_x \left[ \norm{x - \hat{x}_N}_2^2 - \norm{x - \hat{x}'_N}_2^2 \right] \right| \\
&= \frac{1}{2} \left| \mathbb{E}_x \left[ \langle 2x - \hat{x}_N - \hat{x}'_N, \hat{x}'_N - \hat{x}_N \rangle \right] \right| \\
&\le \frac{1}{2} \mathbb{E}_x \left[ \norm{2x - \hat{x}_N - \hat{x}'_N}_2 \norm{\hat{x}_N - \hat{x}'_N}_2 \right].
\end{align*}
By Lemma \ref{lem:uniform_bounds}, the data is bounded ($\norm{x}_2 \le M$) and the predictions are uniformly bounded ($\norm{\hat{x}_N}_2 \le M$). Thus, the factor $\norm{2x - \hat{x}_N - \hat{x}'_N}_2 \le 4M$. Using inequality~\eqref{eq:hatxN-hatx'N} and taking expectation yields,
\begin{align*}
    \left| \mathcal{R}_N(\theta_U, \theta_W) - \mathcal{R}_N(\theta_{U'}, \theta_{W'}) \right| &\le 2M \mathbb{E}_x \left[ \norm{\hat{x}_N - \hat{x}'_N}_2 \right]\\& \le 2M^3 \Big( \max_i \norm{w_i - w'_i}_2 + \max_j \norm{u_j - u'_j}_2 \Big).
\end{align*}
Applying this inequality to the specific continuous and discrete configurations across all times $k \le T/\epsilon$ completes the proof.
\end{proof}

\begin{lemma}[Risk Concentration] \label{lem:risk_conc}
Under the assumptions of \Cref{thm:main_dynamics}, with probability at least $1 - e^{-\delta^2}$,
\[
\max_{k \in [0, T/\epsilon] \cap \mathbb{N}} \left| \mathcal{R}_N(\tilde{U}^{k\epsilon}, \tilde{W}^{k\epsilon}) - \mathcal{R}(\rho_{U, k\epsilon}, \rho_{W, k\epsilon}) \right| \le 5M^4 \cdot \sqrt{\frac{1}{N}} \left[ \sqrt{ \log\left(2 \left(\frac{T}{\epsilon} \lor 1\right)\right)} + 1+ \delta \right]
\]
\end{lemma}

\begin{proof}
The proof proceeds by decomposing the error into a deterministic bias term and a stochastic fluctuation term via the triangle inequality as
\[
\left| \mathcal{R}_N(\tilde{U}^{k\epsilon}, \tilde{W}^{k\epsilon}) - \mathcal{R}(\rho_{U, k\epsilon}, \rho_{W, k\epsilon}) \right| \le \underbrace{\left| \mathcal{R}_N - \mathbb{E}[\mathcal{R}_N] \right|}_{\text{Fluctuations}} + \underbrace{\left| \mathbb{E}[\mathcal{R}_N] - \mathcal{R} \right|}_{\text{Bias}}
\]

As proven during the optimality bounds (Theorem~\ref{thm:closeness}), the expected value of the finite-$N$ risk evaluated on exact independent samples tracks the infinite-width mean-field risk with a second-order error bias. Because the empirical bottleneck is an unbiased estimator, the first-order Taylor expansion terms vanish under expectation, leaving only the variance terms which scale as $\mathcal{O}(1/N)$. Thus we conclude 
\[
\left| \mathbb{E}_{\tilde{U},\tilde{W}}\big[ \mathcal{R}_N(\tilde{U}^{k\epsilon}, \tilde{W}^{k\epsilon}) \big] - \mathcal{R}(\rho_{U,k\epsilon}, \rho_{W,k\epsilon}) \right| \le \frac{3M^4 + M^2}{2N} + \frac{M^4}{N^2} \le \frac{3M^4}{N}.
\]

For any fixed time step $k$, the ideal continuous particles $\tilde{U}^{k\epsilon}$ and $\tilde{W}^{k\epsilon}$ consist of exactly $2N$ independent and identically distributed random variables drawn from their respective population measures $\rho_{U, k\epsilon}$ and $\rho_{W, k\epsilon}$. Modifying a single decoder particle $u_j$ changes the prediction $\hat{x}_N$ by at most $\frac{2M}{N}$, while modifying a single encoder particle $w_i$ shifts the bottleneck by $\frac{2M}{N}$, which in turn alters the prediction by at most $\frac{2M^2}{N}$ (by the $M$-Lipschitz continuity of $\psi$). Taking the worst-case change $\norm{\Delta \hat{x}_N} \le \frac{2M^2}{N}$ (for $M \ge 1$), the change in the squared risk functional is bounded by $\frac{1}{2}(4M)\norm{\Delta \hat{x}_N}$. 

Thus, modifying a single particle $u_j$ or $w_i$ changes the empirical risk $\mathcal{R}_N$ by at most $c = \frac{4M^3}{N}$. Because the function $\mathcal{R}_N$ has bounded differences, we apply McDiarmid's inequality  to bound its deviation from the expectation. The sum of the squared bounded differences across all $2N$ variables is
\[
\sum_{i=1}^{2N} c_i^2 = 2N \left( \frac{4M^3}{N} \right)^2 = \frac{32M^6}{N}.
\]
Therefore, for any $\tau > 0$, the probability that the empirical risk deviates from its expectation is bounded by
\[
\mathbb{P}\left( \left| \mathcal{R}_N(\tilde{U}^{k\epsilon}, \tilde{W}^{k\epsilon}) - \mathbb{E}[\mathcal{R}_N] \right| \ge \tau \right) \le 2 \exp\left( - \frac{2 \tau^2}{32M^6 / N} \right) = 2 \exp\left( - \frac{N \tau^2}{16M^6} \right).
\]

As we require this bound to hold uniformly across all discrete time steps $k \in \{1, \dots, K_T\}$, where $K_T = \lfloor T/\epsilon \rfloor$. Now, after applying a union bound over these steps we get that
\[
\mathbb{P}\left( \max_{k\in\{1,\cdots,K_T\}} \left| \mathcal{R}_N - \mathbb{E}[\mathcal{R}_N] \right| \ge \tau \right) \le 2 \left( \frac{T}{\epsilon} \lor 1 \right) \exp\left( - \frac{N \tau^2}{16M^6} \right).
\]
For the high-probability bound, we set the right-hand side equal to our target failure probability $e^{-\delta^2}$ and solve for $\tau$ given below
\[
\frac{N \tau^2}{16M^6} = \delta^2 + \log\left( 2 \left( \frac{T}{\epsilon} \lor 1 \right) \right) \implies \tau = \frac{4M^3}{\sqrt{N}} \sqrt{ \log\left( 2 \left( \frac{T}{\epsilon} \lor 1 \right) \right) + \delta^2 }.
\]
Using the algebraic inequality $\sqrt{a+b} \le \sqrt{a} + \sqrt{b}$, this is upper bounded by $\frac{4M^3}{\sqrt{N}} \big[ \sqrt{\log(2(T/\epsilon \lor 1))} + \delta \big]$. 

The total error is the sum of the fluctuations $\tau$ and the deterministic bias $\frac{3M^4}{N}$. So, with probability at least $1 - e^{-\delta^2}$,
\[
\max_k \left| \mathcal{R}_N - \mathcal{R} \right| \le 5M^4 \cdot \sqrt{\frac{1}{N}} \left[ \sqrt{ \log\left(2 \left(\frac{T}{\epsilon} \lor 1\right)\right)} + 1+ \delta \right].
\]
\end{proof}

Now, we give proof of \Cref{thm:main_dynamics}. 

\textbf{Proof of \Cref{thm:main_dynamics}:}
We are now ready to establish the main bounds of \Cref{thm:main_dynamics}. First, we establish the bound for the test functions. Let $f$ be a test function with bounded Lipschitz norm $\norm{f}_{\mathrm{Lip}} \le 1$. We want to measure the gap between the empirical evaluation of $f$ on our discrete SGD particles $u_j^k$ and the true continuous expectation over $\rho_{U, k\epsilon}$. By adding and subtracting the evaluation of $f$ on the \textit{ideal continuous} particles $\tilde{u}_j^{k\epsilon}$, and applying the triangle inequality, we get
\begin{align*}
    &\left| \frac{1}{N}\sum_{j=1}^N f(u_j^k) - \int f(u)\, \rho_{U, k\epsilon}(du) \right| \\
    &\qquad\le \underbrace{\left| \frac{1}{N}\sum_{j=1}^N f(u_j^k) - \frac{1}{N}\sum_{j=1}^N f(\tilde{u}_j^{k\epsilon}) \right|}_{\text{Term 1 (Trajectory Error)}} + \underbrace{\left| \frac{1}{N}\sum_{j=1}^N f(\tilde{u}_j^{k\epsilon}) - \int f(u)\, \rho_{U, k\epsilon}(du) \right|}_{\text{Term 2 (Measure Concentration)}}
\end{align*}
We first bound Term 1. Because the test function $f$ is 1-Lipschitz, $|f(u) - f(\tilde{u})| \le \norm{u - \tilde{u}}$. So, we can conclude that
\[
\text{Term 1} \le \frac{1}{N}\sum_{j=1}^N \norm{u_j^k - \tilde{u}_j^{k\epsilon}}\le \sqrt{ \frac{1}{N}\sum_{j=1}^N \norm{u_j^k - \tilde{u}_j^{k\epsilon}}^2 },
\]
where we used Cauchy-Schwarz inequality in the last step. Because $\Delta_U(T)$ is defined as the supremum over all particles and all times $t \le T$ (Lemma \ref{lem:trajectory}), this average is bounded by $\sqrt{\Delta_U(T)}$. So, we can conclude that $\text{Term 1}\leq \sqrt{\Delta_U(T)}$.

For Term 2, we aim to bound the statistical error between the empirical mean of the ideal continuous particles and the true population mean, uniformly over all discrete time steps $k \in \{1, \dots, K_T\}$ where $K_T = \lfloor T/\epsilon \rfloor$, for one fixed test function $f$. 


For a fixed bounded Lipschitz $f$ with $\norm{f}_\infty\leq 1$ and $\norm{f}_{\operatorname{Lip}}\leq 1$ evaluated at a fixed time $k\epsilon$, the ideal continuous particles $\tilde{u}_j^{k\epsilon}$ are exactly $N$ i.i.d. samples from the true marginal distribution $\rho_{U, k\epsilon}$. Since $\|f\|_\infty \le 1$, the variables $f(\tilde{u}_j^{k\epsilon})$ take values in $[-1, 1]$. Applying Hoeffding's inequality, we have for any $\tau > 0$,
\[
\mathbb{P}\left( \left| \frac{1}{N}\sum_{j=1}^N f(\tilde{u}_j^{k\epsilon}) - \mathbb{E}[f] \right| \ge \tau \right) \le 2 \exp\left( - \frac{2 N \tau^2}{(1 - (-1))^2} \right) = 2 \exp\left( -\frac{N \tau^2}{2} \right).
\]
We require this bound to hold simultaneously across all $K_T \le T/\epsilon$ discrete time steps. Applying the union bound, the probability that the maximum error exceeds $\tau$ is
\begin{align*}
    \mathbb{P}\left( \max_{k}  \left| \frac{1}{N}\sum_{j=1}^N f(\tilde{u}_j^{k\epsilon}) - \mathbb{E}[f] \right| \ge \tau \right) 
    &\le 2 \left(\frac{T}{\epsilon}\right) \exp\left(-\frac{N \tau^2}{2}\right).
\end{align*}

To find the high-probability error bound, we set the right-hand side equal to our target failure probability $e^{-\delta^2}$ and solve for the deviation $\tau$ given below
\[
\frac{N \tau^2}{2} = \delta^2 + \log(2T/\epsilon)  \implies \tau = \sqrt{\frac{2}{N}} \sqrt{ \log(2T/\epsilon) + \delta^2 }.
\]
 So, we can say that with probability at least $1 - e^{-\delta^2}$, \begin{align*}
    \operatorname{Term~2}
    \le
    \sqrt{\frac{2}{N}} \sqrt{ \log\!\left(\frac{2T}{\epsilon}\right) + \delta^2 } \le \sqrt{\frac{2}{N}}\left(\sqrt{\log\!\left(\frac{2T}{\epsilon}\right)}+\delta\right).
\end{align*}

Combining Term 1 and Term 2, we obtain the total deviation for any fixed bounded Lipschitz test function evaluated on the decoder particles
\begin{align*}
    \max_{k} \left| \frac{1}{N}\sum_{j=1}^N f(u_j^k) - \int f(u)\, \rho_{U, k\epsilon}(du) \right| \le \sqrt{\Delta_U(T)} + \sqrt{\frac{2}{N}}\left(\sqrt{\log\!\left(\frac{2T}{\epsilon} \lor 1\right)}+\delta\right).
\end{align*}

By \Cref{lem:trajectory}, with probability at least $1-4e^{-\delta^2}$, the joint maximal trajectory deviation $\Delta(T)$ is bounded by,
\[
\Delta(T) \le C_{\operatorname{poly}} (T \lor T^2) \left(\epsilon + \frac{1}{N}\right) \Big[ d+1 + \log \big( N(T/\epsilon \lor 1) \big) + \delta^2 \Big] \exp\Big( C_{\operatorname{exp}} T^2 \Big).
\]

Since $\Delta_U(T) \le \Delta(T)$, taking the square root of both sides gives an upper bound on $\sqrt{\Delta_U(T)}$. We bound the sum $\epsilon + 1/N \le 2(\epsilon \lor N^{-1})$, which yields $\sqrt{\epsilon + 1/N} \le \sqrt{2}\sqrt{\epsilon \lor N^{-1}}$. Now applying the subadditivity property $\sqrt{a+b} \le \sqrt{a} + \sqrt{b}$ to the bracketed terms, we get, 
\begin{align*}
    \sqrt{\Delta_U(T)} \le \sqrt{2 C_{\operatorname{poly}}} (\sqrt{T} \lor T) \sqrt{\epsilon \lor N^{-1}} \left[ \sqrt{d+1 + \log\left(N \left(\frac{T}{\epsilon} \lor 1\right)\right)} + \delta \right] \exp\left( \frac{C_{\operatorname{exp}}}{2} T^2 \right).
\end{align*}

Letting $D = d+1$, we have
\[
\err_{N,D}(\delta) \equiv \sqrt{N^{-1} \lor \epsilon} \left( \sqrt{D + \log\left(N \left(\frac{T}{\epsilon} \lor 1\right)\right)} + \delta \right).
\]
Setting $C = \max\big(\sqrt{2 C_{\operatorname{poly}}}, \frac{C_{\operatorname{exp}}}{2}\big)$, we find that $\sqrt{\Delta_U(T)}$ is bounded by
\[
\sqrt{\Delta_U(T)} \le C (\sqrt{T} \lor T) \exp\big(C T^2\big) \err_{N,D}(\delta).
\]
The statistical fluctuation from Term 2 is $\mathcal{O}(\sqrt{\log(T/\epsilon)/N})$, which is dominated by this unified $\err_{N,D}(\delta)$ rate. So we can conclude with probability at least $1-5e^{-\delta^2}$, 

\begin{align*}
    \max_{k} \left| \frac{1}{N}\sum_{j=1}^N f(u_j^k) - \int f(u)\, \rho_{U, k\epsilon}(du) \right| \le  2C (\sqrt{T} \lor T) \exp\big(C T^2\big) \err_{N,D}(\delta).
\end{align*}

By applying a symmetric argument for the encoder particles, we can conclude that with probability at least $1-e^{-\delta^2}$,
\begin{align*}
    \max_{k} \left| \frac{1}{N}\sum_{j=1}^N f(w_j^k) - \int f(w)\, \rho_{W, k\epsilon}(dw) \right| \le \sqrt{\Delta_W(T)} + \sqrt{\frac{2}{N}}\left(\sqrt{\log\!\left(\frac{2T}{\epsilon} \lor 1\right)}+\delta\right).
\end{align*}
Analogous to the earlier case, we have that with probability at least $1-4 e ^{-\delta^2}$
\[\sqrt{\Delta_W(T)} \le C (\sqrt{T} \lor T) \exp\big(C T^2\big) \err_{N,D}(\delta).\]
Hence with probability at least $1-6e^{-\delta^2}$
\begin{align*}
    &\max_{k} \left| \frac{1}{N}\sum_{j=1}^N f(u_j^k) - \int f(u)\, \rho_{U, k\epsilon}(du) \right| \le  2C (\sqrt{T} \lor T) \exp\big(C T^2\big) \err_{N,D}(\delta)\\
    & \max_{k} \left| \frac{1}{N}\sum_{j=1}^N f(w_j^k) - \int f(w)\, \rho_{W, k\epsilon}(dw) \right| \le  2C (\sqrt{T} \lor T) \exp\big(C T^2\big) \err_{N,D}(\delta)
\end{align*}

Next, we bound the empirical risk and the true risk. By the triangle inequality, we decompose the total risk error into the finite-particle tracking gap and the continuous measure concentration gap as given below
\begin{align*}
    \max_k \left| \mathcal{R}_N(\theta_{U,k}, \theta_{W,k}) - \mathcal{R}(\rho_{U,k\epsilon}, \rho_{W,k\epsilon}) \right| & \le \max_k\underbrace{\left| \mathcal{R}_N(\theta_{U,k}, \theta_{W,k}) - \mathcal{R}_N(\tilde{U}^{k\epsilon}, \tilde{W}^{k\epsilon}) \right|}_{\text{Trajectory Gap}} \\
    &\quad + \max_k\underbrace{\left| \mathcal{R}_N(\tilde{U}^{k\epsilon}, \tilde{W}^{k\epsilon}) - \mathcal{R}(\rho_{U,k\epsilon}, \rho_{W,k\epsilon}) \right|}_{\text{Measure Gap}}.
\end{align*}
By \Cref{lem:risk_lip}, the trajectory gap term is upper bounded by $2M^3 \sqrt{\Delta(T)}$. By \Cref{lem:risk_conc}, the measure gap term is upper bounded by $5M^3 \frac{1}{\sqrt{N}} \left[ \sqrt{d+1+\log(2(T/\epsilon \lor 1))} + \delta \right]$ with probability at least $1-e^{-\delta^2}$. Now, we can say that with probability at least $1-5e^{-\delta^2}$,
\begin{align*}
    \max_k  \left| \mathcal{R}_N(\theta_{U,k}, \theta_{W,k}) - \mathcal{R}(\rho_{U,k\epsilon}, \rho_{W,k\epsilon}) \right| & \le 10M^3C (\sqrt{T} \lor T \lor 1) \exp\big(C T^2\big) \err_{N,D}(\delta).
\end{align*}
So, finally we conclude that with probability at least $1-7e^{-\delta^2}$,
\begin{align*}
    &\max_{k \in [0, T/\epsilon] \cap \mathbb{N}}  \left| \frac{1}{N}\sum_{j=1}^N f(u_j^k) - \int f(u)\, \rho_{U, k\epsilon}(du) \right| \le  2C (\sqrt{T} \lor T) \exp\big(C T^2\big) \err_{N,D}(\delta),\\
    & \max_{k \in [0, T/\epsilon] \cap \mathbb{N}}  \left| \frac{1}{N}\sum_{j=1}^N f(w_j^k) - \int f(w)\, \rho_{W, k\epsilon}(dw) \right| \le  2C (\sqrt{T} \lor T) \exp\big(C T^2\big) \err_{N,D}(\delta),\\
    & \max_{k \in [0, T/\epsilon] \cap \mathbb{N}}  \left| \mathcal{R}_N(\theta_{U,k}, \theta_{W,k}) - \mathcal{R}(\rho_{U,k\epsilon}, \rho_{W,k\epsilon}) \right|  \le 10M^3C (\sqrt{T} \lor T \lor 1) \exp\big(C T^2\big) \err_{N,D}(\delta).
\end{align*}

Finally, we show the almost sure weak convergence of the empirical SGD measures to the continuous Vlasov-McKean solutions at fixed time $t \le T$. 

First, fix a test function $f$. For each $N \in \mathbb{N}$, let us set the confidence parameter to $\delta_N = \sqrt{2 \log N}$, and define the failure event $E_N$ as the subset of the sample space $\Omega$ where our uniform bound fails to hold at sequence step $N$. Substituting $\delta_N$ into our previous probability bound, the measure of this failure event is bounded by
\[
\mathbb{P}(E_N) \le 7 \exp\left( -\delta_N^2 \right) = 7 \exp(-2 \log N) = \frac{7}{N^2}.
\]

Let us consider the limit superior of this sequence of events which is the event that the failures occur infinitely often,
\[
\limsup_{N \to \infty} E_N = \bigcap_{N=1}^\infty \bigcup_{k=N}^\infty E_k.
\]
Because the sum of the probabilities forms a convergent $p$-series,
\[
\sum_{N=1}^\infty \mathbb{P}(E_N) \le \sum_{N=1}^\infty \frac{7}{N^2} < \infty,
\]
the First Borel-Cantelli Lemma guarantees that the probability of the events occurring infinitely often is zero. That is, $\mathbb{P}\left(\limsup_{N \to \infty} E_N\right) = 0$. Since the measure of the complement is $\mathbb{P}\left(\liminf_{N \to \infty} E_N^c\right) = 1 - 0 = 1$, we conclude that almost surely, there exists some finite threshold index $N_0 \in \mathbb{N}$ (which may depend on the random outcome $\omega \in \Omega$) such that for all $k \ge N_0$, the success event $E_k^c$ occurs. Thus, with probability 1, the uniform bounds hold for all sufficiently large $N$.

We evaluate the asymptotic behavior of the error tracking term $\err_{N,D}(\delta_N)$ under the theorem's sequence limits. Substituting $\delta_N = \sqrt{2 \log N}$ gives
\[
\err_{N,D}(\delta_N) = \sqrt{\max\left(\frac{1}{N}, \epsilon_N\right)} \left( \sqrt{D + \log\left(N \left(\frac{T}{\epsilon_N} \lor 1\right)\right)} + \sqrt{2 \log N} \right).
\]
Squaring this error term reveals that it is bounded by a constant multiple of $\max(N^{-1}, \epsilon_N) \log(N/\epsilon_N)$. Under the assumption of the \Cref{thm:main_dynamics} that $\frac{N}{\log(N/\epsilon_N)} \to \infty$ and $\epsilon_N \log(N/\epsilon_N) \to 0$, which implies that the entire tracking error $\lim_{N \to \infty} \err_{N,D}(\delta_N) = 0$.

Let $k_N = \lfloor t/\epsilon_N \rfloor$ be the discrete step corresponding to time $t$. We must account for the temporal shift between the continuous ODE evaluated at the grid point $k_N \epsilon_N$ and the exact time $t$. The difference in time is bounded by $|t - k_N \epsilon_N| \le \epsilon_N$. Under the Vlasov-McKean dynamics, the continuous particles move with a velocity $G_U$ bounded by $M_G \coloneqq 2M^3$ and learning rate bounded by $M_\xi$. Therefore, the continuous measure shifts by at most the maximum particle displacement, giving a 1-Wasserstein bound of
\[
W_1(\rho_{U, k_N \epsilon_N}, \rho_{U, t}) \le M_\xi M_G |t - k_N \epsilon_N| \le M_\xi M_G \epsilon_N.
\]
For any specific test function $f$ with $\norm{f}_{\mathrm{Lip}} \le 1$, the integral difference is bounded by the 1-Wasserstein distance. Adding this temporal shift to our spatial tracking bound via the triangle inequality yields
\begin{align*}
    &\left| \int f(u) d\hat{\rho}^{(N)}_{U,k_N}(u) - \int f(u) \rho_{U, t}(du) \right| \\
    &\le \left| \frac{1}{N}\sum_{j=1}^N f(u_j^{k_N}) - \int f(u)\, \rho_{U, k_N \epsilon_N}(du) \right| + \left| \int f(u) \rho_{U, k_N \epsilon_N}(du) - \int f(u) \rho_{U, t}(du) \right| \\
    &\le 2C (\sqrt{T} \lor T) \exp\big(C T^2\big) \err_{N,D}(\delta_N) + W_1(\rho_{U, k_N \epsilon_N}, \rho_{U, t})\\
    &\le 2C (\sqrt{T} \lor T) \exp\big(C T^2\big) \err_{N,D}(\delta_N) + M_\xi M_G \epsilon_N.
\end{align*}
Because $\err_{N,D}(\delta_N) \to 0$ almost surely and $\epsilon_N \to 0$ deterministically, the total evaluation error for this fixed test function $f$ converges to 0 almost surely.

To extend this convergence to the entire class of bounded continuous functions, we use the topological properties of our parameter space. Because the parameter domain $\Omega_U$ is a compact metric space, the space of bounded continuous functions $C_b(\Omega_U)$ equipped with the sup norm $\norm{\cdot}_\infty$ is separable. Furthermore, the subspace of bounded Lipschitz functions is dense in $C_b(\Omega_U)$. Therefore, there exists a countable dense subset of test functions $\mathcal{D} = \{f_m\}_{m=1}^\infty \subset C_b(\Omega_U)$. 

From our previous analysis, for each individual test function $f_m \in \mathcal{D}$, the sequence of empirical expectations converges to the continuous expectation almost surely. Formally, for each $m \in \mathbb{N}$, there exists an event $A_m$ in the underlying sample space with probability $\mathbb{P}(A_m) = 1$, such that on $A_m$,
\[
\lim_{N \to \infty} \left| \int_{\Omega_U} f_m(u) \, d\hat{\rho}^{(N)}_{U,\lfloor t/\epsilon_N \rfloor}(u) - \int_{\Omega_U} f_m(u) \, \rho_{U, t}(du) \right| = 0.
\]
Let $A_\infty = \bigcap_{m=1}^\infty A_m$ be the intersection of these success events. Because the countable intersection of measure-one sets retains measure one, $\mathbb{P}(A_\infty) = 1$. Thus, on the single event $A_\infty$, the empirical integral converges to the continuous integral simultaneously for all $f_m \in \mathcal{D}$.

We now show that on the event $A_\infty$, convergence holds for any $f \in C_b(\Omega_U)$. Fix an arbitrary $f \in C_b(\Omega_U)$ and let $\varepsilon > 0$. By the density of $\mathcal{D}$, we can select an $f_m \in \mathcal{D}$ such that $\norm{f - f_m}_\infty < \frac{\varepsilon}{3}$. For any $N \in \mathbb{N}$, we apply the triangle inequality to decompose the integration error into three components,
\begin{align*}
    & \left| \int_{\Omega_U} f \, d\hat{\rho}^{(N)}_{U,\lfloor t/\epsilon_N \rfloor} - \int_{\Omega_U} f \, d\rho_{U, t} \right| 
    \le \underbrace{ \left| \int_{\Omega_U} f \, d\hat{\rho}^{(N)}_{U,\lfloor t/\epsilon_N \rfloor} - \int_{\Omega_U} f_m \, d\hat{\rho}^{(N)}_{U,\lfloor t/\epsilon_N \rfloor} \right| }_{\text{Term 1}}\\
    &+ \underbrace{ \left| \int_{\Omega_U} f_m \, d\hat{\rho}^{(N)}_{U,\lfloor t/\epsilon_N \rfloor} - \int_{\Omega_U} f_m \, d\rho_{U, t} \right| }_{\text{Term 2}}
   + \underbrace{ \left| \int_{\Omega_U} f_m \, d\rho_{U, t} - \int_{\Omega_U} f \, d\rho_{U, t} \right| }_{\text{Term 3}}.
\end{align*}

Because both $\hat{\rho}^{(N)}$ and $\rho$ are probability measures, their total mass is exactly 1. We can therefore bound Term 1 and Term 3 using the uniform norm of the difference,
\begin{align*}
    \text{Term 1} &\le \int_{\Omega_U} |f - f_m| \, d\hat{\rho}^{(N)}_{U,\lfloor t/\epsilon_N \rfloor} \le \norm{f - f_m}_\infty \int_{\Omega_U} 1 \, d\hat{\rho}^{(N)}_{U,\lfloor t/\epsilon_N \rfloor} < \frac{\varepsilon}{3}, \\
    \text{Term 3} &\le \int_{\Omega_U} |f_m - f| \, d\rho_{U, t} \le \norm{f - f_m}_\infty \int_{\Omega_U} 1 \, d\rho_{U, t} < \frac{\varepsilon}{3}.
\end{align*}

For Term 2, because we are operating within the almost sure event $A_\infty$, we know that $\lim_{N \to \infty} \text{Term 2} = 0$. Consequently, there exists an integer threshold $N_0$ such that for all $N \ge N_0$, $\text{Term 2} < \frac{\varepsilon}{3}$.

Summing the bounds, we conclude that for all $N \ge N_0$,
\[
\left| \int_{\Omega_U} f \, d\hat{\rho}^{(N)}_{U,\lfloor t/\epsilon_N \rfloor} - \int_{\Omega_U} f \, d\rho_{U, t} \right| < \frac{\varepsilon}{3} + \frac{\varepsilon}{3} + \frac{\varepsilon}{3} = \varepsilon.
\]
Since $\varepsilon > 0$ was arbitrarily chosen, the sequence of expectations converges for every $f \in C_b(\Omega_U)$. Thus, we rigorously achieve weak convergence of the decoder parameters
\[
\hat{\rho}^{(N)}_{U,\lfloor t/\epsilon_N \rfloor} \Rightarrow \rho_{U,t} \quad \text{almost surely.}
\]

By a perfectly symmetric application of the time-shift bounds and identical topological density arguments for the encoder measures, we concurrently guarantee that
\[
\hat{\rho}^{(N)}_{W,\lfloor t/\epsilon_N \rfloor} \Rightarrow \rho_{W,t} \quad \text{almost surely,}
\]
which completes the proof of Theorem 1. \hfill \qedsymbol


\section{Discussion: Generalization to Multi-Dimensional Bottlenecks}
\label{sec:multidim_bottleneck}

The preceding analysis is written for a one-dimensional bottleneck ($b=1$) to avoid vector notation. The same proof strategy should extend to any fixed bottleneck dimension $b>1$. The main qualification is that constants, covering numbers, and particle dimensions may now depend on $b$; the rates in $N$ and the finite-time structure of the bounds are unchanged only when $b$ is fixed independently of $N$.

\paragraph{Structural and Notational Modifications.}
Let the bottleneck be $\boldsymbol{\alpha}(x;\rho_W)\in\mathbb{R}^b$, with encoder feature map $\phi:\mathbb{R}^d\times\Omega_W\to\mathbb{R}^b$ and decoder map $\psi:\mathbb{R}^b\times\Omega_U\to\mathbb{R}^d$. The finite and mean-field bottlenecks become $\boldsymbol{\alpha}_N(x)=N^{-1}\sum_{i=1}^N\phi(x;w_i)$ and $\boldsymbol{\alpha}(x;\rho_W)=\int\phi(x;w)\,d\rho_W(w)$, and the reconstruction is $\hat{x}(x;\rho_U,\rho_W)=\int\psi(\boldsymbol{\alpha}(x;\rho_W);u)\,d\rho_U(u)$. The standing boundedness assumptions should be interpreted in Euclidean/operator norm: $\phi$, $\psi$, $J_w\phi$, $J_{\boldsymbol{\alpha}}\psi$, and the relevant second derivatives are uniformly bounded and Lipschitz, with constants allowed to depend on the fixed dimension $b$.

\paragraph{Vector Fields and the Chain Rule.}
The decoder variation is unchanged except that $\psi$ now takes a vector input:
\[
\Psi_U(u;\rho_U,\rho_W)
=
\mathbb{E}_x\big[\langle \hat{x}(x)-x,\psi(\boldsymbol{\alpha}(x);u)\rangle\big].
\]
The encoder variation is the vector-valued chain-rule analogue of the scalar expression. Defining $\bar J_{\boldsymbol{\alpha}}\psi(\boldsymbol{\alpha})=\int J_{\boldsymbol{\alpha}}\psi(\boldsymbol{\alpha};u)\,d\rho_U(u)$, the first variation in the encoder measure takes the form
\[
\Psi_W(w;\rho_U,\rho_W)
=
\mathbb{E}_x\left[
\left\langle
\bar J_{\boldsymbol{\alpha}}\psi(\boldsymbol{\alpha}(x))^\top(\hat{x}(x)-x),
\phi(x;w)
\right\rangle
\right].
\]
Accordingly, the encoder vector field contains the backpropagated factor $J_w\phi(x;w)^\top\bar J_{\boldsymbol{\alpha}}\psi(\boldsymbol{\alpha}(x))^\top(\hat{x}(x)-x)$. Thus the mean-field dynamics remain well-defined and Lipschitz, with constants $M_G(b)$ and $L_G(b)$ replacing their scalar-bottleneck counterparts.

\paragraph{Coupled Trajectory Tracking (\Cref{lem:trajectory}).}
The decomposition behind \Cref{lem:trajectory} is structurally unchanged: time discretization, particle tracking, measure concentration, and martingale noise are controlled in the same order. The only modification is that the parametric test classes used for measure concentration become vector-valued, for example through $\{u\mapsto \psi(\boldsymbol{\alpha};u):\boldsymbol{\alpha}\in\mathbb{R}^b\}$, $\{u\mapsto J_{\boldsymbol{\alpha}}\psi(\boldsymbol{\alpha};u):\boldsymbol{\alpha}\in\mathbb{R}^b\}$, and $\{w\mapsto\phi(x;w):x\in\mathcal X\}$. For fixed $b$, their covering numbers and Lipschitz constants only change the constants in the coupled Gronwall estimate. Hence the trajectory bound keeps the same form as in \Cref{lem:trajectory}, with $C_{\operatorname{poly}}$, $C_{\operatorname{exp}}$, and $D$ replaced by $b$-dependent analogues $C_{\operatorname{poly}}(b)$, $C_{\operatorname{exp}}(b)$, and $D_b$.

\paragraph{Risk Concentration (\Cref{lem:risk_conc}).}
The McDiarmid step also keeps the same structure. Changing one encoder particle changes $\boldsymbol{\alpha}_N(x)$ by at most $C_b/N$ in Euclidean norm and therefore changes $\hat{x}_N(x)$ by at most $C_b/N$, where $C_b$ depends on the fixed bottleneck dimension and the uniform derivative bounds. Consequently the bounded difference satisfies $c_b\le C_b/N$, so $\sum_i c_b^2\le C_b/N$. After the same union bound over the discretized time grid, \Cref{lem:risk_conc} becomes
\[
\max_k \left| \mathcal{R}_N - \mathcal{R} \right|
\le
C_b\sqrt{\frac{1}{N}}
\left[
\sqrt{D_b+\log\left(N\left(\frac{T}{\epsilon}\lor 1\right)\right)}
+z
\right].
\]
Thus the risk concentration rate is still $N^{-1/2}$ for fixed $b$, with the unified error rate $\err_{N,D}(z)$ replaced by the same expression with $D$ and the constants replaced by their $b$-dependent analogues.

\paragraph{Conclusion and Scope.}
The same fixed-$b$ modification applies to the optimum-comparison and variational perturbation arguments. The scalar Taylor expansion in $\alpha$ is replaced by a multivariate Taylor expansion in $\boldsymbol{\alpha}\in\mathbb{R}^b$, with Hessian remainders bounded by $C_b\|\Delta\boldsymbol{\alpha}\|_2^2$. Since the empirical bottleneck satisfies $\mathbb{E}\|\boldsymbol{\alpha}_N(x)-\boldsymbol{\alpha}(x)\|_2^2\le C_b/N$, the finite-width optimum comparison retains its $N^{-1}$ scaling up to $b$-dependent constants. Similarly, the variational support condition follows by the same small-mass perturbation argument, with the encoder perturbation propagated through the decoder by the multivariate chain rule. Finally, the weak convergence argument is unchanged: the parameter spaces remain compact metric spaces, bounded continuous test functions remain separable, and the tracking and concentration errors vanish as $N\to\infty$ for each fixed $b$. If $b$ is allowed to grow with $N$ or $d$, the $b$-dependence of constants and covering numbers would have to be tracked explicitly and would require a separate analysis.